\documentclass{article}
\usepackage[utf8]{inputenc}
\usepackage{fullpage}
\usepackage[margin=1.0in]{geometry}

\usepackage{amsfonts}
\usepackage{graphicx}
\usepackage{epstopdf}

\usepackage{environ} 
\usepackage{afterpage}
\usepackage{mathrsfs}
\usepackage{bm}
\usepackage{mathtools}
\usepackage{booktabs}
\usepackage{subfigure}

\usepackage{refcount}

\newcommand{\myfootnotetext}[1]{\footnotetext{#1\label{fn:text}%
        \edef\fnmark{\getpagerefnumber{fn:mark}}%
        \edef\fntext{\getpagerefnumber{fn:text}}%
        \ifx\fnmark\fntext\else\ClassWarning{}{footnote mark and text on different pages!}\fi}}

\usepackage{appendix}
\usepackage{amsmath}
\usepackage{amssymb}
\usepackage{latexsym}
\usepackage{color}
\usepackage{tabularx}
\usepackage[english]{babel}
\usepackage{array}
\usepackage{algorithm}
\usepackage{algorithmic}
\usepackage{esint}
\usepackage[nopar]{lipsum}
\usepackage{float}

\usepackage[symbol]{footmisc}

\usepackage{enumerate}

\usepackage{bbm}
\usepackage{color}
\usepackage{xcolor}
\usepackage{subcaption}
\usepackage{makecell}

\usepackage{amsthm}
\newtheorem{theorem}{Theorem}[section]
\newtheorem{assumption}[theorem]{Assumption}

\newtheorem{remark}[theorem]{Remark}
\newtheorem{lemma}[theorem]{Lemma}
\newtheorem{proposition}[theorem]{Proposition}
\newtheorem{definition}[theorem]{Definition}
\newtheorem{corollary}[theorem]{Corollary}

\renewcommand{\algorithmicrequire}{\textbf{Input:}}
\renewcommand{\algorithmicensure}{\textbf{Output:}}
\usepackage{hyperref}
 \usepackage{multirow}
 
\usepackage{tikz}
\DeclareFontFamily{OT1}{pzc}{}
\DeclareFontShape{OT1}{pzc}{m}{it}{<-> s * [1.10] pzcmi7t}{}
\DeclareMathAlphabet{\mathpzc}{OT1}{pzc}{m}{it}

\numberwithin{equation}{section}

\newcommand{\rhoL}{\rho_T^L}

\newcommand{\vertiii}[1]{{\left\vert\kern-0.25ex\left\vert\kern-0.25ex\left\vert #1 
    \right\vert\kern-0.25ex\right\vert\kern-0.25ex\right\vert}}

\newcommand{\mres}{\mathbin{\vrule height 1.2ex depth 0pt width
0.13ex\vrule height 0.13ex depth 0pt width 0.9ex}}

\newcommand{\mbf}[1]{\boldsymbol{#1}}

\newcommand{\real}{\mathbb{R}}

\newcommand{\br}{\mbf{r}}

\newcommand{\bv}{{\mbf{v}}}

\newcommand{\bx}{{\mbf{x}}}
\newcommand{\bX}{\mbf{X}}
\newcommand{\bbX}{\mathbb{X}}

\newcommand{\bY}{\mbf{Y}}
\newcommand{\by}{\mbf{y}}

\newcommand{\bgamma}{\mbf{\gamma}}

\newcommand{\mE}{\mathcal{E}}

\newcommand{\mH}{\mathcal{H}_{{K}}}
\newcommand{\mHe}{\mathcal{H}_{{K}^E}}
\newcommand{\mHa}{\mathcal{H}_{{K}^A}}
\newcommand{\mHpq}{\mathcal{H}_{{K}^{pq}}}

\newcommand{\mK}{{K}}

\newcommand{\R}{\real}

\newcommand{\bZ}{\mbf{Z}}
\newcommand{\bbZ}{\mathbb{Z}}

\newcommand{\intkernele}{{\intkernel^{E}}}

\newcommand{\intkernel}{\phi}

\newcommand{\bintkernel}{{\bm{\phi}}}

\newcommand{\intkernelvar}{\varphi}
\newcommand{\bintkernelvar}{{\bm{\varphi}}}

\newcommand{\rhsfo}{\mathcal{F}}

\newcommand{\E}{\mathbb{E}}
\newcommand{\cov}{\mathrm{Cov}}

\newcommand{\argmin}[1]{\underset{#1}{\operatorname{arg}\operatorname{min}}\;}

\renewcommand{\algorithmicrequire}{\textbf{Input:}}    
\renewcommand{\algorithmicensure}{\textbf{Output:}}
\newcommand{\infnorm}[1]{\| #1\|_{\infty}}

\newcommand{\Rhoxnorm}[1]{\| #1\|_{L^2(\rho_{\mbf{X}})}}

\newcommand{\rhotnorm}[1]{\| #1\|_{L^2(\tilde\rho_{T}^L)}}

\newcommand{\Rhoxinnerp}[2]{\langle #1, #2\rangle_{L^2(\rho_{\mbf{X}})}}

\DeclareMathAlphabet{\mathpzc}{OT1}{pzc}{m}{it}

\title{Data-driven Learning of Interaction Laws in Multispecies Particle Systems with Gaussian Processes: Convergence Theory and Applications}

\author{Jinchao Feng\thanks{School of Sciences, Great Bay University, Dongguan, Guangdong, China
  (\text{jcfeng@gbu.edu.cn}).}
\and Charles Kulick\thanks{Department of Mathematics, University of California, Santa Barbara, Isla Vista, CA 
  (\text{charles@math.ucsb.edu}).}
\and Sui Tang\thanks{Department of Mathematics, University of California, Santa Barbara, Isla Vista, CA 
  (\text{suitang@math.ucsb.edu}).}}

\date{}

\begin{document}

\maketitle

\begin{abstract}%
We develop a Gaussian process framework for learning interaction kernels in multi-species interacting particle systems from trajectory data. Such systems provide a canonical setting for multiscale modeling, where simple microscopic interaction rules generate complex macroscopic behaviors. While our earlier work established a Gaussian process approach and convergence theory for single-species systems, and later extended to second-order models with alignment and energy-type interactions, the multi-species setting introduces new challenges: heterogeneous populations interact both within and across species, the number of unknown kernels grows, and asymmetric interactions such as predator–prey dynamics must be accommodated. We formulate the learning problem in a nonparametric Bayesian setting and establish rigorous statistical guarantees. Our analysis shows recoverability of the interaction kernels, provides quantitative error bounds, and proves statistical optimality of posterior estimators, thereby unifying and generalizing previous single-species theory. Numerical experiments confirm the theoretical predictions and demonstrate the effectiveness of the proposed approach, highlighting its advantages over existing kernel-based methods. This work contributes a complete statistical framework for data-driven inference of interaction laws in multi-species systems, advancing the broader multiscale modeling program of connecting microscopic particle dynamics with emergent macroscopic behavior.  
\end{abstract}

\section{Introduction} \label{sec:intro}

Interacting particle systems provide a natural microscopic description of collective dynamics in biology, physics, and the social sciences.  Pairwise interactions among agents can generate a striking variety of macroscopic behaviors, including flocking, clustering, segregation, and milling. This microscopic-to-macroscopic link makes such systems canonical examples of multiscale modeling: simple rules at the agent level can give rise to complex emergent patterns at the population level. A central challenge is to identify the governing interaction laws. 

Classical approaches have typically prescribed parametric families of interaction kernels and analyzed the resulting dynamics to establish well-posedness and show that qualitative macroscopic patterns emerge \cite{vicsek1995novel,cucker2007emergent,vicsek2012collective,gregoire2004onset,kolokolnikov2011theory,tunstrom2013collective,lewin2015crystallization,vedmedenko2007competing,chuang2007state,abaid2010fish,albi2014stability,krause2000discrete,blondel2009krause,motsch2014heterophilious,bellomo2017active,carrillo2017review}. While these works provide important insights into the range of possible behaviors, they do not resolve the quantitative question of what interaction laws govern real systems. With the increasing availability of high-resolution trajectory data, there is now a growing effort to develop data-driven methods that infer interaction kernels directly from observations \cite{ballerini2008interaction,lukeman2010inferring}.

Many natural and engineered systems are intrinsically multi-species, involving heterogeneous populations that interact both within and across groups. Examples include predator–prey systems, leader-follower opinion models, mixtures of biological or chemical populations, and multi-class pedestrian flows. Compared with the single-species case, multi-species systems display substantially richer behaviors and pose new analytical and computational challenges: populations may segregate or mix depending on interaction strengths, form asymmetric steady states, or evolve into patterns supported on irregular domains with cusps and instabilities \cite{kolokolnikov2011theory,mackey2014two}. These features underscore the need for a rigorous and scalable framework for kernel learning in multi-species systems.

\paragraph{Our contributions}
In this paper, we develop a Gaussian process framework for learning interaction kernels in multi-species particle systems. Building on our earlier work on single-species \cite{feng2024learning} and second-order models with alignment and energy-type interactions \cite{feng2024data}, we make the following contributions:
\begin{itemize}
  \item We formulate a nonparametric Bayesian approach for the joint inference of intra- and inter-species kernels, extending Gaussian process methods to heterogeneous populations.
  \item We establish a rigorous convergence theory, providing recoverability, quantitative error bounds, and statistical optimality of posterior estimators, thereby generalizing our previous results to the multi-species setting.
  \item We present numerical experiments that validate the theoretical predictions and demonstrate the effectiveness and computational advantages of the proposed method.
\end{itemize}
Our results provide a complete statistical framework for data-driven inference in multi-species interacting particle systems, contributing to the broader multiscale modeling program of connecting microscopic agent-level rules with macroscopic emergent behaviors.

\subsection*{Relevant Works}  

Gaussian processes (GPs) are a flexible nonparametric Bayesian tool for supervised learning with built-in uncertainty quantification. They have been successfully applied to dynamical systems, including ODEs, SDEs, and PDEs \cite{heinonen2018learning,archambeau2007gaussian,yildiz2018learning,zhao2020state,raissi2017machine,chen2020gaussian,wang2021explicit,lee2020coarse,akian2022learning,offen2025machine}, where careful adaptation to the structure of dynamical data has led to accurate and robust data-driven models.  

Our earlier work \cite{feng2024learning} developed a GP-based framework for learning interaction kernels in \emph{single-species} particle systems, establishing identifiability and convergence guarantees while embedding translation and rotational invariance. A follow-up study \cite{feng2024data} extended this framework to \emph{second-order} particle systems with alignment and energy-type interactions, emphasizing computational aspects, scalable inference, and applications to real-world fish milling data. The present paper generalizes these ideas to the \emph{multi-species} setting, providing a complete learning theory for both intra- and inter-species kernels. In particular, our results also cover, as a special case, the statistical theory for the model selection problems studied in \cite{feng2024data}.  

In the broader literature, \cite{lu2021learning,miller2020learning} studied kernel inference in heterogeneous particle systems and demonstrated that simultaneously estimating multiple interaction kernels is inherently challenging, with regularization being essential. Related works have developed kernel methods for learning interaction laws \cite{lang2020learning,liu2023random,fiedler2025recent} and, more generally, convolution kernels \cite{li2025automatic}. From this perspective, GPs can be viewed as a probabilistic analogue of kernel methods: while kernel methods impose deterministic regularization via reproducing kernel Hilbert spaces, GPs provide a Bayesian formulation that combines regularization with posterior uncertainty quantification, joint parameter–kernel inference, and principled data-driven prior selection. These features make GPs particularly well-suited for data-driven inference of interaction laws in multi-species systems.  

Finally, we note that the operator-theoretic error analysis introduced in our earlier GP-based work has since been adapted to other contexts, including structure-preserving kernel methods for Hamiltonian \cite{hu2025structure} and Poisson systems \cite{hu2025global}. This further underscores the versatility of the approach and motivates the present generalization to multi-species interacting systems.  

\paragraph{Notations and Preliminaries on Hilbert Space}

Let $\rho$ be a Borel positive measure on $\mathbb{R}^{D}$.  We use $L^2(\mathbb{R}^{D};\rho;\mathbb{R}^{n})$ to denote the set of $L^2(\rho)$ integrable vector-valued functions that map $\mathbb{R}^{D}$ to $\mathbb{R}^{n}$.  Let $\mathcal{S}_1$ be  a  measurable subset of $\mathbb{R}^{m}$, the restriction of the measure $\rho$ on $\mathcal{S}_1$, denoted by $\rho\mres \mathcal{S}_1$, is defined as $\rho\mres \mathcal{S}_1(\mathcal{S}_2)=\rho(\mathcal{S}_1 \cap \mathcal{S}_2)$ for any measurable subset $\mathcal{S}_2$ of $\mathbb{R}^{D}$. 

Let $\mathcal{H}$ be a Hilbert space. We denote by $\mathcal{B}(\mathcal{H})$  the set of bounded linear operators mapping $\mathcal{H}$ to itself. We use $\langle\cdot,\cdot\rangle_{\mathcal{H}}$ to denote its inner product, and we still use $\langle\cdot,\cdot\rangle$ to denote the inner product on the Euclidean space. For $d,N,M,L \in \mathbb{N}^+$, let $\mbf{w}=(\mbf{w}_{m,l,i})_{m,l,i=1}^{M,L,N},\mbf{z}=(\mbf{z}_{m,l,i})_{m,l,i=1}^{M,L,N}\in \mathbb{R}^{dNML}$ with  
$\mbf{w}_{m,l,i},\mbf{z}_{m,l,i} \in\mathbb{R}^d$, we define 
\begin{align}\label{winnerp}
\langle \mbf{w}, \mbf{z}\rangle =\frac{1}{MLN}\sum_{m,l,i=1}^{M,L,N} \langle \mbf{w}_{m,l,i}, \mbf{z}_{m,l,i}\rangle
\end{align} where $\langle \mbf{w}_{m,l,i}, \mbf{z}_{m,l,i}\rangle$ is the canonical inner product over $\mathbb{R}^d$.

Let $A \in \mathcal{B}(\mathcal{H})$,  the notation $\mathrm{Im}(A)$ denotes its image space and $\|A\|_{\mathcal{H}}$  denotes its operator norm.  If $A$ is a Hilbert-Schmidt operator,  then $\|A\|_{HS}$ denotes its Hilbert–Schmidt norm that satisfies $\|A\|_{HS}^2=\mathrm{Tr}(A^*A)$.  For two self-adjoint operators $A, B \in \mathcal{B}(\mathcal{H})$, we say that $A\geq B$ if $A-B$ is a positive operator, i.e. $\langle (A-B)h,h\rangle_{\mathcal{H}} \geq 0$ for all $h \in \mathcal{H}$. If $A$ is a compact positive operator, then $\lambda_n^{}$  represents the $n$th eigenvalue in decreasing order. By the spectral theory of compact operators, 
the eigenfunctions $\{\intkernelvar_n\}_{n=1}^{N}$ (note $N$ can be $\infty$) of $A$ form an orthonormal basis for $\mathcal{H}$ so that $A=\sum_{n=1}^{N}\lambda_n^{}\intkernelvar_n$. For $\tau <0$, we define $A^{\tau}=\sum_{n=1}^{N}\lambda_n^{\tau}\intkernelvar_n$ on the subspace $S_{\tau}$ of  $\mathcal{H}$ given by 
$$S_\tau=\{\sum_{n=1}^{N} a_n\intkernelvar_n|\sum_{n=1}^{N}(a_n\lambda_n^{\tau})^2 \text{ is convergent}\}.$$ If $h \not\in S_\tau$, then $\|A^{\tau}h\|_{\mathcal{H}}=\infty.$

\paragraph{Preliminaries on RKHS} Let $\mathcal{D}$ be a compact domain of $\mathbb{R}^D$. We say that $K: \mathcal{D}\times \mathcal{D} \rightarrow \mathbb{R}$ is a Mercer kernel if it is continuous, symmetric, and positive semidefinite, i.e., for any finite set of distinct points $\{x_1,\cdots,x_M\} \subset \mathcal{D},$ the matrix $(K(x_i,x_j))_{i,j=1}^{M}$ is positive semidefinite.  For $x\in \mathbb{R}^D$, $K_{x}$ is a  function defined on $\mathcal{D}$ such that $K_{x}(y)=K(x,y).$

The Moore–Aronszajn theorem proves that there is a Reproducing Kernel Hilbert Space (RKHS) $\mH$ associated with the kernel $K$, which is defined to be the closure of the linear span of the set of functions $\{K_x:x\in \mathcal{D}\}$ with respect to the inner product $\langle \cdot,\cdot\rangle_{\mH}$ satisfying $\langle K_x, K _y\rangle_{\mH}=K(x,y)$. For every $f\in \mH$, we have $\langle f,K_x\rangle_{\mH}=f(x)$. This property is called the reproducing property. Common examples of RKHSs include the Sobolev spaces.

\paragraph{Organization of the paper}  
The remainder of the paper is organized as follows. In Section~\ref{sec:setup}, we introduce the multi-species interacting particle system, establish notation, and formulate the kernel learning problem. Section~\ref{sec: Methodology} presents the Gaussian process framework for joint inference of intra- and inter-species interaction kernels. Our main theoretical results, including convergence guarantees and statistical optimality of the estimators, are stated and proved in Section~\ref{sec:theory}. Section~\ref{sec:numerics} provides numerical experiments that validate the theoretical predictions and illustrate the effectiveness of the proposed method. We conclude in Section~\ref{sec:conclusion} with a summary and a discussion of directions for future work.

\section{Model Setup and Problem Formulation} 
\label{sec:setup}

We consider an interacting particle system with two types of agents in the Euclidean space $\mathbb{R}^d$. The dynamics are governed by the first-order system:  for $i=1,\dots,N_1$
\begin{align} \label{e:firstordersystem1}
  \dot{\bx}_i(t) &= \frac{1}{N} \Bigg[\sum_{i' = 1}^{N_1} \intkernel^{11}(||\bx_{i'}(t)-\bx_i(t)||)(\bx_{i'}(t)-\bx_i(t))+\sum_{i'=N_1+1}^{N} \intkernel^{12}(||\bx_{i'}(t)-\bx_i(t)||)(\bx_{i'}(t)-\bx_i(t))\Bigg],
\end{align}
for $i=N_1+1,\cdots,N$
\begin{align} \label{e:firstordersystem2}
    \dot{\bx}_i(t) &= \frac{1}{N} \Bigg[\sum_{i' = 1}^{N_1} \intkernel^{21}(||\bx_{i'}(t)-\bx_i(t)||)(\bx_{i'}(t)-\bx_i(t))+\sum_{i'=N_1+1}^{N} \intkernel^{22}(||\bx_{i'}(t)-\bx_i(t)||)(\bx_{i'}(t)-\bx_i(t))\Bigg],
\end{align} where $N=N_1+N_2$ is the total number of agents, with $N_1$ agents of type~1 and $N_2$ agents of type~2.  The interaction kernels $\{\phi^{pq}\}_{p,q=1}^{2}:\R_+\to\R$ encode how agents of type $p$ influence those of type $q$. In general, $\phi^{12}$ and $\phi^{21}$ need not coincide, reflecting asymmetric interactions such as predator–prey dynamics. The velocity of each agent is obtained by superimposing the interactions with all other agents, each directed toward the other agent and weighted by the kernel evaluated at their mutual distance.  

This framework generalizes single-species models by incorporating both intra- and inter-species interactions. It has been applied to describe a variety of collective behaviors, including heterogeneous particle dynamics, predator–prey systems, and leader–follower models in opinion dynamics. Compared with the single-species case, two-species systems exhibit significantly richer dynamics, such as segregation versus mixing, asymmetric steady states, and pattern formation on irregular domains with cusps and instabilities.  

\begin{table}[tbh]
\centering
\caption{Notation for two-species first-order models.}
\label{tab:1stOrder_def} 
\small{
\begin{tabular}{ c | c }
\hline
Variable & Definition \\
\hline\hline
$\bx_i(t)\in \R^d$ & state (position, opinion, etc.) of agent $i$ at time $t$ \\
\hline
$\|\cdot\|$ & Euclidean norm in $\R^d$ \\
\hline
$\br_{ii'}(t)\in \R^d$ & displacement $\bx_{i'}(t)-\bx_{i}(t)$ \\
\hline
$r_{ii'}(t)\in \R^+$ & distance $r_{ii'}(t)=\|\br_{ii'}(t)\|$ \\
\hline
$N$ & total number of agents ($N=N_1+N_2$) \\
\hline
$N_{k}$ & number of agents of type $k$ ($k=1,2$) \\
\hline
$C_k$ & set of indices of agents of type $k$ \\
\hline
$\phi^{pq}$ & kernel for the influence of type-$q$ agents on type-$p$ agents \\
\hline
\end{tabular}
}
\end{table}

We assume that the system \eqref{e:firstordersystem1}–\eqref{e:firstordersystem2} governs the dynamics of observed trajectories, and that the only unknown quantities are the interaction kernels $\{\phi^{pq}\}_{p,q=1}^{2}$. The types of agents are known. Our objective is to infer these kernels from observed trajectory data and to establish convergence guarantees for the estimators.  

For compactness, the system can be written as
\begin{equation}
 \dot{\bX}(t) = \mathcal{F}_{\bintkernel}(\bX(t)),
 \label{eq:firstOrder_compact}
\end{equation}
where $\bX(t):= \begin{bmatrix}\bx_1(t)\\\cdots\\\bx_N(t) \end{bmatrix} \in \mathbb{R}^{dN}$ is the concatenated state vector and $\mathcal{F}_{\bintkernel}$ denotes the interaction operator determined by $\bintkernel=(\phi^{11},\phi^{12},\phi^{21},\phi^{22})$.  

The training data consist of sampled trajectories with positions and velocities,
\[
\{ \bX^{(m)}(t_\ell), \dot{\bX}^{(m)}(t_\ell)\}_{m=1,\ell=1}^{M,L}, 
\qquad 0=t_1<\dots<t_L=T,
\]
generated from $M$ independent initial conditions $\bX^{(m)}(0)$ drawn from a probability measure $\mu_0^{\bX}$ on $\R^{dN}$. We also consider the noisy setting in which velocity observations are corrupted by additive Gaussian noise:
\[
\dot{\bX}^{(m)}(t_\ell) = \mathcal{F}_{\bintkernel}(\bX^{(m)}(t_\ell)) + \mbf{\epsilon}^{(m,\ell)}, 
\qquad \mbf{\epsilon}^{(m,\ell)} \sim \mathcal{N}(0, \sigma^2 I_{dN}).
\]

The learning problem is therefore to recover the interaction kernels $\{\phi^{pq}\}_{p,q=1}^{2}$ from such observations. In what follows, we develop a Gaussian process framework to perform this inference and provide a rigorous convergence theory for the resulting estimators.  

\section{Methodology}
\label{sec: Methodology}

\subsection{ Learning approach based on GPs}
\subsubsection{Prior}
We place independent Gaussian process priors on each interaction kernel:
\begin{eqnarray}
\intkernel^{pq} &\sim& \mathcal{GP}(0,K_{\theta_{pq}}(r,r')), \qquad (p,q)\in \{1,2\}^2,
\end{eqnarray}
with covariance kernels $K_{\theta_{pq}}(\cdot,\cdot)$ parameterized by hyperparameters $\mbf{\theta} = \{\theta_{pq}\}_{p,q=1}^{2}$.

\begin{table}[H]
\centering
\caption{Notation for first-order systems.}
\label{tab:firstOrder_vecdef} 
\small{
\small{\begin{tabular}{ c | c }
\hline
Variable                    & Definition \\
\hline\hline
$\bX \in \R^{dN}$ &  vectorization of position vectors $(\bx_i)_{i=1}^{N}$\\
\hline 
$\br_{ij}, \br'_{ij}\in \R^d$ & $\bX_j-\bX_i$, $\bX'_j-\bX'_i$  \\
\hline
$r_{ij},r'_{ij}\in \R^+$ & $r_{ij}=\|\br_{ij}\|,  r'_{ij}=\|\br'_{ij}\|$ \\
\hline
$\rhsfo_{\intkernel^{pq}} \in \mathbb{R}^{dN_p}$ & \makecell{  interaction force field corresponding to the interaction kernel $\intkernel^{pq}$ }\\
\hline
$\rhsfo_{\bintkernel}$ & interaction force  field with  $\bintkernel = (\intkernel^{11},\intkernel^{12},\intkernel^{21},\intkernel^{22})$ \\
\hline
\end{tabular}}  
}
\end{table}
 
Because the force field $\rhsfo_{\bintkernel}$ is a linear functional of the kernels $\phi^{pq}$,
it follows that for any pair of system states $\bX,\bX'$, the induced forces
$\rhsfo_{\bintkernel}(\bX), \rhsfo_{\bintkernel}(\bX')$ are jointly Gaussian, 
 and
 \begin{equation}
 \begin{bmatrix}
 \rhsfo_{\bintkernel}(\bX)\\
 \rhsfo_{\bintkernel}(\bX')
 \end{bmatrix}
 \sim \mathcal{N} (\bm{0}, K_{\bintkernel}(\bX,\bX')),
\end{equation}
where $K_{\bintkernel}(\bX,\bX')$ is the covariance matrix  
\begin{eqnarray}
 \cov(\rhsfo_{\bintkernel}(\bX),\rhsfo_{\bintkernel}(\bX'))=\big(\cov([\rhsfo_{\bintkernel}(\bX)]_i,[\rhsfo_{\bintkernel}(\bX')]_j) \big)_{i,j=1,1}^{N,N},
\label{eqSigma2}
\end{eqnarray}
with $(i,j)$th block
\begin{eqnarray*}
 & & \cov([\rhsfo_{\bintkernel}(\bX)]_i,[\rhsfo_{\bintkernel}(\bX')]_j)=\\
 & &\begin{cases}
     \frac{1}{N^2}\big(\sum_{1\leq k, k' \leq N_1} K_{\theta_{11}}(r_{ik},r'_{jk'})\br_{ik}{\br'_{jk'}}^T + \sum_{N_1 < k, k' \leq N} K_{\theta_{12}}(r_{ik},r'_{jk'})\br_{ik}{\br'_{jk'}}^T\big) & 1\leq i,j \leq N_1,\\
     \frac{1}{N^2}\big(\sum_{1\leq k, k' \leq N_1} K_{\theta_{21}}(r_{ik},r'_{jk'})\br_{ik}{\br'_{jk'}}^T + \sum_{N_1 < k, k' \leq N} K_{\theta_{22}}(r_{ik},r'_{jk'})\br_{ik}{\br'_{jk'}}^T\big) & N_1 <i, j\leq N,\\  
     0 & otherwise.
 \end{cases}
\end{eqnarray*} {See Table \ref{tab:firstOrder_vecdef} for the definitions}. Note that when agent $i$ and agent $j$ are from different types, the covariance of $[\rhsfo_{\bintkernel}(\bX)]_i$ and $[\rhsfo_{\bintkernel}(\bX')]_j$ is zero due to the independence assumption of $\{\phi^{pq}\}$.
In summary, by \eqref{eq:firstOrder_compact}, the observation $\bZ = \dot \bX$ in the model follows the Gaussian distribution
 \begin{equation}
 \begin{bmatrix}
 \bZ\\ \bZ'
 \end{bmatrix}
 \sim \mathcal{N} (
 \mbf{0}
 , K_{\bintkernel}(\bX,\bX')).
\label{eqZdist2}
\end{equation}

\subsubsection{Training of hyperparameters }
\label{subsec:trainhyp}

Suppose that the training data consists of  $\bbX = [\bX^{(1,1)},\dots,\bX^{(M,L)}]^T \in \mathbb{R}^{dNML}$, and $\bbZ = [\bZ^{(1,1)}_{\sigma^2},\dots,\bZ^{(M,L)}_{\sigma^2}]^T \in \mathbb{R}^{dNML}$ where we used $\bX^{(m,l)}: =\bX^{(m)}(t_{l})$ and
\begin{equation}
    \bZ^{(m,l)}_{\sigma^2}  = \rhsfo_{\bintkernel}(\bX^{(m,l)}) + \mbf{\epsilon}^{(m,l)},
\end{equation}
with i.i.d noise $\mbf{\epsilon}^{(m,l)} \sim \mathcal{N}(0, \sigma^2 I_{dN})$. We then have
\begin{equation}
    \bbZ \sim \mathcal{N}(\mbf{0},   K_{\bintkernel}(\bbX,\bbX;\mbf{\theta}) + \sigma^2 I_{dNML}),
\end{equation}
where the covariance matrix $K_\intkernel(\bbX,\bbX;\mbf{\theta}) = \big(\cov(\rhsfo_{\bintkernel}(\bX^{(m,\ell)}),\rhsfo_{\bintkernel}(\bX^{(m',\ell')})) \big)_{m,m',\ell,\ell'=1,1,1,1}^{M,M,L,L} \in \mathbb{R}^{dNML \times dNML}$ can be computed by using \eqref{eqSigma2}.\\
Therefore, we can train the hyperparameters $\theta$ by maximizing the probability of the observational data, which is equivalent to minimizing the negative log marginal likelihood (NLML) (see Chapter 4 in \cite{williams2006gaussian})
\begin{eqnarray}
    -\log p(\bbZ|\bbX,\mbf{\theta},\sigma^2) &=& \frac{1}{2} \bbZ^T(K_{\bintkernel}(\bbX,\bbX;\mbf{\theta}) + \sigma^2I)^{-1}\bbZ \notag\\ &\ & \qquad +\frac{1}{2}\log|K_{\bintkernel}(\bbX,\bbX;\mbf{\theta})+\sigma^2I| + \frac{dNML}{2} \log 2\pi.
\label{eq:likelihood}
\end{eqnarray}
To solve for the hyperparameters $(\mbf{\theta},\sigma)$, we can apply conjugate gradient (CG) optimization (see Chapter 5 in \cite{williams2006gaussian} ) to minimize the negative log marginal likelihood using the fact that the partial derivatives of the marginal likelihood w.r.t. the hyperparameters can be computed by
\begin{equation}
    \frac{\partial}{\partial \theta_{pq}} \log p(\bbZ|\bbX,\mbf{\theta},\sigma^2) = \frac{1}{2} \mathrm{Tr}\left( (\bgamma \bgamma^T - (K_{\bintkernel}(\bbX,\bbX;\mbf{\theta}) + \sigma^2I)^{-1}) \frac{\partial K_{\bintkernel}(\bbX,\bbX;\mbf{\theta})}{\partial \theta_{pq}}\right),
\label{eqtheta}
\end{equation}
\begin{equation}
\frac{\partial}{\partial \sigma} \log p(\bbZ|\bbX,\mbf{\theta},\sigma^2) =  \mathrm{Tr}\left( (\bgamma \bgamma^T - (K_{\bintkernel}(\bbX,\bbX;\mbf{\theta}) + \sigma^2I)^{-1}) \right)\sigma.
\label{eqsigma}
\end{equation}
where $\bgamma = (K_{\bintkernel}(\bbX,\bbX;\mbf{\theta})+ \sigma^2I)^{-1} \bbZ$.
\\
The marginal likelihood does not simply favor the models that fit the training data best, but induces an automatic trade-off between data-fit and model complexity \cite{rasmussen2001occam}. This flexible training procedure distinguishes Gaussian processes from other kernel-based methods \cite{tipping2001sparse,scholkopf2002learning,vapnik2013nature} and regularization based approaches \cite{tihonov1963solution,tikhonov2013numerical,poggio1990networks}.

\begin{table}[H]
\caption{Notation for covariances.}
\label{tab:methodology_def} 
\centering
\resizebox{0.8\textwidth}{!}{%
\begin{tabular}{ c | c }
\hline
Variable                    & Definition \\
\hline\hline
$K_{\mbf{\theta}}(\cdot,\cdot)$                         &  covariance kernel function  with parameters $\mbf{\theta}$ \\
\hline
$K_{\theta_{pq}}(\cdot,\cdot)  $                      &  covariance kernels for modelling $\intkernel^{pq}$\\
\hline
$K_{\rhsfo_{\intkernel}}(\cdot,\cdot)$  & \makecell{ covariance function between $\rhsfo_{\intkernel}(\cdot)$ and $\rhsfo_{\intkernel}(\cdot)$} \\
\hline
\makecell{$K_{\bintkernel,\intkernel^{pq}}(\cdot,\cdot):= K_{\intkernel^{pq},\bintkernel}(\cdot,\cdot)^T$} &   covariance function between $\rhsfo_\bintkernel(\cdot)$ and $\intkernel^{pq}(\cdot)$ \\
\hline
\end{tabular} 
}
\end{table}

\subsubsection{Prediction}

 After the training procedure, we obtain updated priors on the interaction kernel functions. We first show how to predict the value $\intkernel^{pq}(r^*)$ using the mean of  its posterior distribution. Note that 
 
  \begin{equation}
    \begin{bmatrix}
    \rhsfo_{\bintkernel}(\bbX)\\
    \intkernel^{pq}(r^\ast)
    \end{bmatrix}
    \sim \mathcal{N} \left( 0,
    \begin{bmatrix}
    K_{\bintkernel}(\bbX, \bbX) & K_{\bintkernel,\intkernel^{pq}}(\bbX, r^\ast)\\
    K_{\intkernel^{pq},\bintkernel}(r^\ast, \bbX) & K_{\theta_{pq}}(r^\ast,r^\ast)
    \end{bmatrix}
    \right),
\end{equation} 
where $K_{\bintkernel,\intkernel^{pq}}(\bbX, r^*) = K_{\intkernel^{pq},\bintkernel}(r^*,\bbX)^T$ denotes the covariance function between $\rhsfo_{\bintkernel}(\bbX)$ and $\intkernel^{pq}(r^*)$ which can be computed elementwise by 
\begin{eqnarray*}
 & & \cov([\rhsfo_{\bintkernel}(\bX)]_i,\intkernel^{pq}(r^*))=\\
 & &\begin{cases}
     \frac{1}{N}\big(\sum_{1\leq k \leq N_1} K_{\theta_{11}}(r_{ik},r^{\ast})\br_{ik}\big) & 1 \leq i \leq N_1, p=1, q =1,\\
     \frac{1}{N}\big(\sum_{N_1 < k\leq N} K_{\theta_{12}}(r_{ik},r^{\ast})\br_{ik}\big) & 1 \leq i \leq N_1, p=1, q =2,\\
     \frac{1}{N}\big(\sum_{1\leq k \leq N_1} K_{\theta_{21}}(r_{ik},r^{\ast})\br_{ik}\big) & N_1 < i \leq N, p=2, q =1,\\
     \frac{1}{N}\big(\sum_{N_1 < k\leq N} K_{\theta_{22}}(r_{ik},r^{\ast})\br_{ik}\big) & N_1 < i \leq N, p=2, q =2,\\
     0 & otherwise.
 \end{cases}
 \end{eqnarray*}

Thus, conditioning on $\rhsfo_{\bintkernel}(\bbX)$, we obtain 
\begin{equation}
    p(\intkernel^{pq}(r^\ast)|\bbX,\bbZ,r^\ast) \sim \mathcal{N}(\bar{\intkernel}^{pq}(r^{\ast}),var(\intkernel^{pq}(r^{\ast}))),
\end{equation}
where
\begin{equation}
    \bar{\intkernel}^{pq}(r^{\ast}) = K_{\intkernel^{pq},\bintkernel}(r^\ast,\bbX)(K_{\bintkernel}(\bbX,\bbX) + \sigma^2I)^{-1} \bbZ,
\label{eq:estimated phie}  
\end{equation}
\begin{equation}
    var(\intkernel^{pq}(r^{\ast})) = K_{\theta_{pq}}(r^\ast,r^\ast) - K_{\intkernel^{pq},\bintkernel}(r^\ast,\bbX)(K_{\bintkernel}(\bbX,\bbX) + \sigma^2I)^{-1}K_{\bintkernel,\intkernel^{pq}}(\bbX,r^\ast).
    \label{eq:estimated var phie}
\end{equation}
The posterior variance $var(\intkernel^{pq}(r^{\ast}))$ can be
used as a good indicator for the uncertainty of the estimation $\bar{\intkernel}^{pq}(r^{\ast})$ based on our Bayesian approach.\\

\noindent
Moreover, using the estimated interaction kernels $\hat\bintkernel(r^\ast) := \{\bar{\intkernel}^{pq}(r^{\ast})\}$, we can predict the dynamics based on the equations
\begin{equation}
 \hat\bZ(t) = \rhsfo_{\hat\bintkernel}(\bX(t)).
 \end{equation}

We have applied this approach to various examples and achieved superior empirical performance. We refer the reader to Section \ref{sec:numerics} for the detailed numerical results and their analysis. For error analysis on the trajectory prediction errors, one can use Theorem 9 in \cite{miller2020learning} and we skip the step here.

\begin{algorithm}[H]
\algorithmicrequire\ $(\bbX, \bbZ)$ (training data), $r^\ast$ (test point), $K_{\theta_{pq}}$ (covariance function), $\rhsfo_{\bintkernel}$ (interaction functions)
\begin{algorithmic} [1]
\STATE $(\hat{\mbf{\theta}},\hat\sigma^2) = \argmin{\mbf{\theta},\sigma^2} -\log p(\bbZ|\bbX,\mbf{\theta},\sigma^2)$\\
\hfill\COMMENT{solve for parameters by minimizing NLML using CG and \eqref{eq:likelihood}-\eqref{eqsigma} }
\STATE $L := \textrm{cholesky}(K_{\bintkernel}(\bbX,\bbX) + \hat\sigma^2I)$
\STATE $\gamma:= L^T\backslash(L\backslash \bbZ)$
\STATE $K_{pq}^\ast := K_{\bintkernel, \intkernel^{pq}}(\bbX, r^\ast)$ \hfill\COMMENT{compute covariances between $\rhsfo_{\bintkernel}(\bbX)$ and $\intkernel^{pq}(r^\ast)$}
\STATE $\bar{\intkernel}^{pq}(r^{\ast}) : = (K_{pq}^\ast)^T \gamma$
\hfill\COMMENT{predictive mean \eqref{eq:estimated phie}}
\STATE $\bv_{pq} = L \backslash K_{pq}^\ast$
\STATE $var(\intkernel^{pq}(r^{\ast})) := K_{\hat\theta_{pq}}(r^\ast,r^\ast) -  \bv_{pq}^T \bv_{pq}$
\hfill\COMMENT{predictive variance \eqref{eq:estimated var phie}}
\end{algorithmic}
\algorithmicensure\ $\bar{\intkernel}^{pq}(r^{\ast})$ (mean), $var(\intkernel^{pq}(r^{\ast}))$ (variance)
\caption{{\bf Predictions} \label{Algorithm}}

\end{algorithm}

\section{Learning theory} 
\label{sec:theory}

Our numerical results in Section \ref{sec:numerics} show that the interaction kernels in various systems can be learned very well from a small amount of noisy data. These results demonstrate the effectiveness of the Gaussian process approach.

In this section,  we assume that the interaction kernels are assigned Gaussian priors $\mathcal{GP}(0, \tilde{K}^{pq})$, and focus on the prediction step. Our goal is to establish a learning theory which analyzes both the performance of the posterior mean  \eqref{eq:estimated phie} that approximates the true interaction kernel and the marginal posterior variance  \eqref{eq:estimated var phie} that provides a pointwise quantification of uncertainty.

For ease of notation, we rewrite the system as 
\begin{align}\label{firstorder}
 \dot\bX(t) &= \big(\dot\bX_1(t), \dot\bX_2(t)\big)^\top
 = \rhsfo_{\bintkernel}(\bX(t)) \\
 &= \Big(\rhsfo_{\phi^{11}}(\bX_1(t)) + \rhsfo_{\phi^{12}}(\bX(t)),
         \;\rhsfo_{\phi^{21}}(\bX(t)) + \rhsfo_{\phi^{22}}(\bX_2(t))\Big)^\top,
\end{align}
where $\bX_1 = (\bx_1,\dots,\bx_{N_1})^T$, $\bX_2 = (\bx_{N_1+1},\dots,\bx_{N})^T$, and $\rhsfo_{\bintkernel}:\mathbb R^{dN}\to\mathbb R^{dN}$.

\paragraph{GP estimators for two-type agent systems} In two-type agent systems,  the noisy trajectory dataset is given as  
\begin{align}\label{empiricaldata}
\{\bbX_M, \bbZ_{\sigma^2, M}\}
\end{align} with 
 
 \begin{align*}
\bbX_{M}&= \mathrm{Vec}\big(\{\bX^{(m,l)}\}_{m=1,l=1}^{M,L}\big) \in \mathbb{R}^{dNML},\\
 \bbZ_{\sigma^2,M}&= \mathrm{Vec}\big(\{\dot\bX^{(m,l)}+\sigma \mbf{\epsilon}^{(m,l)}\}_{m=1,l=1}^{M,L}\big) = \mathrm{Vec}\big(\{\rhsfo_{\bintkernel}(\bX^{(m,\ell)}) + \sigma\,\mbf{\epsilon}^{(m,\ell)}\}_{m=1,l=1}^{M,L}\big) \in \mathbb{R}^{dNML}
 \end{align*} where we observe the dynamics at $0=t_1< t_2<\cdots<t_L=T$; $m$ indexes trajectories corresponding to different initial conditions at $t_1=0$; $\bX^{(m,1)} \stackrel{i.i.d}{\sim} \mu_0^\bx$,
 $\mu_0^\bx$ is a probability measure on $\mathbb{R}^{dN}$; and $\mbf{\epsilon}^{(m,l)}  \stackrel{i.i.d}{\sim} \mathcal{N}(\bm{0}, I_{dN})$ is the noise term where we assume that  $\mu_0^\bx$ is independent of the  distribution of noise. We let
\begin{align*}
\bbX^1_{M}&= \mathrm{Vec}\big(\{\bX_1^{(m,l)}\}_{m=1,l=1}^{M,L}\big) \in \mathbb{R}^{dN_1ML}, \quad \bbX^2_{M}= \mathrm{Vec}\big(\{\bX_2^{(m,l)}\}_{m=1,l=1}^{M,L}\big) \in \mathbb{R}^{dN_2ML},\\
 \bbZ^1_{\sigma^2,M}&= \mathrm{Vec}\big(\{\dot\bX_1^{(m,l)}+\sigma\mbf{\epsilon}_1^{(m,l)}\}_{m=1,l=1}^{M,L}\big) \in \mathbb{R}^{dN_1ML},\\
 \bbZ^2_{\sigma^2,M} &= \mathrm{Vec}\big(\{\dot\bX_2^{(m,l)}+\sigma\mbf{\epsilon}_2^{(m,l)}\}_{m=1,l=1}^{M,L}\big) \in \mathbb{R}^{dN_2ML}
 \end{align*}
with $\mbf{\epsilon}_p^{(m,\ell)}\stackrel{\text{i.i.d.}}{\sim}\mathcal N(\mathbf 0,I_{dN_p})$ independent across $p,m,\ell$.

We now recast the learning approach for two-agent systems. We place independent GP priors $\phi^{pq}\sim\mathcal{GP}(0,\tilde K^{pq})$ on $[0,R]$, with Mercer kernels $\tilde{K}^{pq}$ defined on $[0,R]\times [0,R]$ which may be dependent on the size of the observational data.
Conditioning on data $\{\bbX_M,\bbZ_{\sigma^2,M}\}$, the posterior mean for $\phi^{pq}(r^\ast)$ is
\begin{equation}\label{secondorder:pos}
\bar{\phi}^{pq}_{M}(r^\ast)
= \tilde K_{\phi^{pq},\,\bintkernel}(r^\ast,\bbX_M)\,\big(\tilde K_{\bintkernel}(\bbX_M,\bbX_M)+\sigma^2 I\big)^{-1}\,\bbZ_{\sigma^2,M},
\end{equation}
where the matrices $\tilde{K}_{\intkernel^{pq},\bintkernel}(r^\ast,\bbX_M)$ and $\tilde{K}_{\bintkernel}(\bbX_M,\bbX_M)$ denote the covariance function between $\rhsfo_{\bintkernel}(\bbX_M)$ and $\intkernel^{pq}(r^{\ast})$, and  $\rhsfo_{\bintkernel}(\bbX_M)$ and $\rhsfo_{\bintkernel}(\bbX_M)$ respectively. That is
\begin{eqnarray}
   \tilde{K}_{\intkernel^{pq},\bintkernel}(r^\ast, \bbX_M)&=& \tilde{K}_{\bintkernel,\intkernel^{pq}}(\bbX_M, r^\ast)^T= \cov(\intkernel^{pq}(r^\ast),\rhsfo_{\bintkernel}(\bbX_M))\in \mathbb{R}^{1\times dNML},\\
    \tilde{K}_{\bintkernel}(\bbX_M, \bbX_M)&=&\cov(\rhsfo_{\bintkernel}(\bbX_M),\rhsfo_{\bintkernel}(\bbX_M)) \in \mathbb{R}^{dNML \times dNML}.
   \end{eqnarray}
The marginal posterior covariance that provides a quantification of uncertainty for prediction of $\intkernel^{pq}$ at the point $r^{\ast} \in \mathbb{R}^+$ is given by 
\begin{eqnarray}\label{secondorder:var}
    \mathrm{Var}( \intkernel^{pq}_M(r^{\ast})|\bbZ_{\sigma^2,M}) = \tilde{K}_{\intkernel^{pq}}(r^\ast,r^\ast) - \tilde{K}_{\intkernel^{pq},\bintkernel}(r^\ast,\bbX_M)(\tilde{K}_{\bintkernel}(\bbX_M,\bbX_M) + \sigma^2I)^{-1}\tilde{K}_{\bintkernel,\intkernel^{pq}}(\bbX_M,r^\ast)
    . 
\end{eqnarray} 

\subsection{Connection with inverse problem} 
\label{inverseproblem}
\paragraph{Relevant function spaces}We introduce a probability measure on $\mathbb{R}^{dN}$:
\begin{align}\label{rhox}
\rho_{\bX}:=\mathbb{E}_{\bX(0) \sim \mu_0^\bx}\bigg[\frac{1}{L}\sum_{l=1}^{L}\delta_{\bX(t_l)}\bigg],
\end{align} where $\delta$ is the Dirac $\delta$ distribution and $\bX(t_l) \in \mathbb{R}^{dN}$ is the position vector  of all agents at time $t_l$. 

We  introduce an associated $L^2$ space, denoted by  $L^2(\mathbb{R}^{dN};\rho_{\bX};\mathbb{R}^{dN})$. For two functions $\mbf{f}=[\mbf{f}_1,\cdots,\mbf{f}_N]^T$ and $\mbf{g}=[\mbf{g}_1,\cdots,\mbf{g}_N]^T$ with the components $\mbf{f}_i, \mbf{g}_i: \mathbb{R}^{dN} \rightarrow \mathbb{R}^d$ for $i=1,\cdots,N$, their inner product is defined by 
$$\Rhoxinnerp{\mbf{f}}{\mbf{g}}=\frac{1}{N}\sum_{i=1}^{N}\int_{\mathbb{R}^{dN}}\langle \mbf{f}_i(\bX),  \mbf{g}_i(\bX)\rangle d\rho_{\bX}.$$

 Let $\mK$ be a Mercer kernel that is defined on $[0,R]\times [0,R]$ and $\mH$ be the RKHS associated to $\mK$. 
\begin{assumption}\label{assump1}
We assume that the true interaction functions $\intkernel^{pq} \in \mHpq$, and $$\kappa^2_{pq}=\mathrm{sup}_{r\in [0,R]} {\mK^{pq}}(r,r) <\infty.$$
\end{assumption}

Recall that we require the interaction function $\intkernel^{pq}$ to lie in $W_c^{1,\infty}([0,R])$ to ensure the well-posedness of the system \eqref{firstorder}. Therefore, it is reasonable to assume that the true kernel lies in $\prod_{p,q} \mHpq$. For example, we can choose a Mat\'ern kernel whose associated RKHS contains $W_c^{1,\infty}([0,R])$ as a subspace. 
\begin{lemma}\label{infbound} By Assumption 1, we have that, for any $\bintkernelvar = (\intkernelvar_{pq}) \in \prod_{p,q} \mHpq$, there holds $\|\intkernelvar_{pq}\|_{\infty}\leq \kappa_{pq} \|\intkernelvar_{pq}\|_{\mHpq}$.
\end{lemma}
\begin{proof}By the reproducing property of $\mK^{pq}$, we have that $$|\intkernelvar_{pq}(r)|=|\langle \intkernelvar_{pq}, \mK^{pq}_{r}\rangle_{\mHpq}|\leq \|\intkernelvar_{pq}\|_{\mHpq}\|\mK^{pq}_{r}\|_{\mHpq}\leq \kappa_{pq} \|\intkernelvar_{pq}\|_{\mHpq}.$$ 
The conclusion follows.  
\end{proof}

\paragraph{Formulation of the inverse problem.} Now we define a linear operator $A: \prod_{p,q} \mHpq \rightarrow  L^2(\mathbb{R}^{dN};\rho_{\bX};\mathbb{R}^{dN})$ by 
 \begin{align}\label{equationinverse}
 A\bintkernelvar=\rhsfo_{\bintkernelvar},
 \end{align} where $\rhsfo_{\bintkernelvar}$ is the right hand side of system \eqref{firstorder} by replacing $\bintkernel$ with $\bintkernelvar$.  Then $A$ is a bounded linear operator (see Proposition \ref{propertyA}).  In the case of ``infinite data",  our learning problem is equivalent to solving the linear equation \eqref{equationinverse} in $\prod_{p,q} \mHpq$ given $A$ and $ \rhsfo_{\bintkernel}$ in $ L^2(\mathbb{R}^{dN};\rho_{\bX};\mathbb{R}^{dN})$ and therefore is an inverse problem. 
 
However, this inverse problem may be ill-posed. This happens when the solution is not unique or does not depend continuously on  $\rhsfo_{\bintkernel}$. The uniqueness of the solution is not obvious. This can be seen from the heuristic argument: the interaction kernels $\intkernel^{pq}$ depend only on one variable, but are observed through a collection of non-independent linear measurements with values $\dot{\bx}_i$, the l.h.s. of \eqref{e:firstordersystem1},\eqref{e:firstordersystem2}, at locations $r_{ii'}:=\|\bx_{i'}-\bx_i\|$, with coefficients $\br_{ii'}:=\bx_{i'}-\bx_i$.  One could attempt to recover $\{\intkernel^{pq}(r_{ii'}) \}_{i,i'}$ from the  equations of $\dot{\bx_i}$'s by solving  the corresponding linear system. Unfortunately, this linear system is usually underdetermined as $dN$ (number of known quantities) $\leq 2N(N-1)$ (number of unknowns) and in general one will not be able to recover the values of $\intkernel^{pq}$ at locations $\{r_{ii'}\}_{i,i'}$.
 
 In the context of inverse problems, to overcome the possible ill-posedness, one may  introduce the Tikhonov regularization \cite{tikhonov2013numerical}  term to solve the regularized least squares problem 
  \begin{align}\label{equationinverseregularizer}
\argmin{\bintkernelvar \in \prod_{p,q} \mHpq}\Rhoxnorm{A\bintkernelvar-\rhsfo_{\bintkernel}}^2+ \sum_{p,q} \lambda^{pq}\|\intkernelvar_{pq}\|_{\mHpq}^2, \quad\lambda^{pq}>0
 \end{align} 
 
Later in this paper, we show that, with an appropriate Gaussian prior,  our posterior mean estimator \eqref{secondorder:pos} is in fact the solution to the empirical version of the risk \eqref{equationinverseregularizer}. We further derive a  Representer theorem (Theorem \ref{maingp}) to show the posterior mean estimators are in fact linear combinations of the kernel functions $K_r^{pq}$, where $r$ ranges in pairwise distances of agents coming from the observational data, confirming the intuition that $\intkernel^{pq}$ are being learned at the pairwise distances.

\subsubsection{Well-posedness by a coercivity condition}\label{wellcoercivity} In our numerical experiments, we find that our estimators produce faithful approximations to the ground truth and the accuracy significantly improves with additional data. This motivates us to study under which conditions the inverse problem is well-posed and verify that this condition is generically satisfied. 

Note that the observational variables for the interaction kernel $\intkernel^{pq}$ consist of pairwise distances, in \cite{lu2019nonparametric,lu2021learning},  a probability measure on $\mathbb{R}^+$ that encodes the information about the dynamics marginalized to pairwise distance can be introduced as the following:
let $I_1:=\{1,\dots,N_1\}$ and $I_2:=\{N_1+1,\dots,N\}$. For $(p,q)\in\{1,2\}^2$, define
\[
\mathcal P_{pq}:=\big\{(i,i'):\, i\in I_p,\; i'\in I_q,\; i'\neq i\big\},\qquad
Z_{pq}:=
\begin{cases}
N_p(N_p-1), & p=q,\\
N_p N_q, & p\neq q.
\end{cases}
\]
The pairwise-distance law (marginalized dynamics) is the probability measure on $\mathbb R^+$:
\begin{equation}\label{rho_def_equation}
\rho_T^{pq,L}(dr) \;:=\; \frac{1}{LZ_{pq}}\sum_{\ell=1}^{L}
\sum_{(i,i')\in \mathcal P_{pq}} \;
\mathbb E_{\bX(0)\sim\mu_0^\bx}\big[\,\delta_{\,r_{ii'}(t_\ell)}(dr)\,\big],
\end{equation}
where $\delta$ is the Dirac $\delta$ distribution, so that $\E_{\mu_0^\bx}[\delta_{r_{ii'}(t)}(dr)]$ is the distribution of the random variable $r_{ii'}(t)= || \bx_{i}(t) -\bx_{i'}(t) ||$, with $\bx_{i}(t)$ being the position of particle $i$ at time $t$. 

The probability measure $\rho_T^{pq,L}$ depends on the distribution of initial conditions $\mu_0^\bx$ while it is independent of the observed data. Note that it is on the support of $\rho_T^{pq,L}$ that $\intkernel^{pq}$ could be learned. Without loss of generality, we assume that $\rho_T^{pq,L}$ is non-degenerate on $[0,R]$\footnote{For example, we can choose $\mu_0^\bx:=\mathrm{Unif}[-\frac{R}{2}, \frac{R}{2}]^{dN}$. Then $\mathrm{Supp}(\rho_T^{pq})=[0,R]$ and $\mathrm{Supp}(\rho_T^{pq})\subset \mathrm{Supp}(\rho_T^{pq,L})$ for $L>1$. }. Due to the structure of the equation, we introduce a positive measure that appears naturally in estimating the error of estimators 

\begin{equation}
d\tilde\rho_T^{pq,L}(r) \;:=\; r^2\, d\rho_T^{pq,L}(r)\quad\text{on }[0,R]. \label{eq:tilderhoe}
\end{equation}

To ensure the well-posedness, we require that $\bintkernel = (\intkernel^{pq})$ is the unique solution to \eqref{equationinverse}, so $A$ has to be injective. Now we introduce a sufficient condition to guarantee the injectivity of the operator $A$. Due to Assumption \ref{assump1}, $\prod_{p,q} \mH^{pq}$ can be naturally embedded as a subspace of  $L^2([0,R];\tilde\rho_T^{L};\mathbb{R}\times\mathbb{R})$.
\begin{definition}
 We say that the system \eqref{firstorder} satisfies the coercivity condition if there exist constants $c_{\mathcal H^{pq}}>0$ such that $ \forall \bintkernelvar \in \prod_{p,q} \mHpq$,
\begin{align}\label{coercivity}
\|A\bintkernelvar\|^2_{L^2(\rho_{\bX})}=\|\rhsfo_{\bintkernelvar}\|^2_{L^2(\rho_{\bX})}\geq \sum_{p,q }c_{\mHpq}\|\intkernelvar_{pq}\|^2_{L^2( \tilde\rho_T^{pq,L})}. \end{align}
\end{definition}

Then if $A\bintkernelvar=0$ for $\bintkernelvar \in \prod_{p,q} \mHpq$, we conclude that $\bintkernelvar=0$ everywhere on $[0,R]^4$ due to non-degeneracy of $\tilde\rho_T^L$ on $\prod_{p,q} \mHpq$ and the continuity of $\bintkernelvar$. Therefore, $A$ is injective.  The coercivity condition introduces constraints on $\prod_{p,q} \mHpq$ and on the distribution of the solutions of the system, and it is therefore natural that it depends on the distribution $\mu_0$ of the initial condition $\bX(0)$, and the true interaction kernel $\bintkernel$.

When $L=1$, a concrete instance satisfying \eqref{coercivity} appears as Proposition~13 in~\cite{lu2021learning}. Theorem~C.1 of~\cite{feng2024data} establishes an analogous coercivity condition for the joint learning of energy and alignment-based interaction kernels; the same reasoning extends to our setting. Related notions of identifiability have been investigated in~\cite{miller2020learning}, which proves recoverability of structured combinations of interaction kernels in second-order heterogeneous models; see~\cite{tang2024identifiability} for an extension to manifold domains. Compared with~\cite{miller2020learning}, \eqref{coercivity} requires a stronger, kernel-level identifiability: it aims to recover each $\phi^{pq}$ individually rather than merely their aggregate effect. For $L>1$, the main analytic difficulty arises from implicit correlations among the pairwise empirical measures, which break the independence structure available in the $L=1$ case.

 Finally, we conjecture that \eqref{coercivity} is generically satisfied for a broad class of multi-species interacting systems under sufficiently rich initial conditions for $L\geq 1$, a view supported by our numerical learning results, while a rigorous characterization is left to future work.

\subsection{Connection with the Kernel Ridge Regression (KRR)}
 \label{kernelridgeregression}
 When applying the Gaussian process approach to solve classical nonparametric regression problems, we understand the posterior mean and marginal posterior variance by leveraging the connection with Kernel Ridge Regression (KRR):   the posterior mean can be viewed as a KRR estimator to solve a regularized least square empirical risk functional. The marginal posterior variance can be intriguingly interpreted as the bias of a noise-free KRR estimator \cite{yang2017frequentist,kanagawa2018gaussian}.

Our learning problem shifts the regression target function to $\bintkernel$ with dependent observational data and therefore departs from the classical setting.  In this section, we show that the posterior mean and marginal posterior variance obtained in \eqref{secondorder:pos} and  \eqref{secondorder:var}  still coincides with KRR estimators for a suitable regularized least square risk functional, which generalizes the classical facts. We present the main result below:

\begin{theorem}\label{maingp} Given the noisy trajectory data $\bbZ_{\sigma^2,M}$ \eqref{empiricaldata}, if  $\intkernel^{pq} \sim \mathcal{GP} (0, \tilde K^{pq})$, with $ \tilde K^{pq}=\frac{\sigma^2 K^{pq}} {MNL\lambda^{pq}}$ for some $\lambda^{pq}$, $p,q = 1,2$, then 
\begin{itemize}
\item  the posterior mean $\bar\phi_M = (\bar\phi_M^{pq})$ in \eqref{secondorder:pos} coincides with the  KRR estimator $\phi_{\prod_{p,q} \mHpq}^{\lambda,M}$ to the regularized empirical least square risk functional  
\begin{align}
\phi_{\prod_{p,q} \mHpq}^{\lambda,M}& := (\phi_{\mHpq}^{pq,\lambda,M}): =\argmin{\bintkernelvar\in \prod_{p,q} \mHpq}\mE^{\lambda,M}(\bintkernelvar),\\
\mE^{\lambda,M}(\bintkernelvar):&=\frac{1}{LM}\sum_{l=1,m=1}^{L,M}\| \rhsfo_{\bintkernelvar}(\bX^{(m,l)})-\bZ_{\sigma^2}^{(m,l)}\|^2+\sum_{p,q}\lambda^{pq} \|\intkernelvar^{pq}\|_{\mHpq}^2. \label{regularizedrisk1}
\end{align}
where $\rhsfo_{\bintkernelvar}(\bX^{(m,l)}) = (\rhsfo_{\intkernelvar^{11}}(\bX_1^{(m,l)}) + \rhsfo_{\intkernelvar^{12}}(\bX^{(m,l)}), \rhsfo_{\intkernelvar^{21}}(\bX^{(m,l)}) + \rhsfo_{\intkernelvar^{22}}(\bX_2^{(m,l)}))^T$.

\item the marginal posterior variance  \eqref{secondorder:var} can be written as 
\begin{align}
\mathrm{Var}(\phi_M^{pq}(r_*) | \bbZ_{\sigma^2,M})&=\frac{\sigma^2}{ML\lambda^{pq} N}[K_{r_*}^{pq}(r_*)-K_{r_*}^{pq,\lambda^{pq},M}(r_*)],
\end{align} where $K_{r_*}^{pq}(\cdot):= K^{pq}(r_*,\cdot)$, and
$K_{r_*}^{pq,\lambda^{pq},M}$ are the minimizers to the empirical regularized risk functional 
\begin{align}\label{kr}
\argmin{\bintkernelvar\in \prod_{p,q} \mHpq}\frac{1}{LM}\sum_{l=1,m=1}^{L,M}& \left\|\rhsfo_{\bintkernelvar}(\bX^{(m,l)})-
\begin{bmatrix}
    \rhsfo_{K_{r_*}^{11}}(\bX_1^{(m,l)}) + \rhsfo_{K_{r_*}^{12}}(\bX^{(m,l)})\\
    \rhsfo_{K_{r_*}^{21}}(\bX^{(m,l)}) + \rhsfo_{K_{r_*}^{22}}(\bX_2^{(m,l)})
\end{bmatrix}\right\|^2
 \notag\\
&\quad +\sum_{pq} \lambda^{pq} \|\intkernelvar^{pq}\|_{\mHpq}^2.
\end{align}

\end{itemize}

\end{theorem}

We prove Theorem \ref{maingp} by deriving a Representer theorem (see Appendix \ref{appendix:operator}) for the empirical risk functional \eqref{regularizedrisk1}, which is also applicable to the risk functional \eqref{kr}. 

\subsection{Non-asymptotic analysis of reconstruction error}

In this subsection, we shall assume that $\intkernel^{pq}\sim \mathcal{GP}(0,\tilde K^{pq})$ with $ \tilde K^{pq}=\frac{\sigma^2 K^{pq}} {MNL\lambda^{pq}}$  $(\lambda^{pq}>0)$ and the coercivity condition \eqref{coercivity} holds.  Thanks to Theorem \ref{maingp}, it suffices to analyze the performance of  KRR estimators $\phi^{\lambda,M}_{\prod_{p,q} \mHpq}$ and $K_{r_*}^{pq,\lambda^{pq},M}$.  In the context of learning theory for KRR,  it is typical to analyze the residual error, which in our case is given by $\|A\phi^{\lambda,M}_{\prod_{p,q} \mHpq}-A\bintkernel\|^2_{L^2(\rho_{\bX})}$ (see Corollary \ref{residual}). The coercivity condition  \eqref{coercivity} implies that this residual error is equivalent to $\|\phi^{\lambda,M}_{\prod_{p,q} \mHpq}-\bintkernel\|^2_{L^2( \tilde\rho_T^{L})}$. In \cite{lu2019nonparametric}, the authors proposed a learning approach for noise-free trajectory data, based on least squares, and show that the estimators can achieve the min-max optimal convergence rate in  $M$ with respect to the ${L^2( \tilde\rho_T^{L})}$ norm. In this paper, we focus on the reconstruction error  $\|\phi^{\lambda,M}_{\mH}-\bintkernel\|_{\prod_{p,q} \mHpq}$, which is typically analyzed in the context of inverse problems. We shall perform a non-asymptotic analysis as $M$ and $\lambda=(\lambda^{pq})$ varies. In particular, we show that by an appropriate choice of $\lambda$, one can achieve the convergence rate in $\prod_{p,q} \mHpq$ norm that coincides with the classical setting.  The developed theoretical framework is also applicable for analyzing the reconstruction errors $\|K_{r_*}^{pq,\lambda^{pq},M}- K_{r_*}^{pq}\|_{\mHpq}$, which provides an upper bound on worst case $L^\infty$ error of marginal posterior variance.

Our analysis is based on the decomposition of the reconstruction error as the sum of two types of errors $$\bintkernel^{\lambda,M}_{\prod_{p,q} \mHpq}-\bintkernel =\underbrace{ \bintkernel^{\lambda,M}_{\prod_{p,q} \mHpq}-\bintkernel^{\lambda,\infty}_{\prod_{p,q} \mHpq}}_{\text{Sample error}}+\underbrace{\bintkernel^{\lambda,\infty}_{\prod_{p,q} \mHpq}-\bintkernel}_{\text{Approximation error}}.$$ 

\paragraph{Analysis of sample error} We employ the operator representation: 
\begin{align*}
\bintkernel^{\lambda,M}_{\prod_{p,q} \mHpq}&=(B_M+\lambda)^{-1}A_{M}^{*}\bbZ_{\sigma^2,M}\\
&=\underbrace{(B_M+\lambda)^{-1}B_M\bintkernel}_{\tilde\bintkernel_{\prod_{p,q} \mHpq}^{\lambda,M}}+\underbrace{(B_M+\lambda)^{-1}A_{M}^{*}\mathbb{W}_M}_{\text{Noise term}},\\
\bintkernel^{\lambda,\infty}_{\prod_{p,q} \mHpq}&=(B+\lambda)^{-1}B\bintkernel, 
\end{align*} where $\tilde\bintkernel_{\prod_{p,q} \mHpq}^{\lambda,M}$ is the empirical minimizer of $\mE^{\lambda,M}(\cdot)$ for noise-free observations and $\mathbb{W}$ denotes the noise vector. 

We first provide non-asymptotic analysis of the sample error $\|(B_M+\lambda)^{-1}B_M\bintkernelvar-(B+\lambda)^{-1}B\bintkernelvar\|_{\prod_{p,q} \mHpq}$ for any $\bintkernelvar\in \prod_{p,q} \mHpq$, and apply it to $\bintkernel$ to obtain a bound on $\|\tilde\bintkernel_{\prod_{p,q} \mHpq}^{\lambda,M}-\bintkernel_{\prod_{p,q} \mHpq}^{\lambda,M}\|_{\prod_{p,q} \mHpq}$; then we estimate the ``noise part" $\bintkernel_{\prod_{p,q} \mHpq}^{\lambda,M}-\tilde\bintkernel_{\prod_{p,q} \mHpq}^{\lambda,M}$ to get the final result on the sample error shown below.

\begin{theorem}[$\mH$-bound] For any $\delta \in (0,1)$, it holds with probability at least $1-\delta$ that 
\begin{eqnarray}
& &\|\bintkernel_{\prod_{p,q} \mHpq}^{\lambda,M}-\bintkernel_{\prod_{p,q} \mHpq}^{\lambda,\infty}\|_{\prod_{p,q} \mHpq} \notag\\
&\lesssim&   \frac{8\kappa_{max} R^2\|\intkernel\|_{\infty}\sqrt{2\log(8/\delta)}}{\sqrt{M}\lambda_{min}}(C_{{\prod_{p,q} \mHpq}}+\frac{C_{\kappa,R,\lambda}\sqrt{2\log(8/\delta)}}{\sqrt{M\lambda_{min}}}) + \frac{8\kappa_{max} R\sigma \log(8/\delta)}{\sqrt{c}\lambda_{min} d \sqrt{MLN}}\notag\\
\end{eqnarray}
where $c$ is an absolute constant appearing in the Hanson-Wright inequality (Theorem \ref{HAnson}),  $\|\bintkernelvar\|_{\infty} = \max(\|\intkernelvar^{pq}\|_{\infty})$, $C_{{\prod_{p,q} \mHpq}}=2\sqrt{\frac{2}{c_{min}}}+1$,  $C_{\kappa,R,\lambda}=8\kappa_{max}R + 4\sqrt{\lambda_{min}}$, and $c_{min} = \min(c_{\mHpq})$, $\lambda_{min} = \min(\lambda^{pq})$, $\kappa_{max} = \max(\kappa^{pq})$.
\end{theorem}
For detailed proofs, refer to Appendix \ref{appendix:finitesample}.

\paragraph{Analysis of approximation error $\| \bintkernel_{\prod_{p,q} \mHpq}^{\lambda,\infty}-\bintkernel\|_{\prod_{p,q} \mHpq}$} 
To estimate the approximation error, we follow the standard argument in the literature of Tikhonov regularization, see Section 5 in  \cite{caponnetto2005fast}. By assuming the coercivity condition holds, and $\bintkernel \in \mathrm{Im}\, B^{\gamma}$ with $0< \gamma \leq \frac{1}{2}$, we can prove the following theorem.

\begin{theorem}[Convergence rate of reconstruction error in $\prod_{p,q} \mHpq$ norm] Assume the coercivity condition \eqref{coercivity} and $\bintkernel \in \mathrm{Im}(B^\gamma)$ for some $0<\gamma\le \tfrac12$. Choose $\lambda \asymp M^{-\frac{1}{2\gamma+1}}$. For any $\delta \in (0,1)$, it holds with probability at least $1-\delta$ that 
$$\|\bintkernel_{\prod_{p,q} \mHpq}^{\lambda,M}-\bintkernel\|_{\prod_{p,q} \mHpq} \lesssim  C( \bintkernel, \kappa, R,  c_{\mH}, \sigma)\log(\frac{8}{\delta}) M^{-\frac{\gamma}{2\gamma+1}}$$
with $C=\max\{\frac{\kappa_{max} R^2 \|\bintkernel\|_{\infty}}{\sqrt{c_{min}}}, \frac{2\kappa_{max} R\sigma}{\sqrt{cLN}d}, \|B^{-\gamma}\bintkernel \|_{\prod_{p,q} \mHpq}\}$, and $c_{min} = \min_{p,q}(c_{\mHpq})$, $\kappa_{max} = \max_{p,q}(\kappa^{pq})$. 
\label{convergence}
\end{theorem}

See detailed proofs in Appendix \ref{appendix:finitesample}. Moreover, we can also apply the same framework to the reconstruction errors $\|K_{r_*}^{pq,\lambda^{pq},M}- K_{r_*}^{pq}\|_{\mHpq}$ and construct an upper bound on the worst case $L^\infty$ error for the marginal posterior variances, which provides insight regarding uncertainty quantification.

\section{Numerical Examples} \label{sec:numerics}

We now analyze the performance of Algorithm \ref{Algorithm} developed in Section \ref{sec: Methodology} across three examples of widely applicable multi-species interacting agent systems in two dimensions, which realize the model of \eqref{e:firstordersystem1} and \eqref{e:firstordersystem2}. We focus on particle aggregation dynamics under two different interaction potential models in Section \ref{sec: numerical_sec_1} and examine predator-prey flocking interactions in Section \ref{sec: pred}. In Experiment \ref{sec: repulsive}, we show the effect of noise on the learned functions and the robust prediction provided by our framework. Experiment \ref{sec: repulsiveK} builds upon this result to show the effect of varying amounts of data and the performance in the low data regime. Finally, Experiment \ref{sec: pred} carries out the full optimization algorithm to select well-suited hyperparameters and achieve high performance in a difficult setting, highlighting the full power of the Gaussian process approach. As all systems considered are comprised of two distinct species, four interaction kernels are learned in each set of dynamics. For all reported errors, the mean and standard deviation are shown across $10$ independent trials.

\paragraph{Numerical Setup} We simulate all trajectory data on the time interval $[0,T]$ with given i.i.d initial conditions generated from the probability measures $\mu_0^\bx = \mathrm{Unif}([-1,1]^2)$. For the training datasets, we generate $M$ trajectories and observe each trajectory at $L$ equidistant times $0 = t_1 < t_2 < \cdots < t_L = T$. I.i.d. Gaussian noises are added directly to $\bbZ$ with level $\sigma$ for each trajectory. For error computation, we construct the empirical approximation to the probability measure $\tilde\rho_T^{pq,L}$ as defined in \eqref{eq:tilderhoe} with $2000$ randomly initialized trajectories using identical system parameters, and let $[0,R]$ be the support.

\paragraph{Error Metrics} 

In all numerical experiments we report two errors for each learned kernel $\phi^{pq}$. We first consider the $L^{\infty}([0,R])$ relative error, defined by:

\begin{equation}\label{eq:linf_relative_error_formula}
    \frac{ \max_{r \in [0,R]} | \bar{\intkernel}^{pq}(r) - \intkernel^{pq}(r) |}{\max_{r \in [0,R]} | \intkernel^{pq}(r) |} ,
\end{equation}
where $R$ is the maximal value of $r$ witnessed in the empirical data. Second is the $L^2(\tilde\rho_T^{pq,L})$ relative error, defined by:

\begin{equation}\label{eq:l2rho_relative_error_formula}
    \frac{ \| \bar{\intkernel}^{pq}(r) - \intkernel^{pq}(r) \|_{L^2(\tilde\rho_T^{pq,L})}}{\| \intkernel^{pq}(r) \|_{L^2(\tilde\rho_T^{pq,L})}},
\end{equation}
where $\tilde\rho_T^{pq,L}$ is the probability measure defined in \eqref{eq:tilderhoe}. For both kernel error quantities, when the true kernel is identically zero, absolute errors are instead reported. All errors are computed through discretization of the measured interval into $1000$ points.

For trajectory prediction errors, relative errors are computed between the true trajectory of interest $\bX$ and the corresponding predicted trajectory using the learned kernels, denoted $\overline \bX$, as:

\begin{equation}\label{eq:traj_relative_error_formula}
    \max_{t \in I} \frac{\| \overline{\bX}(t) - \bX(t) \|_2 }{ \| \bX(t) \|_2}.
\end{equation}

Note this error depends on a set time interval $I$. We record four separate errors for each experiment: using a training data trajectory and $I = [0,T]$ we compute the training prediction error, and using $I = [T,2T]$ we recover the temporal generalization error on the training set. Using a new initial condition as test data, we similarly utilize both $I = [0,T]$ and $I = [T,2T]$ to compute test trajectory errors. Each trajectory is computed at $100$ equidistant time points in each interval to discretize the error calculation.

\paragraph{Choice of the covariance function.} We choose the Mat\'{e}rn covariance function defined on $[0,R] \times [0,R]$ for all Gaussian process priors in our numerical experiments, i.e.,
\begin{equation}
    K_\theta(r,r')=s_\intkernel^2 \frac{2^{1-\nu}}{\Gamma(\nu)}(\frac{\sqrt{2\nu}|r-r'|}{\omega_{\intkernel}})^\nu B_\nu(\frac{\sqrt{2\nu}|r-r'|}{\omega_{\intkernel}}),
\end{equation}
where the parameter $\nu > 0$ determines the smoothness; $\Gamma(\nu)$ is the Gamma function; $B_\nu$ is the modified Bessel function of second kind; the hyperparameters $\theta = \{s_\intkernel^2, \omega_{\intkernel}\}$ parameterize the amplitude and scales. In our numerical examples, we choose $\nu = 3/2$ as an appropriate level of smoothness.

Let $k_{\mathrm{Mat\acute{e}rn}(\nu)}$ denote the Mat\'ern kernel with smoothness parameter $\nu>0$ restricted to $[0,R]$. The associated RKHS $\mathcal{H}_{\mathrm{Mat\acute{e}rn}(\nu)}$ is norm-equivalent to the Sobolev/Bessel potential space $H^{s}([0,R])=W_2^{s}([0,R])$ with
\[
s=\nu+\tfrac12 .
\]
That is, there exist constants $c_1,c_2>0$ such that for all $f\in H^{s}([0,R])$,
\[
  c_1\|f\|_{H^{s}([0,R])}\;\le\;\|f\|_{\mathcal{H}_{\mathrm{Mat\acute{e}rn}(\nu)}}\;\le\;c_2\|f\|_{H^{s}([0,R])}.
\]
In particular, for $\nu=\tfrac32$ we have $s=2$ and hence
\[
  \mathcal{H}_{\mathrm{Mat\acute{e}rn}(3/2)}\;\simeq\;H^{2}([0,R])=W_2^{2}([0,R]),
\]
so elements of this RKHS admit weak derivatives up to order $2$ in $L^2([0,R])$.

\paragraph{Summary of the Numerical Experiments} 

\begin{itemize}
    \item The proposed Gaussian Process learning algorithm successfully performs a highly accurate approximation of true interaction functions from small amounts of noisy data. In all examples, numerical errors of learned functions are sufficiently small to allow for highly accurate trajectory prediction across both larger temporal settings and new initial conditions.
    \item The experiments of \ref{sec: numerical_sec_1} show the strong effect of lower noise and additional data upon kernel and trajectory predictions. This convergence behavior shows that across reasonable ranges of noise values and data amounts, our method is capable of suitably accurate performance.
    \item Experiment \ref{sec: pred} shows the essential benefit of the Gaussian Process approach through utilizing optimization of the kernel parameters to result in a better fit in predicted interaction functions in a situation where small errors cause large divergences in trajectory. The optimized hyperparameters are able to satisfactorily capture the dynamics, while unoptimized hyperparameters struggle in the low-data regime.
\end{itemize}

\subsection{Example 1: Two Species Particle Aggregation Dynamics}\label{sec: numerical_sec_1}

Two-species particle aggregation dynamics arise in diverse settings, from nanoscale self-assembly in materials science \cite{liu2010ionconfig, liu2011ionassembly} to microbial and animal group organization in biology \cite{tsimring1995bacteria, vliet2022microbialex, levine2000flockingex1}, and even to leader–follower interactions in the social sciences \cite{bertram2003opiniondynamics}. Such models are compelling because they capture a richer spectrum of emergent behaviors than single-species systems, including segregation, mixed clustering, and multiscale spatial arrangements. In this paper, we focus on the framework introduced in \cite{mackey2014two}, which provides a representative two-species aggregation model and demonstrates that the dynamics can evolve toward steady states with nontrivial geometric structures. This setting is both practically motivated and mathematically rich, and serves as a natural testbed for our data-driven inference methodology.  

We define three functions utilized across examples for our interaction function construction. In the function $G_0$, the constant $C = 0.9357796257$ results in particular instabilities of interest in the dynamics. In the repulsive example, all kernels are positive at small distances and negative at long distances, modeling particles that attract when close and repel when further apart. For the linear-repulsive dynamics, the intra-species interactions are modeled similarly, but inter-species interactions are linear and remain negative throughout the domain, modeling species with only repulsive interactions. See Table \ref{tab:ex_info_repulsive_kernels} for the true interaction functions in each example.

\begin{equation*}
\begin{aligned}
    G_0(x) &= 1 + 2(1-x) + x^{-\frac 1 4} - C\\
    G_3(x) &= 1 + (1-x) + (1-x)^2\\
    G_5(x) &= \frac 3 2 (1-x)^2 +(1-x)^3 - (1-x)^4
\end{aligned}
\end{equation*}

\begin{table}[h!]
\caption{True interaction kernels for particle aggregation dynamics.}
\label{tab:ex_info_repulsive_kernels} 
\begin{center}
\begin{tabular}{ c c c}
\hline
 System & Repulsive \ref{sec: repulsive} & Linear-Repulsive \ref{sec: repulsiveK}  \\
 \hline 
$\phi^{11}$ &  $G_0(\frac 1 2 r^2)$  &  $G_3(r) + 1.1158 G_0(r)$   \\

$\phi^{12}$ &  $\frac 1 2 G_0(\frac 1 2 r^2)$  &   $-4r$  \\

$\phi^{21}$ &  $\frac 1 2 G_0(\frac 1 2 r^2)$  &   $-4r$  \\

$\phi^{22}$ &  $G_0(\frac 1 2 r^2)$  &   $G_5(r) + 1.3 G_0(r)$  \\
\hline 
\end{tabular}
\end{center}
\end{table}

\subsubsection{Repulsive Interaction Potentials}\label{sec: repulsive} 

For our first example, we analyze the behavior of our kernel learning pipeline utilizing a standard repulsive potential, which scales as $\frac{1}{\sqrt{r}} - r^2$ and thus provides a steady repulsive force with a singularity at the origin. Of note is the ability of this potential to apply negative force at longer distances, which draws particles into a steady-state solution of a ring formation, with different particles from each species scattered throughout a ring at distances corresponding to roughly equal forces exerted from all neighbors. As the true interaction potentials are singular at the origin, we truncate each for $r < 0.25$ by a function of the form $ae^{-br}$ with $a,b$ chosen so that the function and its first derivative match at $r=0.25$.

We first show the performance of our method for the repulsive potentials with $N_1 = N_2 = 10$ agents of each type, $L = 10$ time steps, $M = 10$ training trajectories, and dynamics evolution on the interval $[0,T]$ with $T=5$. We also add noise of $\sigma = 0.01$. Performance is shown in Figure \ref{fig:repulsive_figs}, where for this modest amount of training data, we are able to effectively learn each kernel even in the presence of noise and successfully recover the single-ring steady state dynamics.

\begin{figure}[H]
\centering
\subfigure[$\phi^{11}$]{
\includegraphics[width=0.24\linewidth]{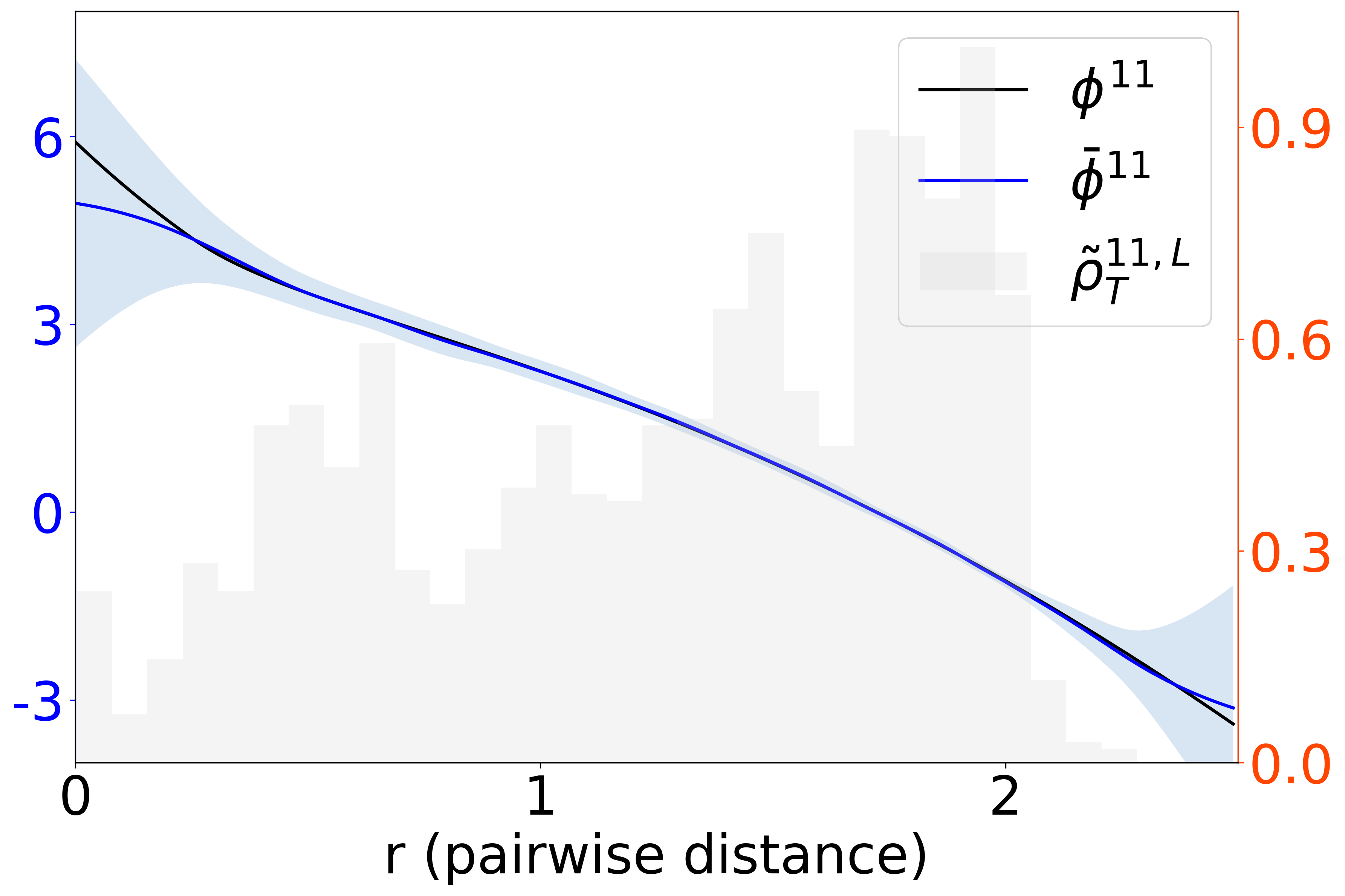}
}
\hfill
\subfigure[$\phi^{12}$]{
\includegraphics[width=0.24\linewidth]{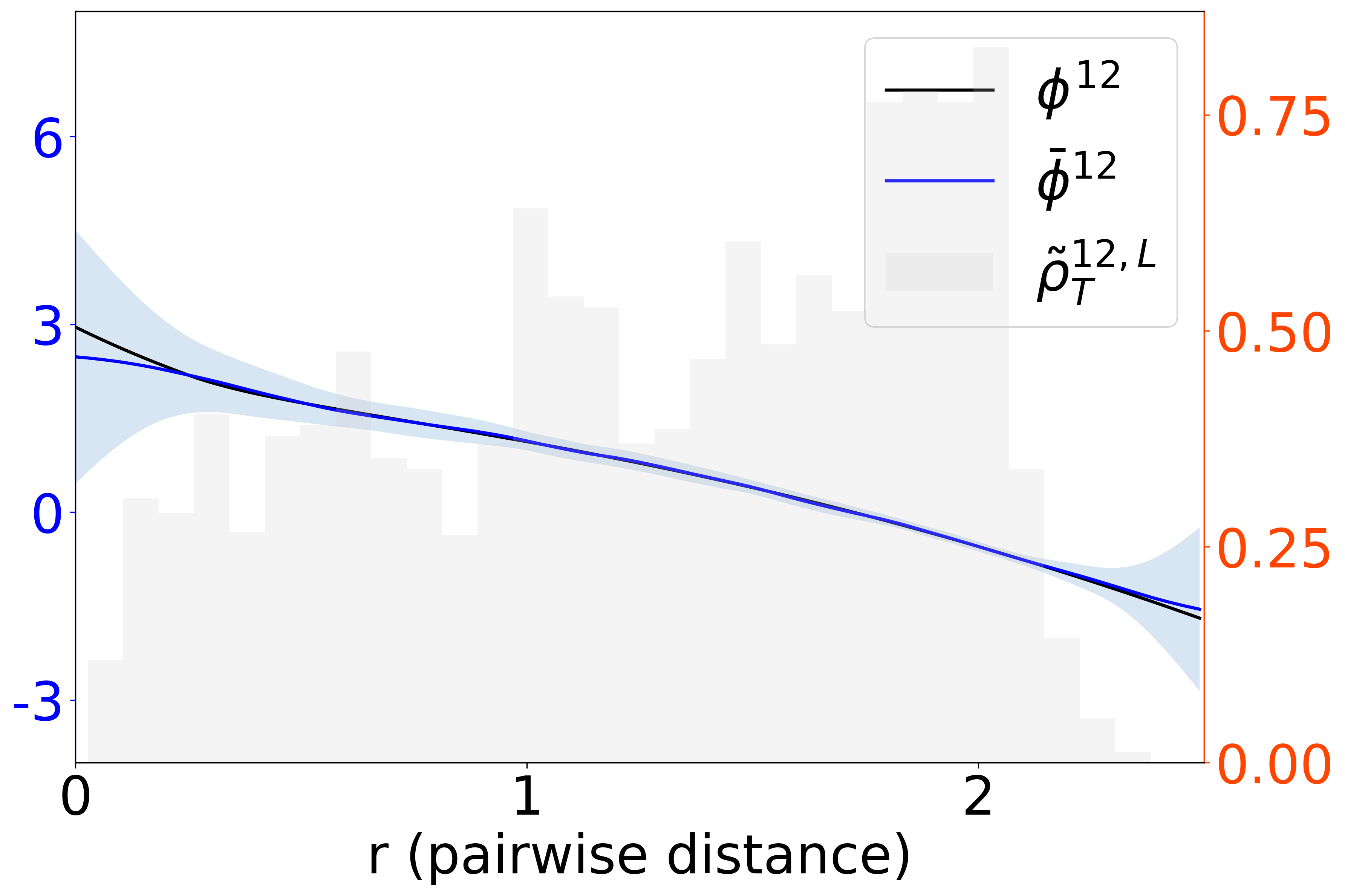}
}
\hfill
\subfigure[Training Data]{
\includegraphics[width=0.4\linewidth]{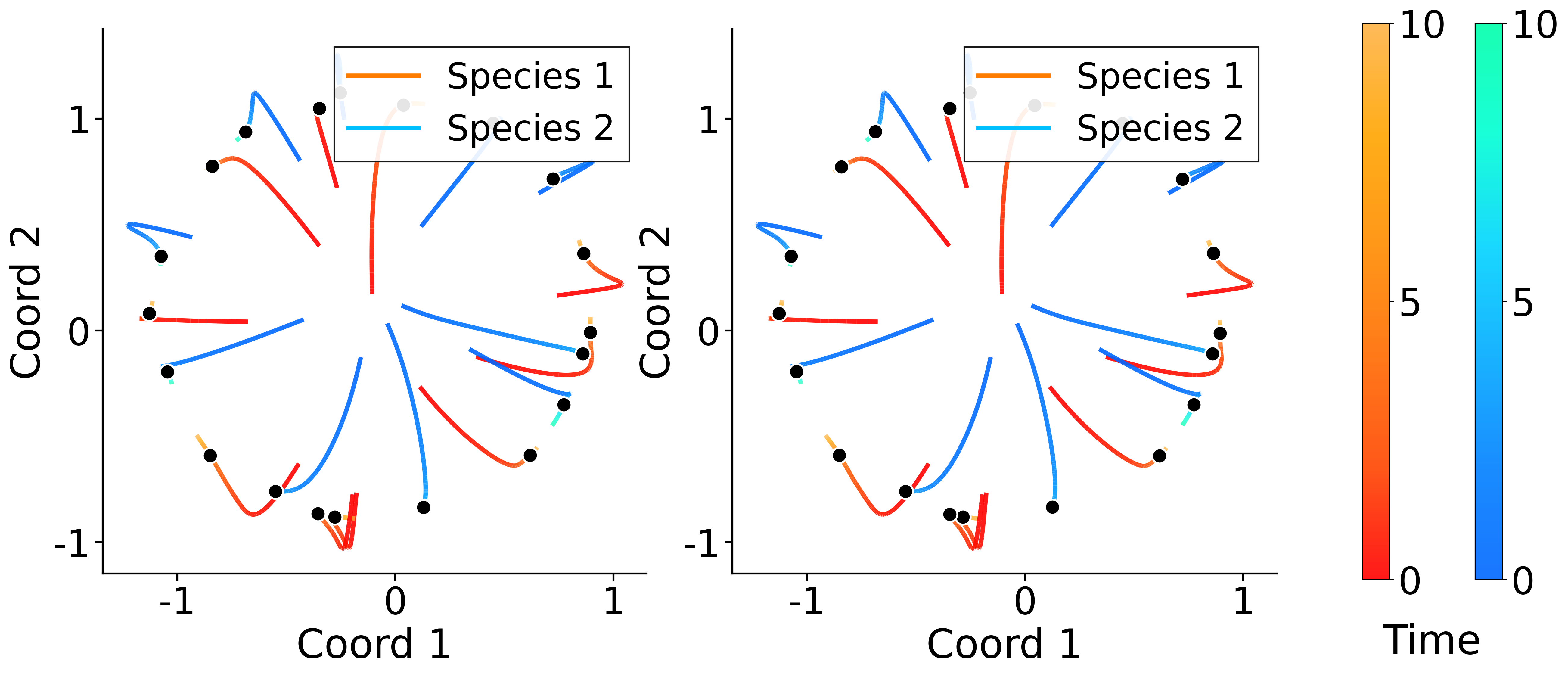}
}

\vspace{0.5em}

\subfigure[$\phi^{21}$]{
\includegraphics[width=0.24\linewidth]{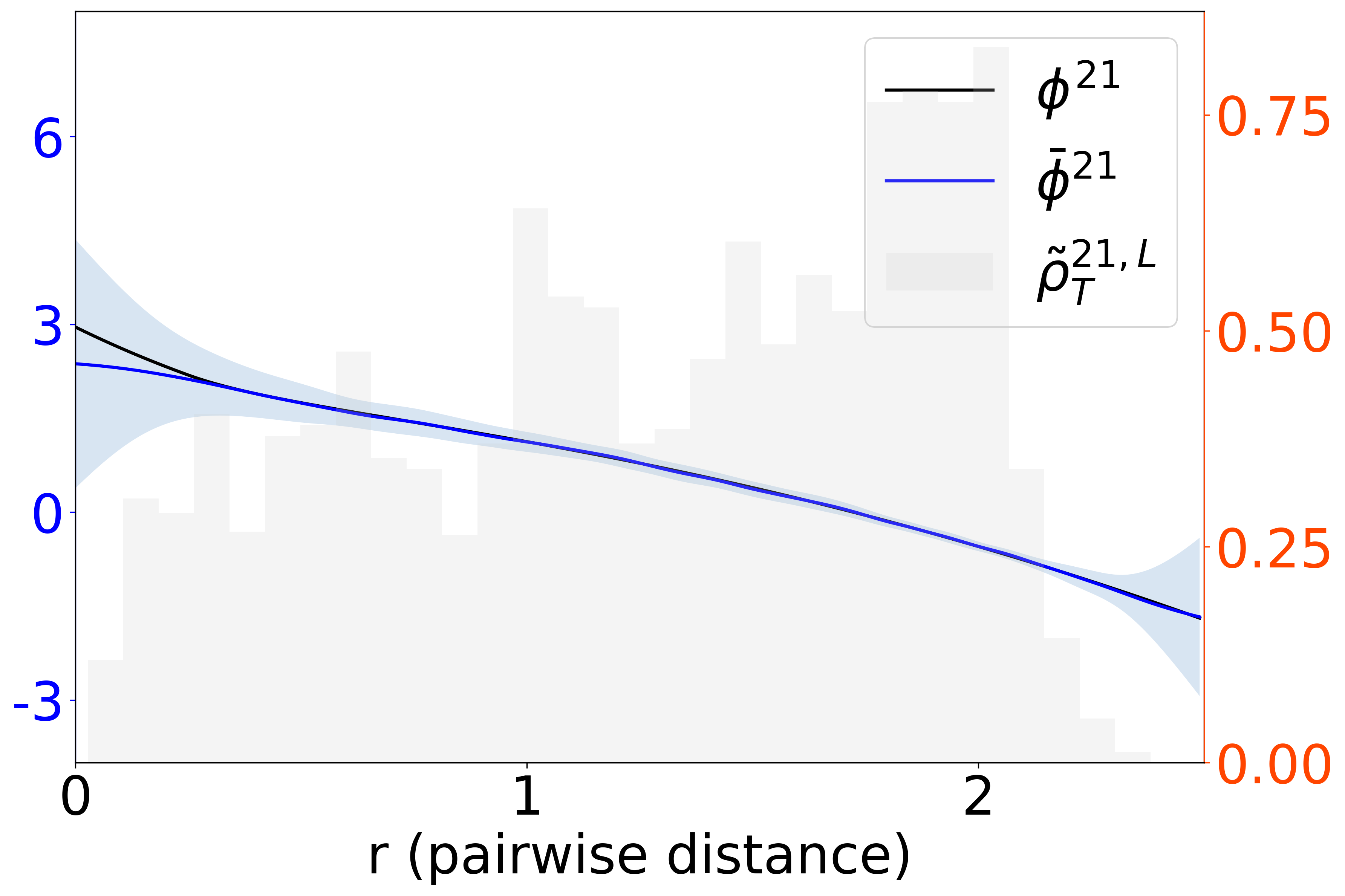}
}
\hfill
\subfigure[$\phi^{22}$]{
\includegraphics[width=0.24\linewidth]{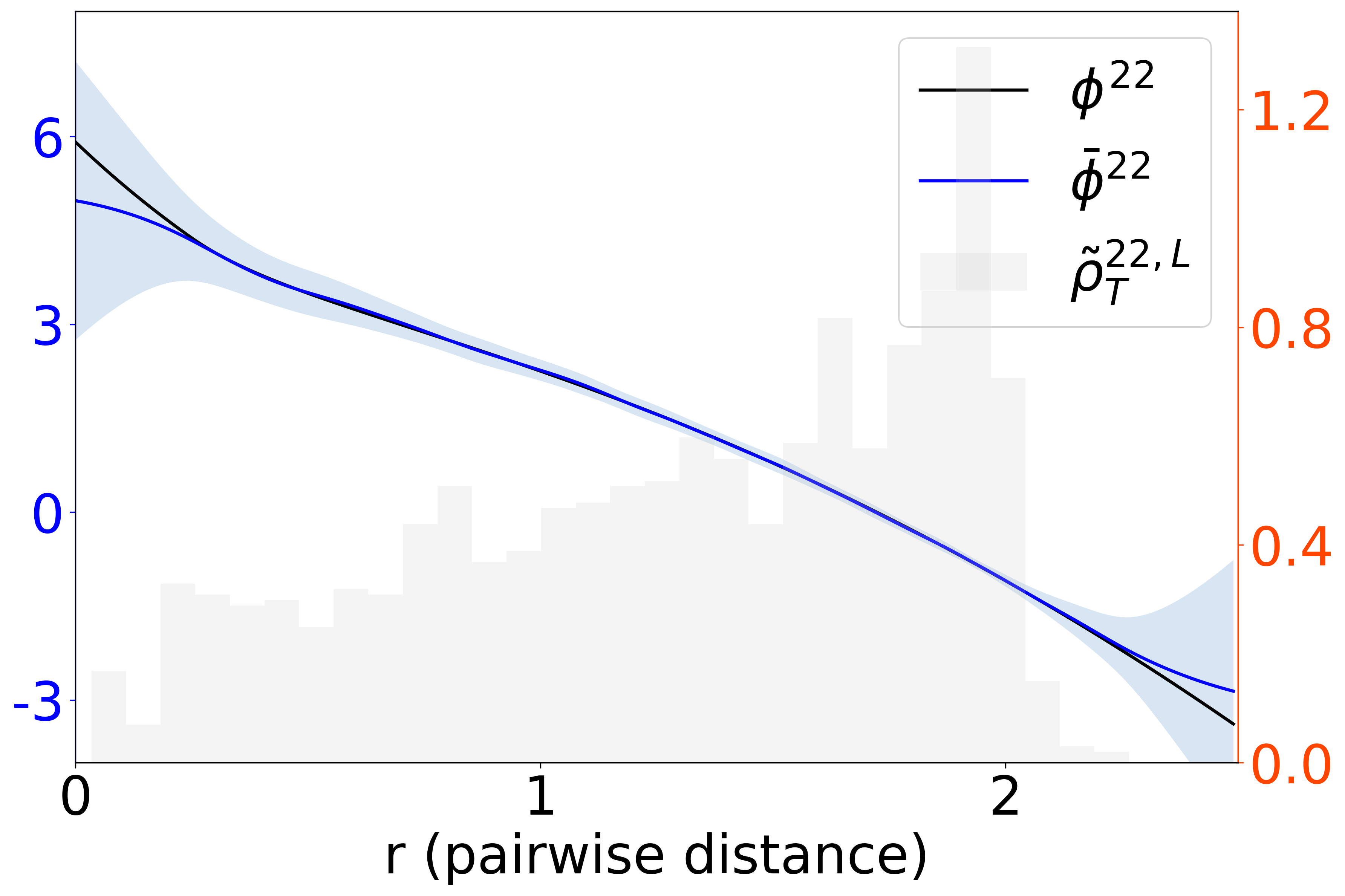}
}
\hfill
\subfigure[Testing Data]{
\includegraphics[width=0.4\linewidth]{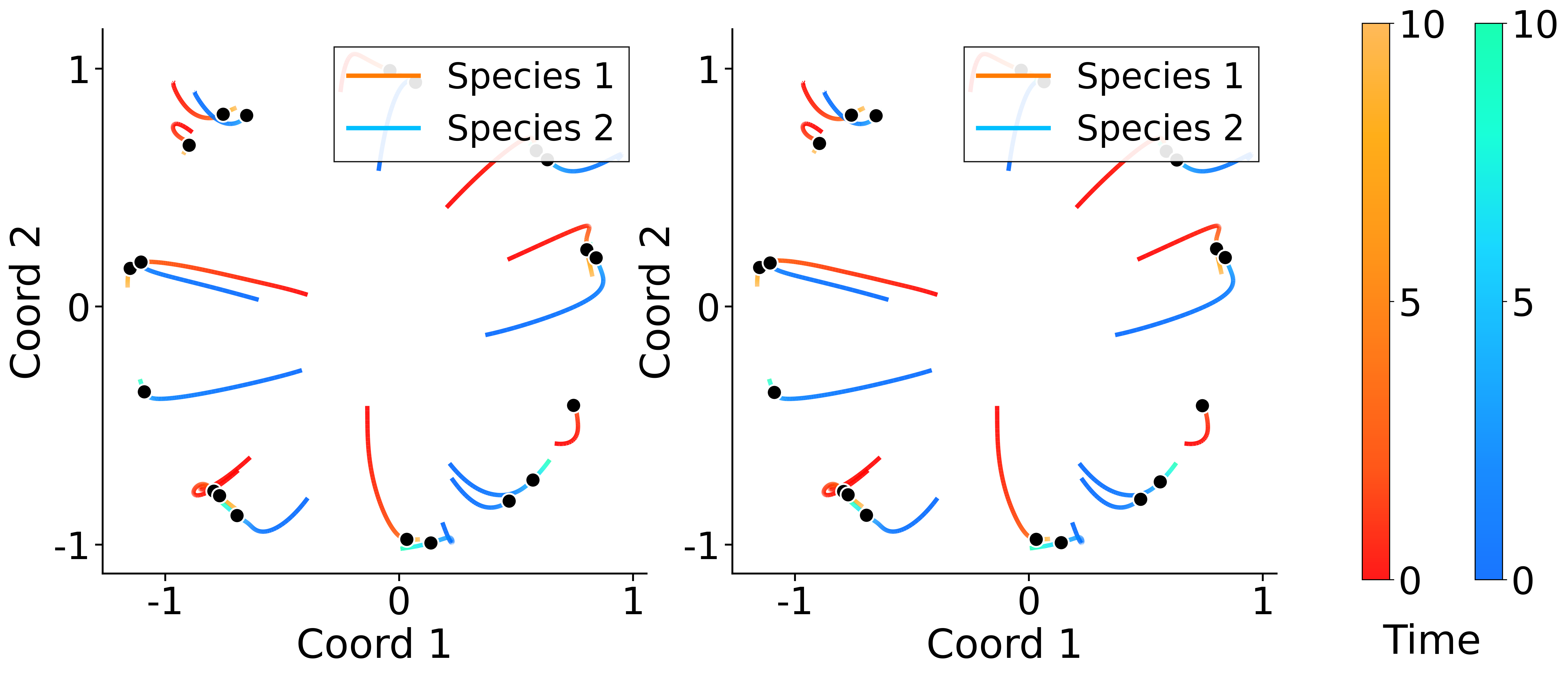}
}

\caption{Results of kernel learning for the repulsive potential dynamics with $N_1 = N_2 = 10$, $L=10$, and $M=10$ with noise $\sigma = 0.01$. Left, Center: The four interaction kernels are shown with true function in black and predicted mean in blue, with the shaded region indicating the standard deviation band. Gray bars show the empirical distribution of pairwise distances. Right: Training and testing data trajectory prediction plots on $[0,2T]$ are presented, with the true dynamics on the left of each pair and the predicted dynamics on the right. A black dot marks each trajectory at the time snapshot $t=T$. The top pair utilizes a training trajectory to test temporal generalization, while the bottom pair uses test data. The system evolution and steady-state behavior are extremely similar when using the predicted interaction functions.}
\label{fig:repulsive_figs}
\end{figure}

As shown in Figure \ref{fig:repulsive_figs}, the learned interaction kernels are very accurate on the support of the data. Accuracy degrades when very close to the origin, but this does not result in any meaningful loss of accuracy in dynamics prediction as interaction kernel outputs are scaled by $r$ and quickly vanish near zero. To further examine the performance of our method, we examine the effect of noise on the final prediction. We run our learning framework for noise levels of $\sigma \in \{0, 0.0001, 0.0005, 0.001, 0.005, 0.01, 0.05, 0.1\}$ and report the final errors in Figure \ref{fig:convergence_rates_noise} below, and in Tables \ref{tab:kernel_errors_repulsive} and \ref{tab:trajectory_errors_repulsive} in the appendix.

\begin{figure}[tbhp]
\centering
\subfigure{
\includegraphics[width=0.31\linewidth]{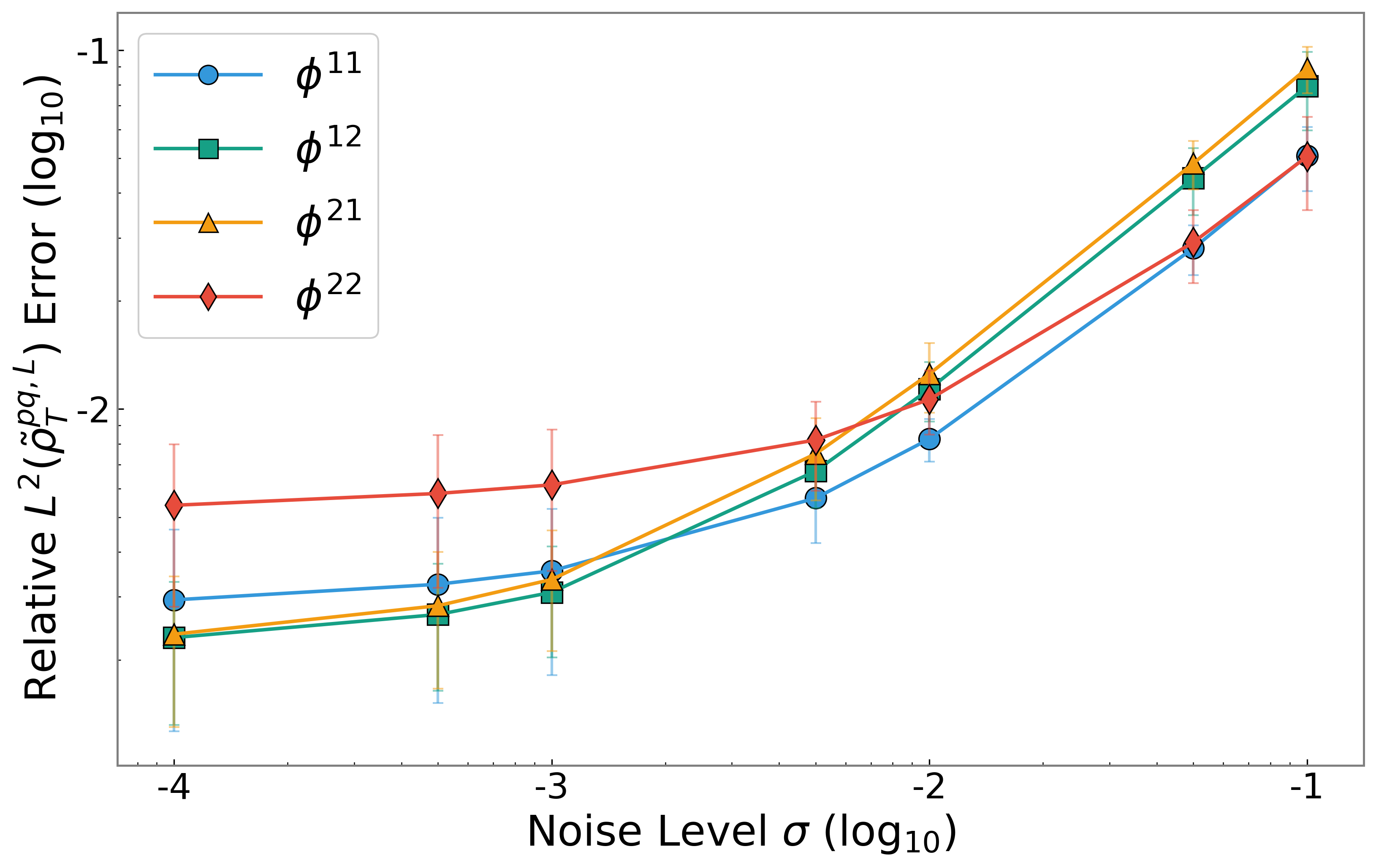}
}
\hfill
\subfigure{
\includegraphics[width=0.31\linewidth]{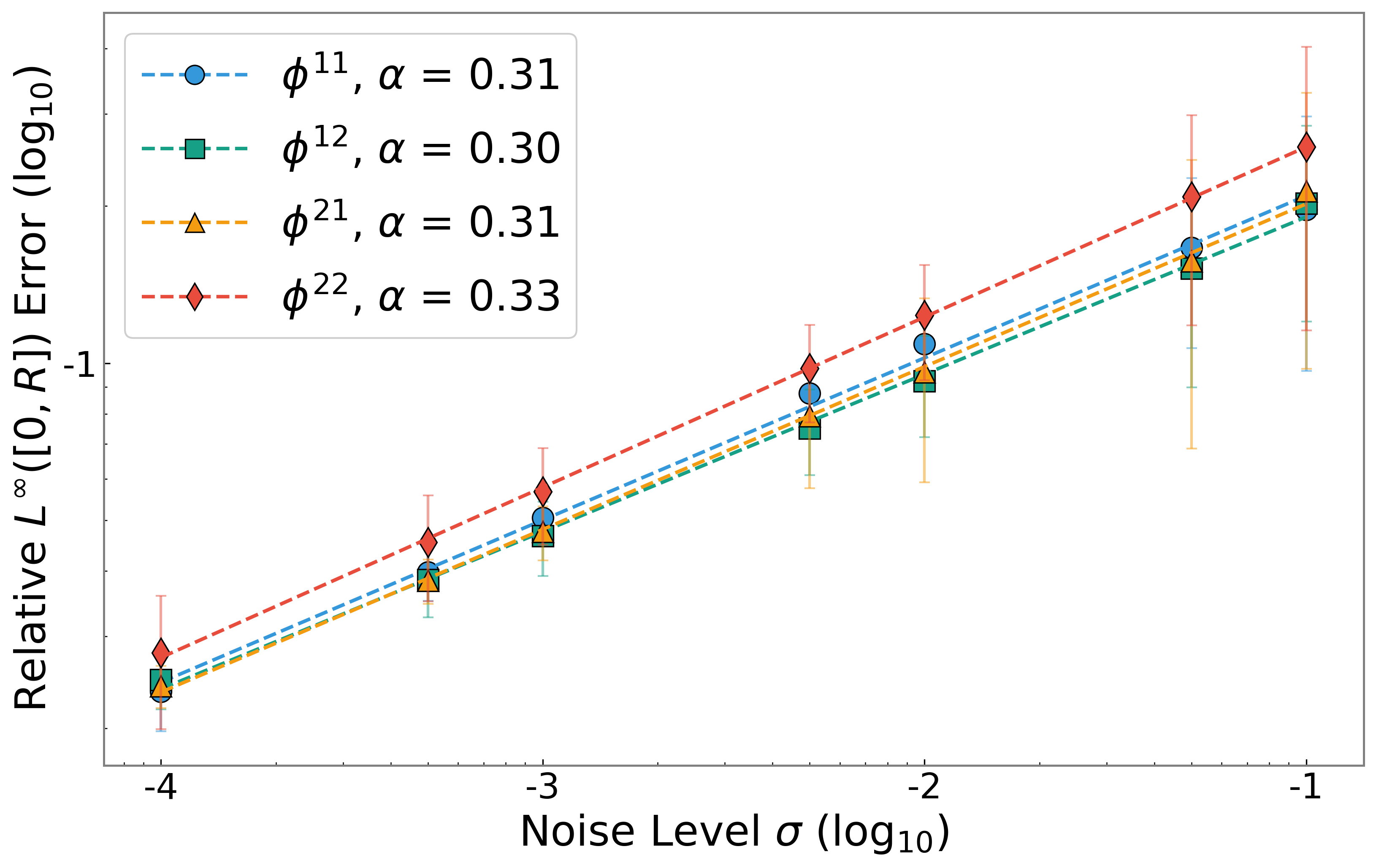}
}
\hfill
\subfigure{
\includegraphics[width=0.31\linewidth]{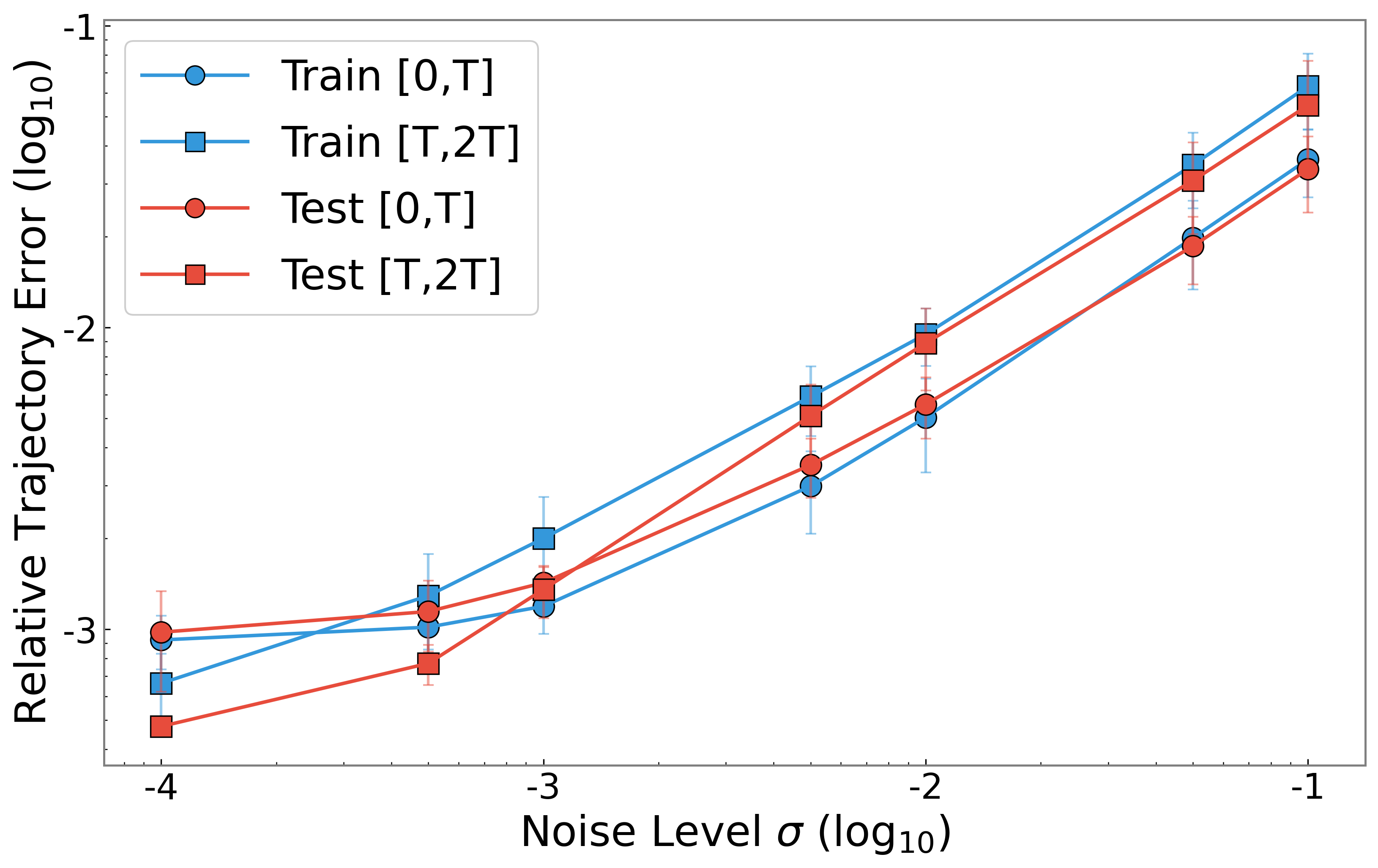}
}
\caption{Analysis of the noise dependence of kernel learning and trajectory prediction errors as a function of the noise level $\sigma$ for the repulsive potential dynamics on a log-log plot. Each curve shows the mean error across ten random seeds, with error bars indicating standard deviation. (Left) Relative $L^2(\tilde\rho_T^{pq,L})$ errors for the four interaction kernels. (Center) Relative $L^\infty([0,R])$ errors for the four interaction kernels. Note the consistent linear behavior; the slope $\alpha$ in the legend indicates the power-law rate of error growth (error $\sim \sigma^\alpha$) as the noise increases. Once noise is very small, bias (discretization + finite basis) dominates, hence the plateau. (Right) Relative trajectory prediction errors for training data (blue) and test data (red) on both the training period $[0,T]$ and temporal generalization period $[T,2T]$. $L^2(\tilde\rho_T^{pq,L})$ error and trajectory error steadily decrease until around $\sigma = 10^{-3}$, with smaller noise levels yielding diminished returns past this point as they approach the zero noise accuracy level.}
\label{fig:convergence_rates_noise}
\end{figure}

As noise decreases, kernel estimation errors for all four kernels, as well as the corresponding trajectory prediction errors, significantly decrease in mean, as expected through our theoretical analysis. Of note in the log-log plots is the linear trend up to $\sigma = 10^{-3}$ showing a strong dependence upon noise level past an initial threshold; as noise decreases to suitably low levels, error plateaus as the performance nears the accuracy of the zero noise limit.

\subsubsection{Linear-Repulsive Interaction Potentials}\label{sec: repulsiveK}

We now analyze a repulsive potential system with strong coupling effects where cross-species interactions scale linearly \cite{mackey2014two}. Compared to the repulsive potentials of Experiment \ref{sec: repulsive}, cross-species interactions remaining negative even at small distances leads to behavior where closely-positioned particles are quickly displaced. The emergent steady-state manifests as concentric rings, where each ring consists solely of one type of particle, as opposed to the singular mixed ring of Experiment \ref{sec: repulsive}. As the true kernels of $\phi^{11}$ and $\phi^{22}$ are again singular at the origin, we truncate these functions at $r = 0.5$ by a function of the form $ae^{-br}$, choosing the values of $a$ and $b$ to ensure continuity of the interaction function and its derivative at the cutoff.

For this experiment, we additionally focus on the effect of surplus data on the prediction performance of both interaction potentials and overall dynamics. We set $N_1 = N_2 = 5, L = 2, \sigma = 0.05$ and vary $M \in \{1, 10, 50, 100, 250, 500, 750, 1000\}$ to learn in various data regimes, with dynamics evolved on $T = 5$. In Figure \ref{fig:convergence_rates}, we show the convergence behavior of all errors while varying $M$, with complete results also presented in Tables \ref{tab:kernel_errors_repulsiveK} and \ref{tab:trajectory_errors_repulsiveK} in the appendix.

\begin{figure}[tbhp]
\centering
\subfigure{
\includegraphics[width=0.31\linewidth]{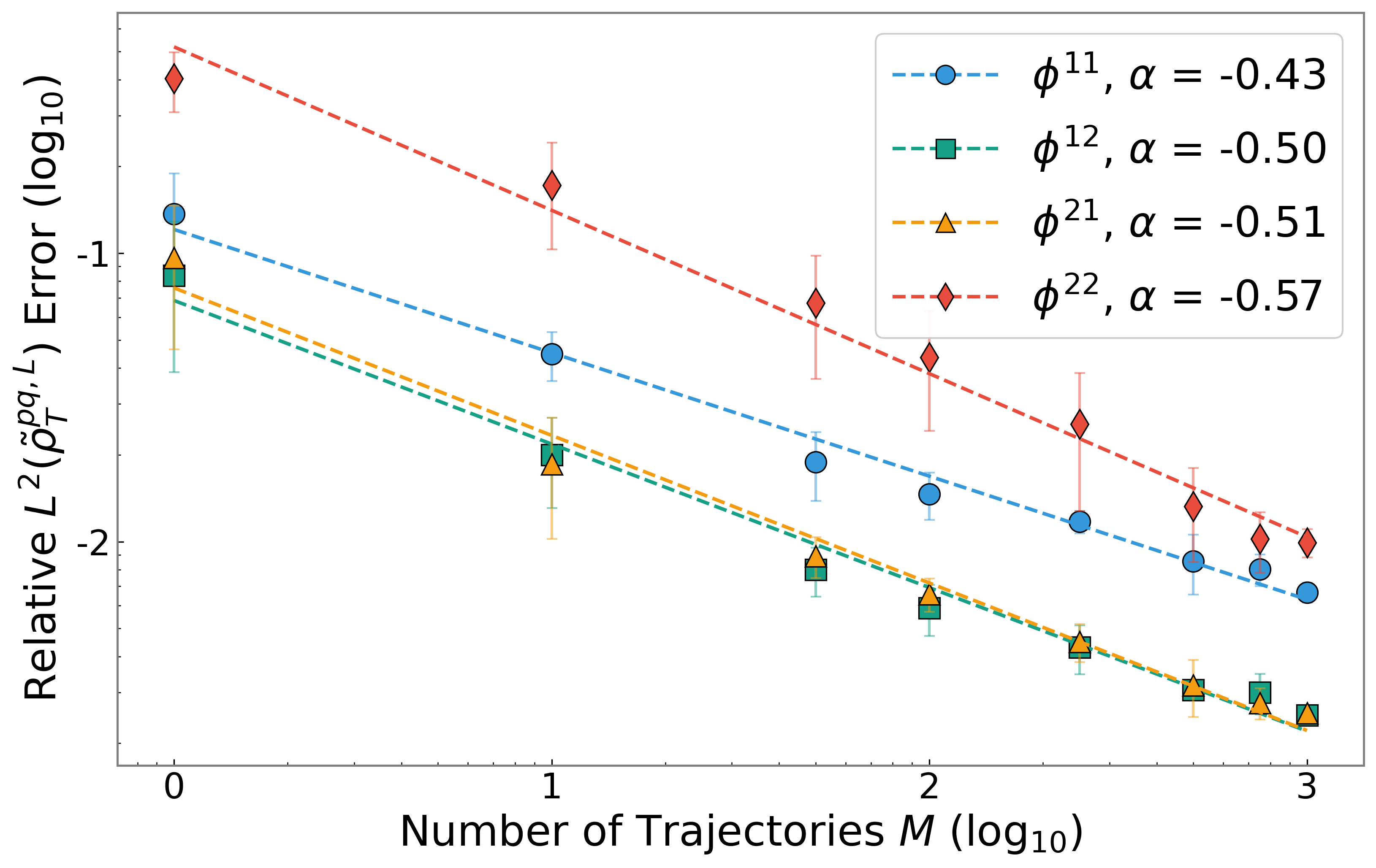}
}
\hfill
\subfigure{
\includegraphics[width=0.31\linewidth]{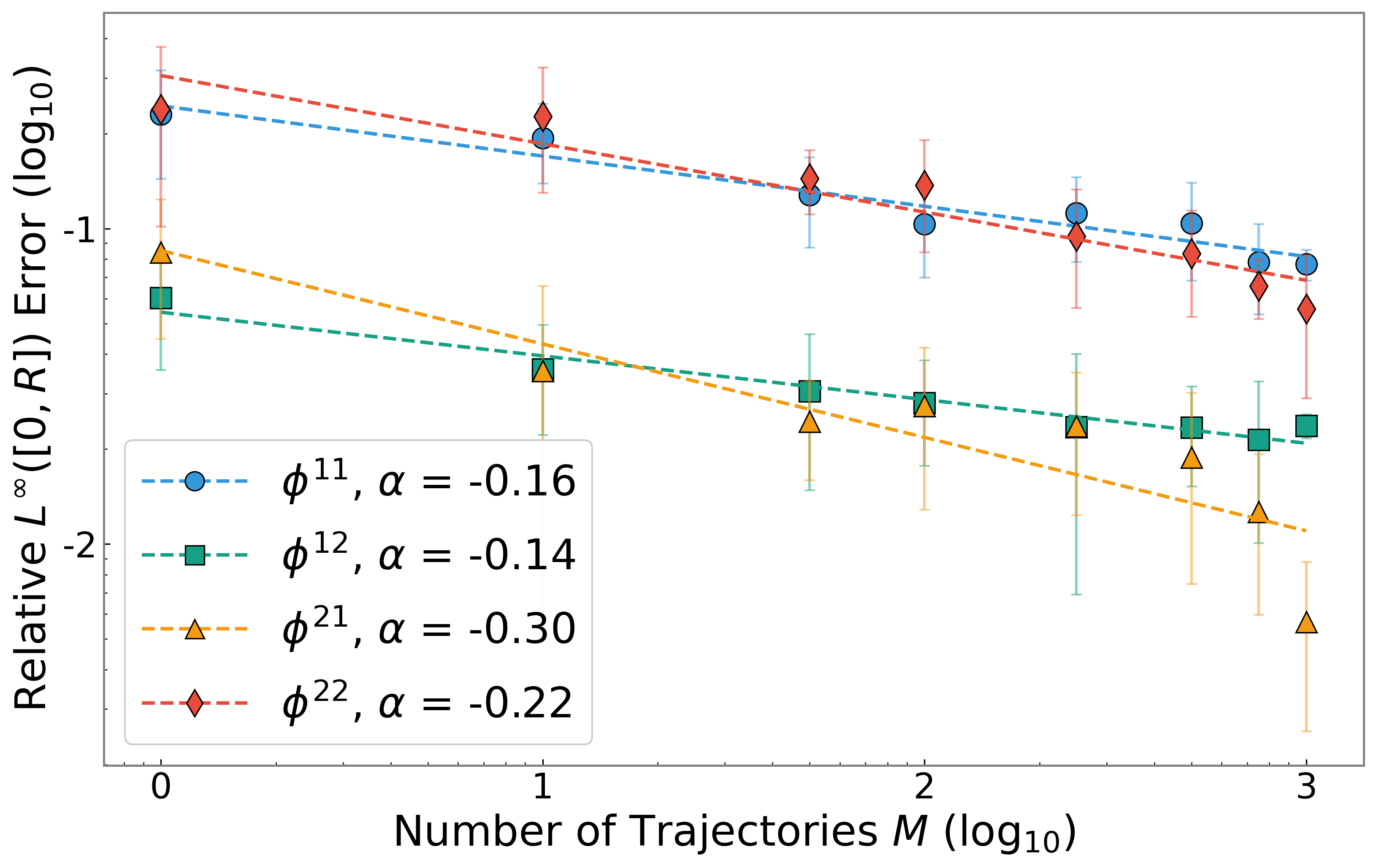}
}
\hfill
\subfigure{
\includegraphics[width=0.31\linewidth]{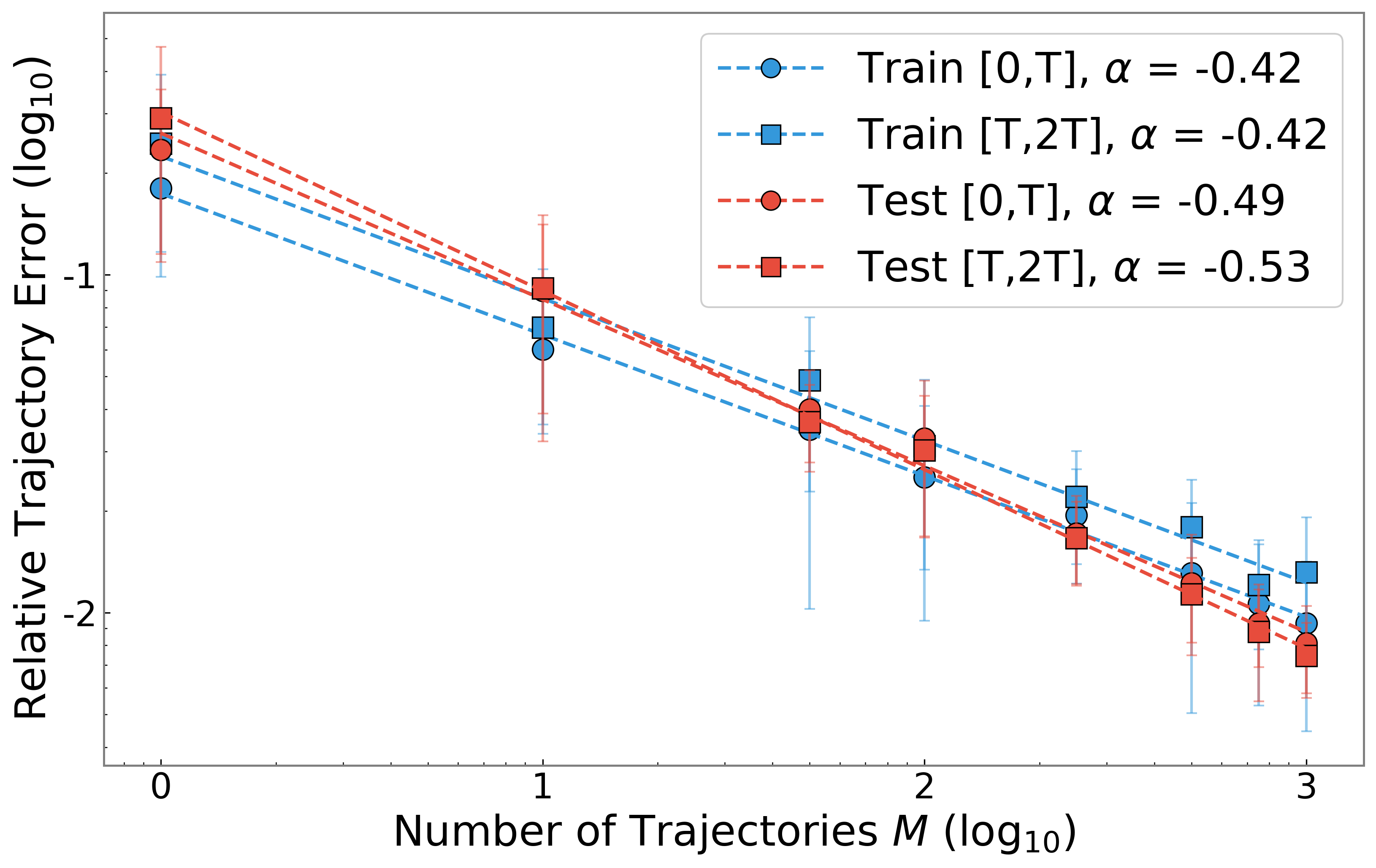}
}
\caption{Convergence analysis of kernel learning and trajectory prediction errors as a function of the number of training trajectories $M$ for the linear-repulsive potential dynamics. Each curve shows the mean error across ten random seeds, with error bars indicating standard deviation. The slope $\alpha$ in the legend indicates the power-law convergence rate (error $\sim M^\alpha$). (Left) Relative $L^2(\tilde\rho_T^{pq,L})$ errors for the four interaction kernels. (Center) Relative $L^\infty([0,R])$ errors for the four interaction kernels. (Right) Relative trajectory prediction errors for training data (blue) and test data (red) on both the training period $[0,T]$ and temporal generalization period $[T,2T]$.}
\label{fig:convergence_rates}
\end{figure}

As in the previous example, we report the $L^{\infty}([0,R])$ and $L^2(\tilde\rho_T^{pq,L})$ errors, and the relative trajectory prediction errors. As $M$ increases, kernel estimation errors for all four kernels, as well as the corresponding trajectory prediction errors, significantly decrease in both mean and standard deviation. The relative $L^2(\tilde\rho_T^{pq,L})$ error and trajectory error converge with observed rates near $-\frac 1 2$, while the relative $L^\infty([0,R])$ error converges with more modest rates that are in line with the predicted range of Theorem \ref{convergence}. This example shows the data-driven nature of our approach, as an abundance of data will naturally lead to more accurate predictions even while keeping all other hyperparameters constant. We also show the qualitative behavior of the learned kernels and their generated dynamics in Figure \ref{fig:repulsivek_figs}.

\begin{figure}[tbhp]
\centering
\subfigure[$\phi^{11}$]{
\includegraphics[width=0.27\linewidth]{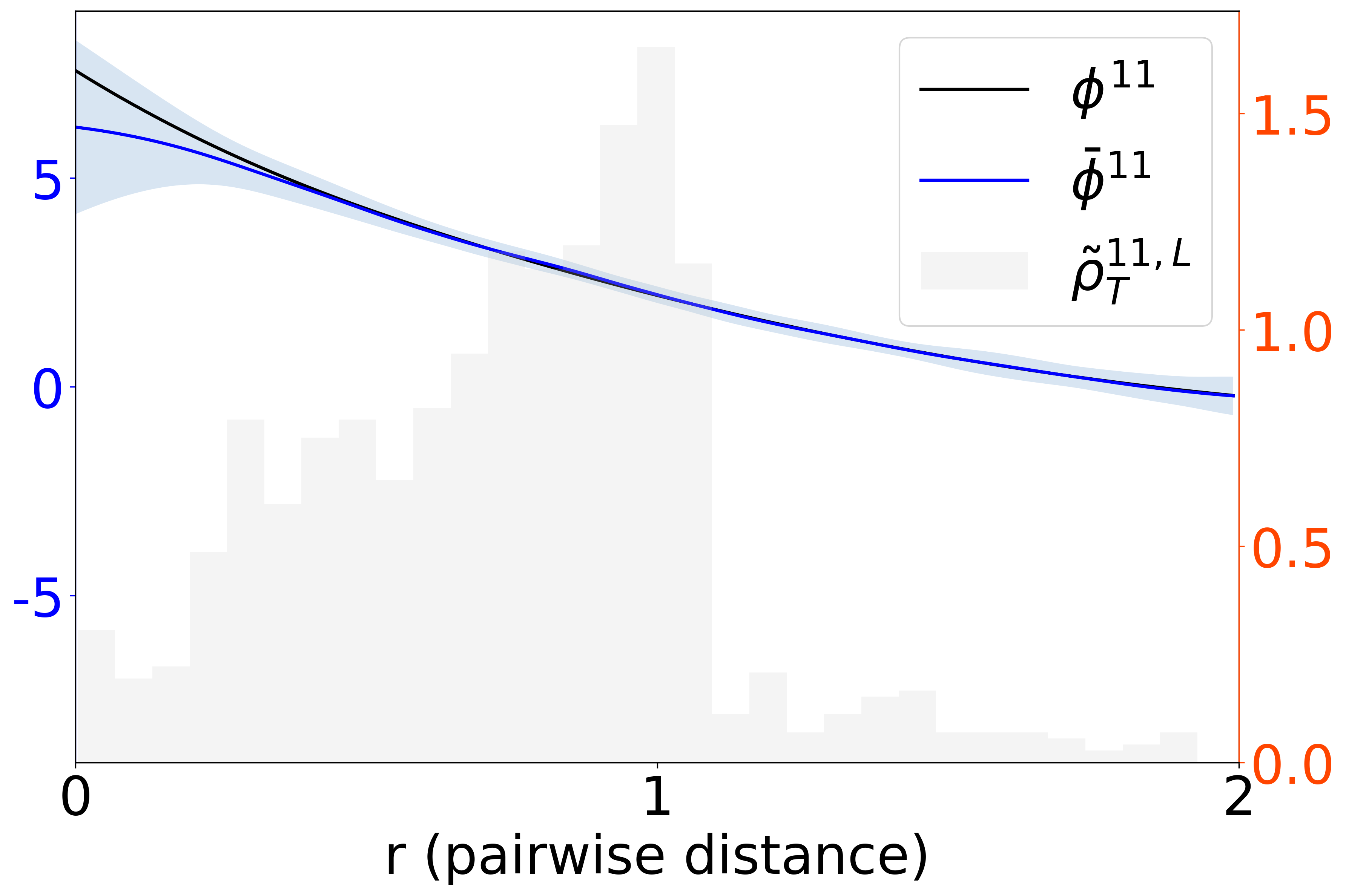}
}
\subfigure[$\phi^{12}$]{
\includegraphics[width=0.27\linewidth]{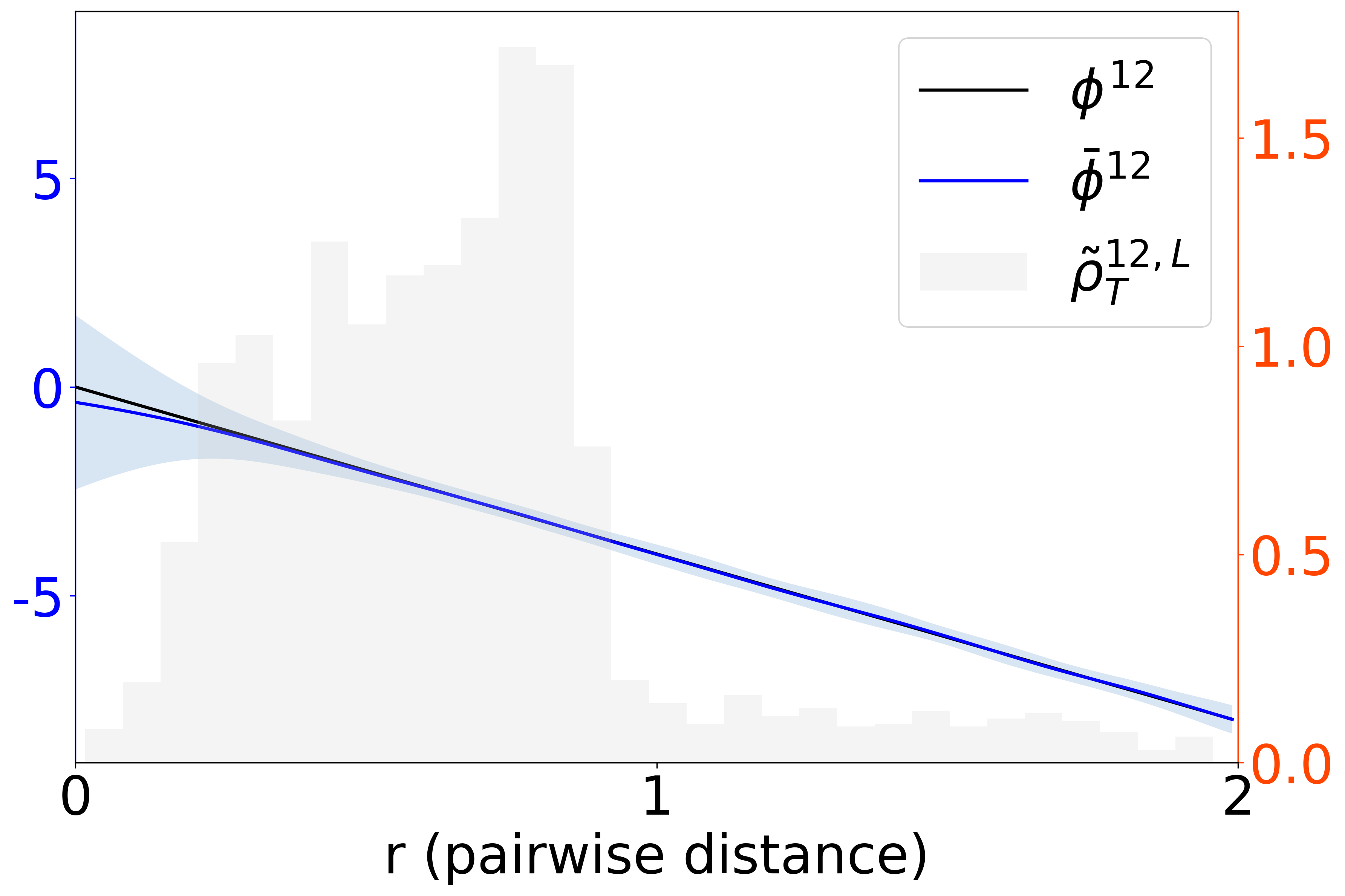}
}
\subfigure[Training Data]{
\includegraphics[width=0.39\linewidth]{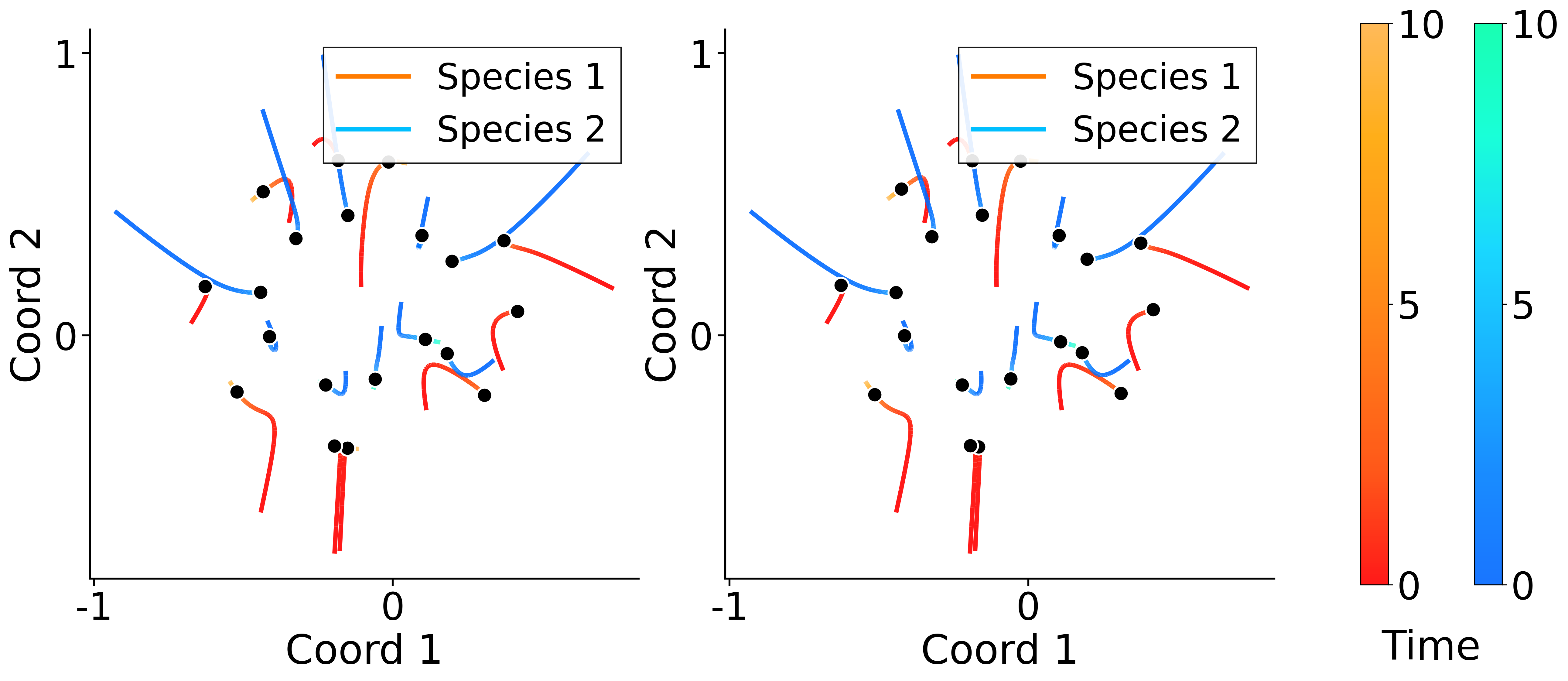}
}

\vspace{0.5em}

\subfigure[$\phi^{21}$]{
\includegraphics[width=0.27\linewidth]{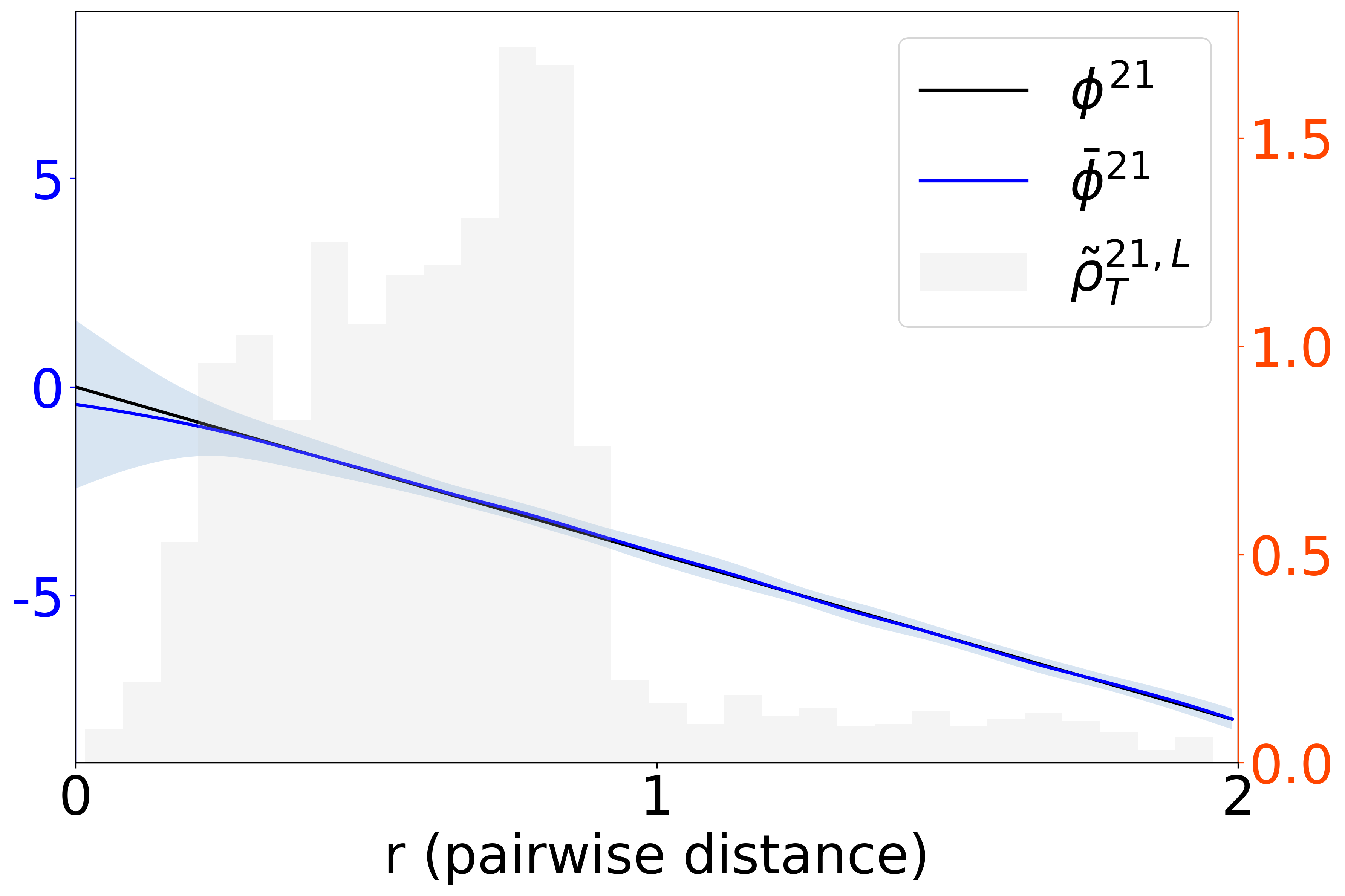}
}
\hfill
\subfigure[$\phi^{22}$]{
\includegraphics[width=0.27\linewidth]{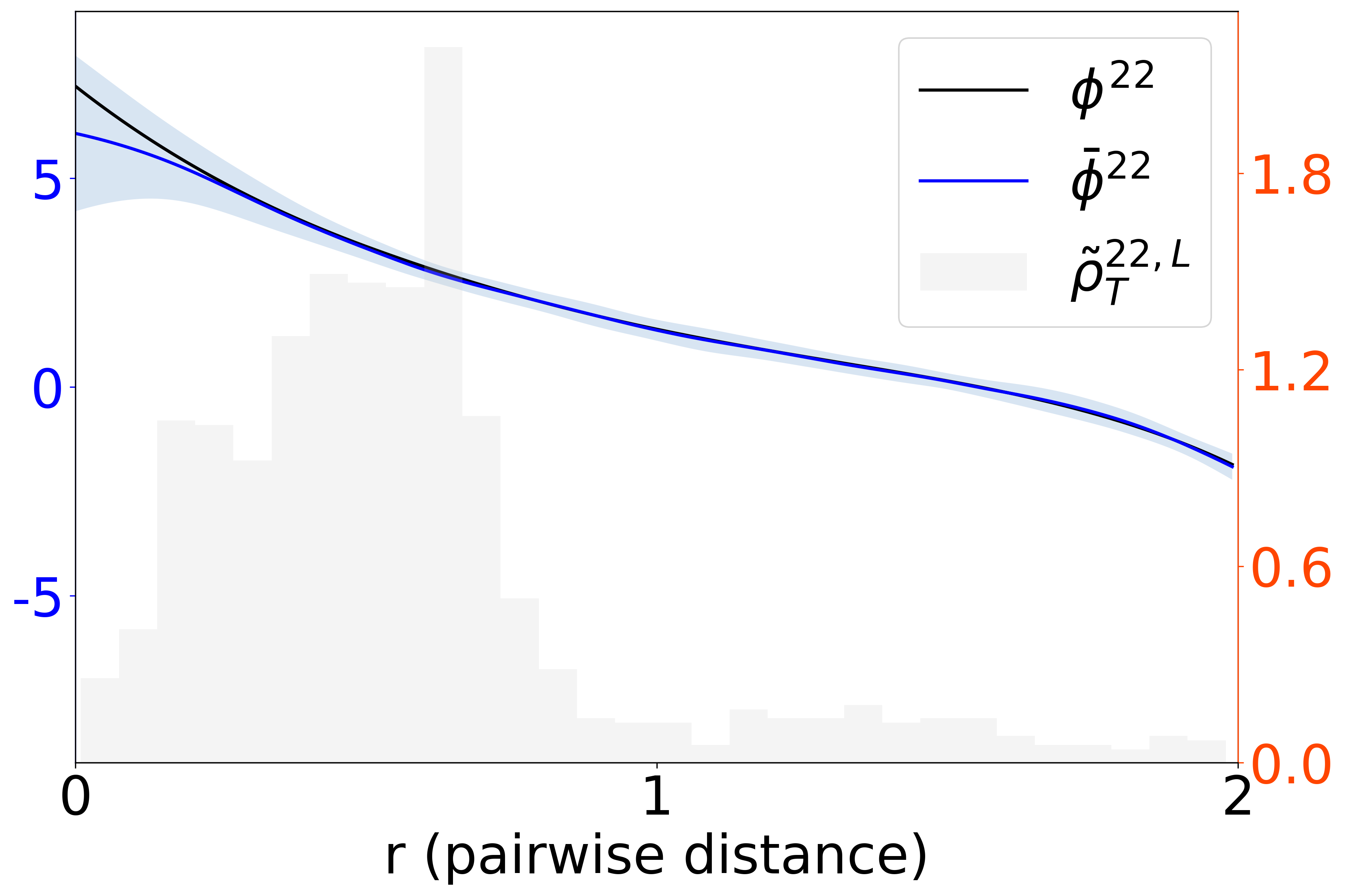}
}
\hfill
\subfigure[Testing Data]{
\includegraphics[width=0.39\linewidth]{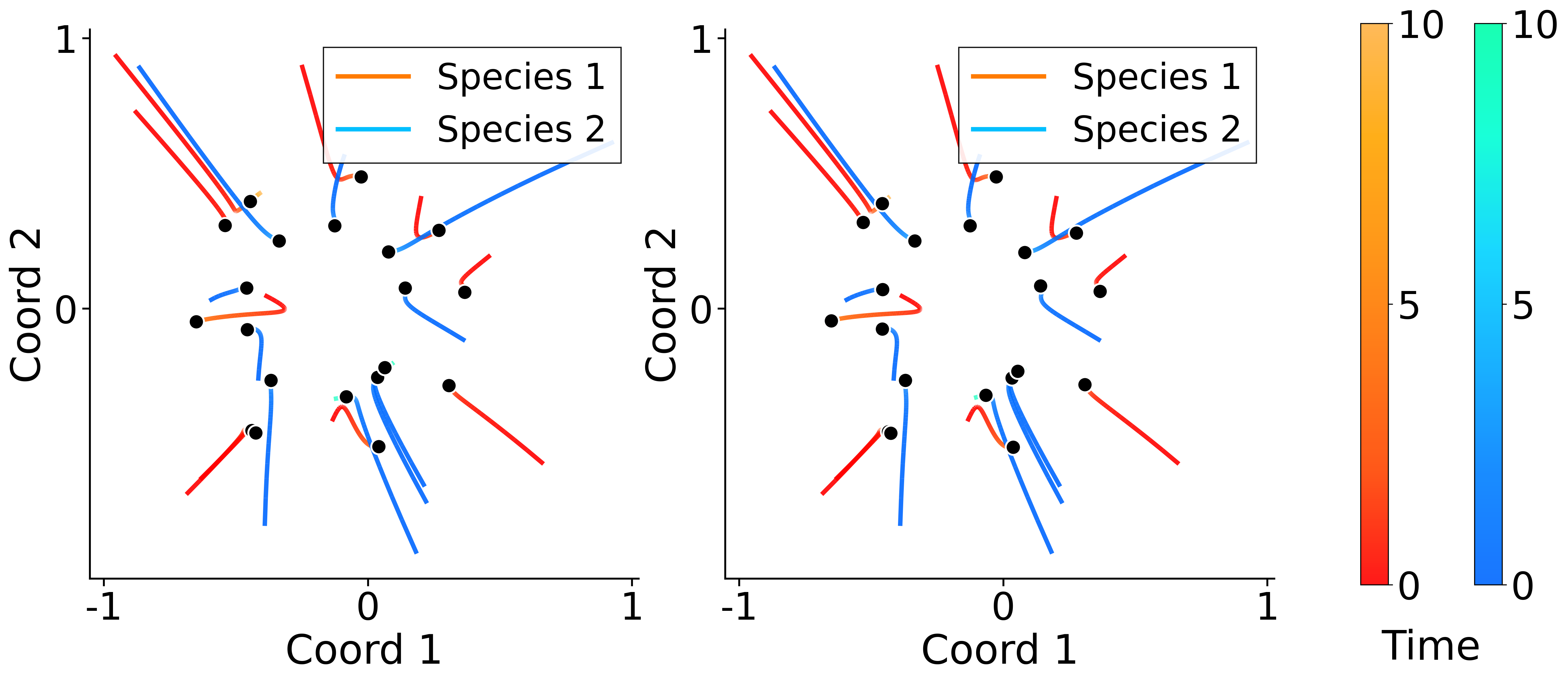}
}

\caption{Results of kernel learning for the linear-repulsive potential dynamics with $N_1 = N_2 = 10$, $L=5$, and $M=5$ with noise $\sigma = 0.01$. Left, Center: The four interaction kernels are shown with true function in black and predicted mean in blue, with the shaded region indicating the standard deviation band. Gray bars show the empirical distribution of pairwise distances. Right: Training and testing data trajectory prediction plots on $[0,2T]$ are presented, with the true dynamics on the left of each pair and the predicted dynamics on the right. A black dot marks each trajectory at the time snapshot $t=T$. The top pair utilizes a training trajectory to test temporal generalization while the bottom pair uses test data. The predicted interaction functions are sufficiently accurate to closely reconstruct the true dynamics.}
\label{fig:repulsivek_figs}
\end{figure}

A common issue facing Gaussian process methods is the slow computation of large-scale problems. One approach to scaling is to learn interaction potentials from smaller systems and transfer the results to the prediction of larger systems. We show the effectiveness of this learning acceleration technique in Figure \ref{fig:repulsivek_scaledup}. While kernels are learned on the smaller $N_1 = N_2 = 10$ setting, accurate prediction of dynamics for $N_1 = N_2 = 100$ is possible with the same functions, requiring no additional training time and allowing for extension to very large systems with only the computational cost of an ODE solver.

\begin{figure}[tbhp]
\centering

\subfigure[$N_1 = N_2 = 10$]{
\includegraphics[width=0.31\linewidth]{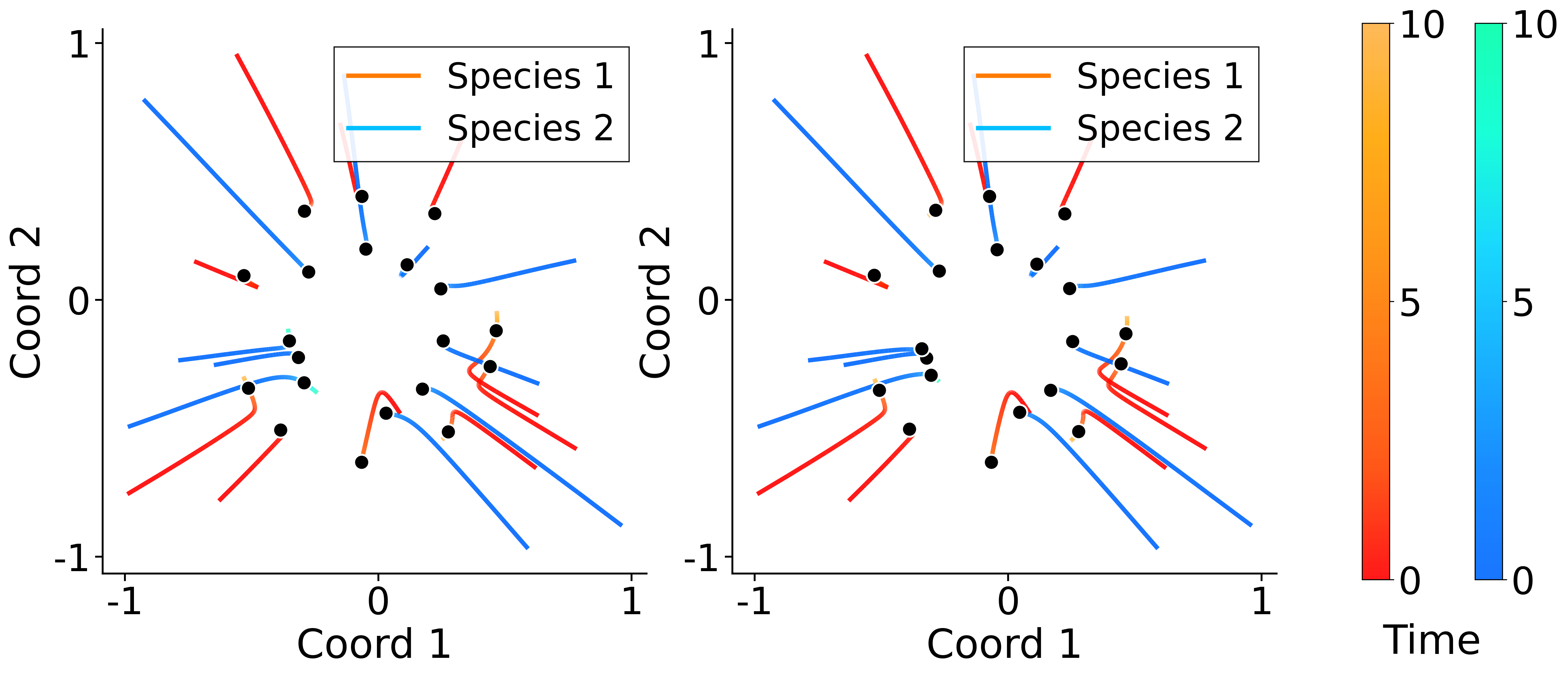}
}
\hfill
\subfigure[$N_1 = N_2 = 50$]{
\includegraphics[width=0.31\linewidth]{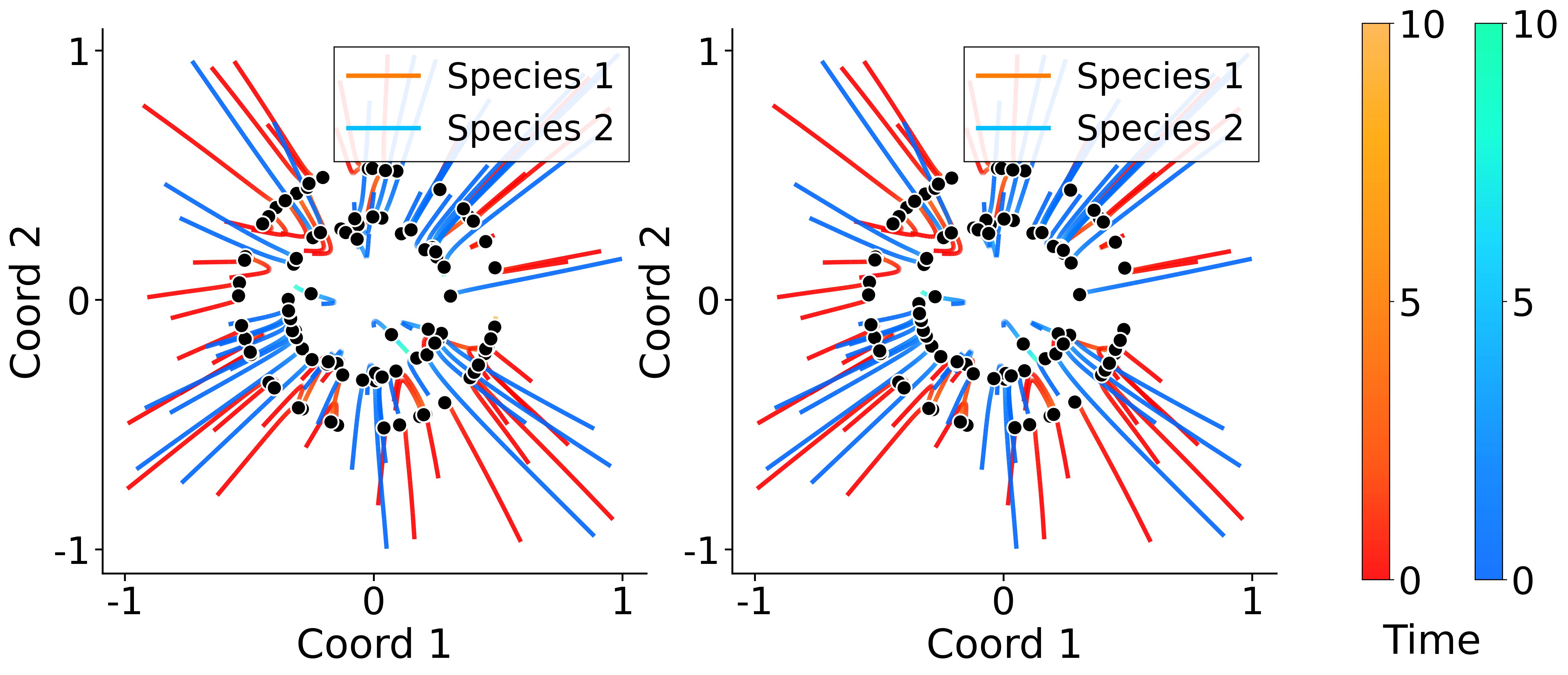}
}
\hfill
\subfigure[$N_1 = N_2 = 100$]{
\includegraphics[width=0.31\linewidth]{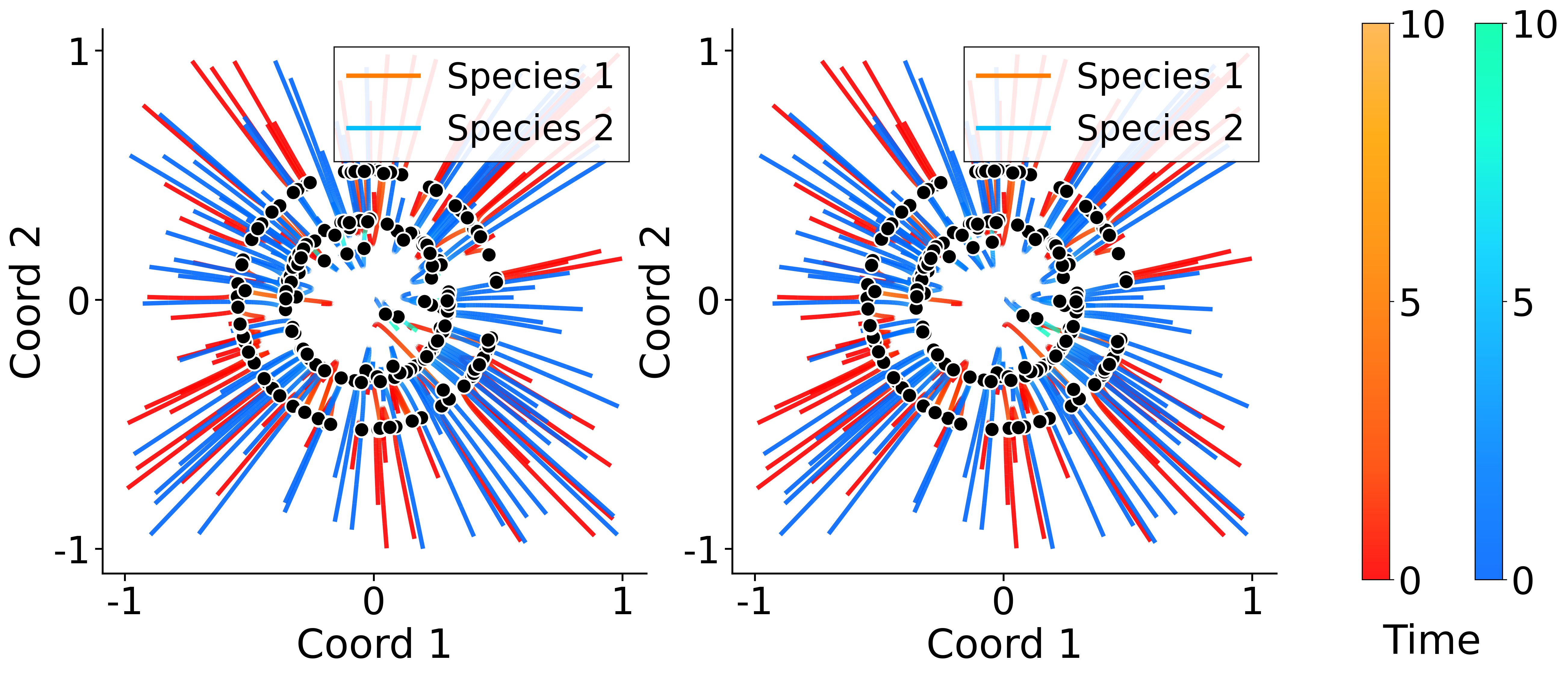}
}

\caption{Kernel learning result for linear-repulsive potentials, with learned kernels from $N_1 = N_2 = 10$ used for dynamics prediction on systems with larger numbers of particles, $N_1 = N_2 = 50$ and $N_1 = N_2 = 100$ . Learned kernels transfer well and predict dynamics with high fidelity.}
\label{fig:repulsivek_scaledup}
\end{figure}

\subsection{Predator-Prey Interactions}\label{sec: pred}

In this experiment, we consider the predator-prey dynamics of \cite{chen2014predatorswarm}. These interaction potentials are fundamentally different than the repulsive interactions of Experiment \ref{sec: numerical_sec_1}, as particles of the prey species exhibit an attractive interaction force, while cross-species interactions remain repulsive. Additionally, predators exhibit no intra-species force, with the true interaction potential remaining identically zero. This type of model, as the name suggests, is primarily inspired by applications from mathematical biology in the flocking behaviors of animals in the presence of predators, which has been extensively studied and continues to attract attention \cite{parrish1999complexity, ballerini2008empirical, abaid2010fish, chakraborty2020predatorprey}. The resulting dynamics can exhibit several steady-state behaviors depending on the parameters utilized. We present the true interaction functions in Table \ref{tab:ex_info_prey_kernels}.

\begin{table}[h!]
\caption{True interaction kernels for predator-prey dynamics.}
\label{tab:ex_info_prey_kernels} 
\begin{center}
\begin{tabular}{ c c c}
\hline
 System & Predator-Prey  \\
 \hline 
$\phi^{11}$ &  $r^{-2} - a$   \\

$\phi^{12}$ &  $b r^{-2}$  \\

$\phi^{21}$ &  $-c r^{-p}$  \\

$\phi^{22}$ &  0  \\
\hline 
\end{tabular}
\end{center}
\end{table}

We examine two particular solution behaviors. First, we choose $a = 1, b=3.0, c=0.2, p=2.5$ leading to a migratory solution where prey flocks and flees from the chasing predators. Second, we choose $a = 1, b=3.4, c=0.9, p=2.5$ leading to the formation of a ring of prey animals which capture the predators in the center, leading to rapid motion which is highly sensitive to small changes in the interaction forces. We truncate all singular interaction potentials at $r = 0.5$ by a function of the form $ae^{-br}$ to ensure potentials are well-behaved near the origin.

These dynamics are both extremely sensitive to small changes in the interaction potentials, as even minor differences in regions of low data support can result in different macroscopic steady state behaviors, such as different migration directions or reversed ring orbit patterns. As such, while previous examples have been able to exhibit satisfactory trajectory prediction errors with default hyperparameters for the Mat\'ern kernel, optimization is necessary for extremely high-accuracy predictions here. This underscores the necessity of the data-driven Gaussian Process approach, as a default kernel method with less accurate prediction of interactions fails to learn sufficiently well. We utilize $50$ iterations of L-BFGS optimization on the log likelihood, which optimizes all Matérn amplitudes and length-scales jointly for the four kernels, as the objective uses the exact GP marginal likelihood with $O(n^3)$ cost per kernel (where $n$ is the number of distance samples). We first show the performance of these optimized kernels for the migratory dynamics with $N_1 = 20$ prey, $N_2 = 3$ predators, $L = 10$ timesteps, and $M = 3$ trajectories with $\sigma = 0.01$. Dynamics are evolved on a longer timescale with $T=25$. In Figure \ref{fig:prey_figs}, we show the qualitative behavior of the learned kernels and the generated predator-prey dynamics after full optimization.

\begin{figure}[htp]
\centering
\subfigure[$\phi^{11}$]{
\includegraphics[width=0.24\linewidth]{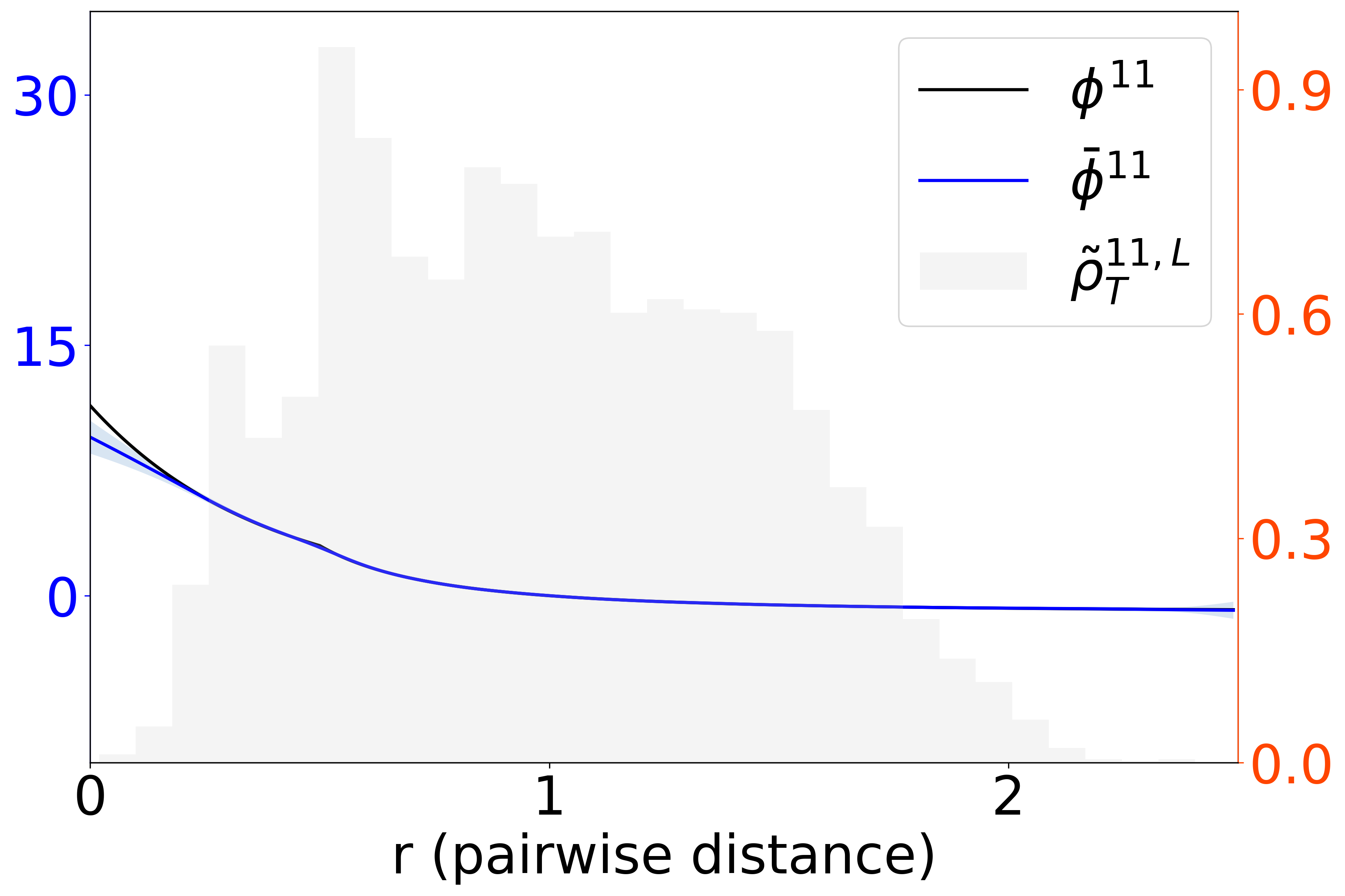}
}
\hfill
\subfigure[$\phi^{12}$]{
\includegraphics[width=0.24\linewidth]{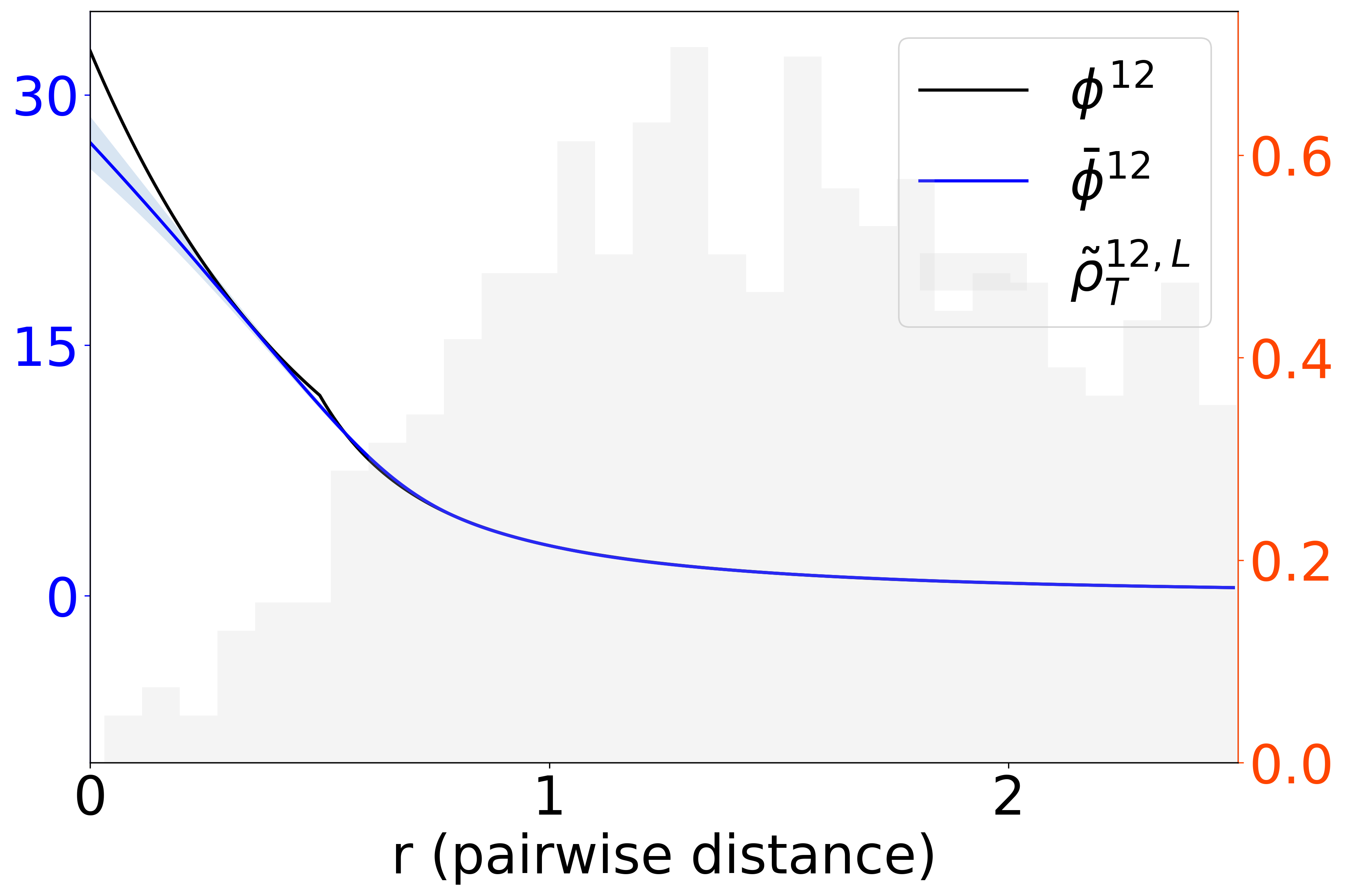}
}
\hfill
\subfigure[Training Data]{
\includegraphics[width=0.4\linewidth]{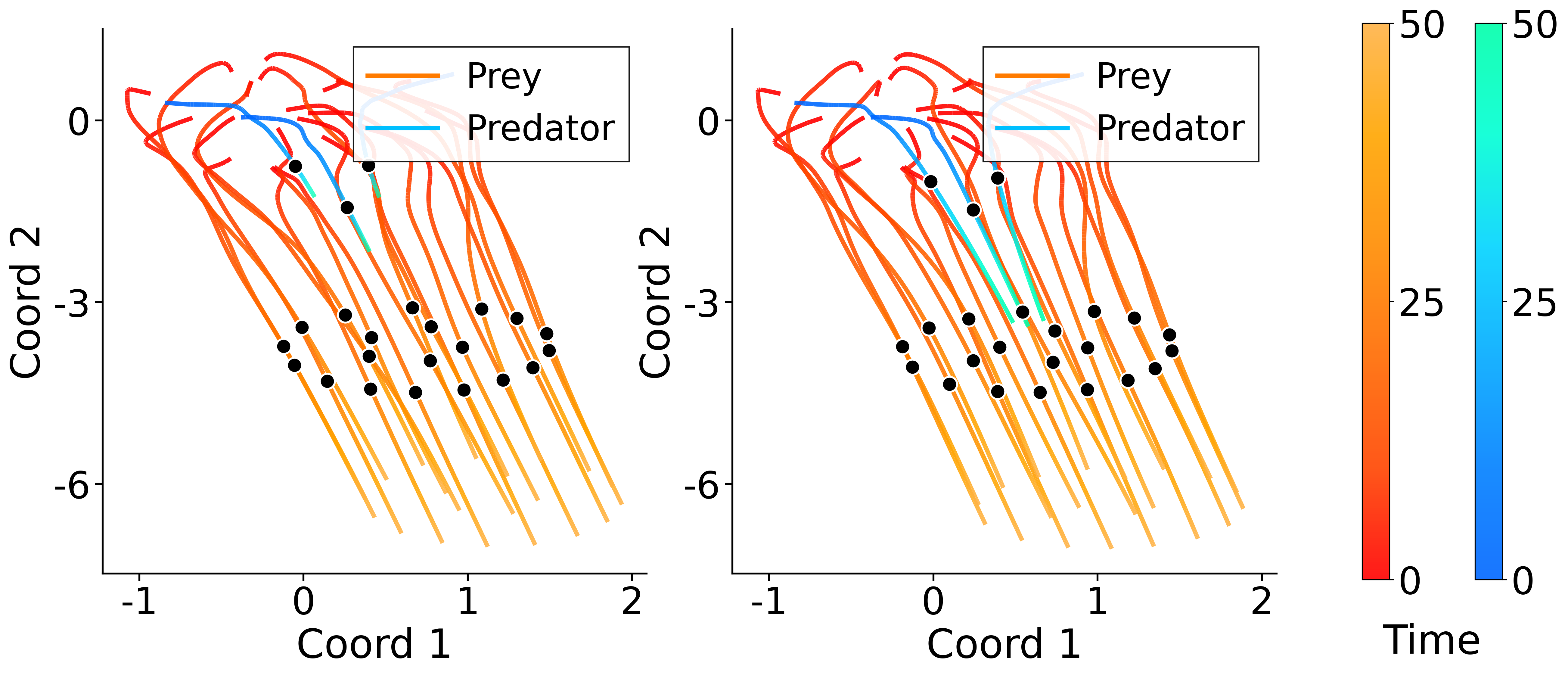}
}

\vspace{0.5em}

\subfigure[$\phi^{21}$]{
\includegraphics[width=0.24\linewidth]{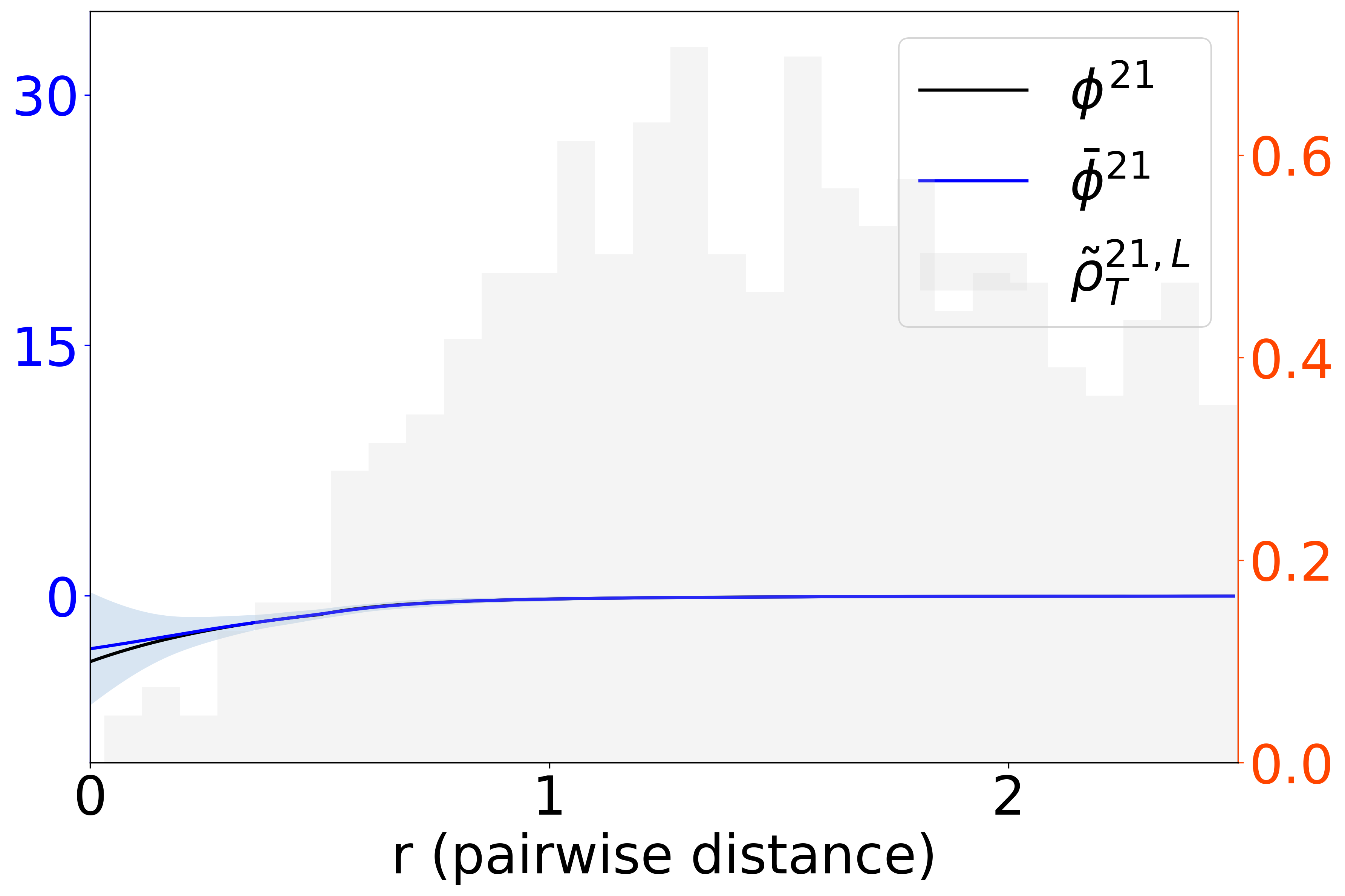}
}
\hfill
\subfigure[$\phi^{22}$]{
\includegraphics[width=0.24\linewidth]{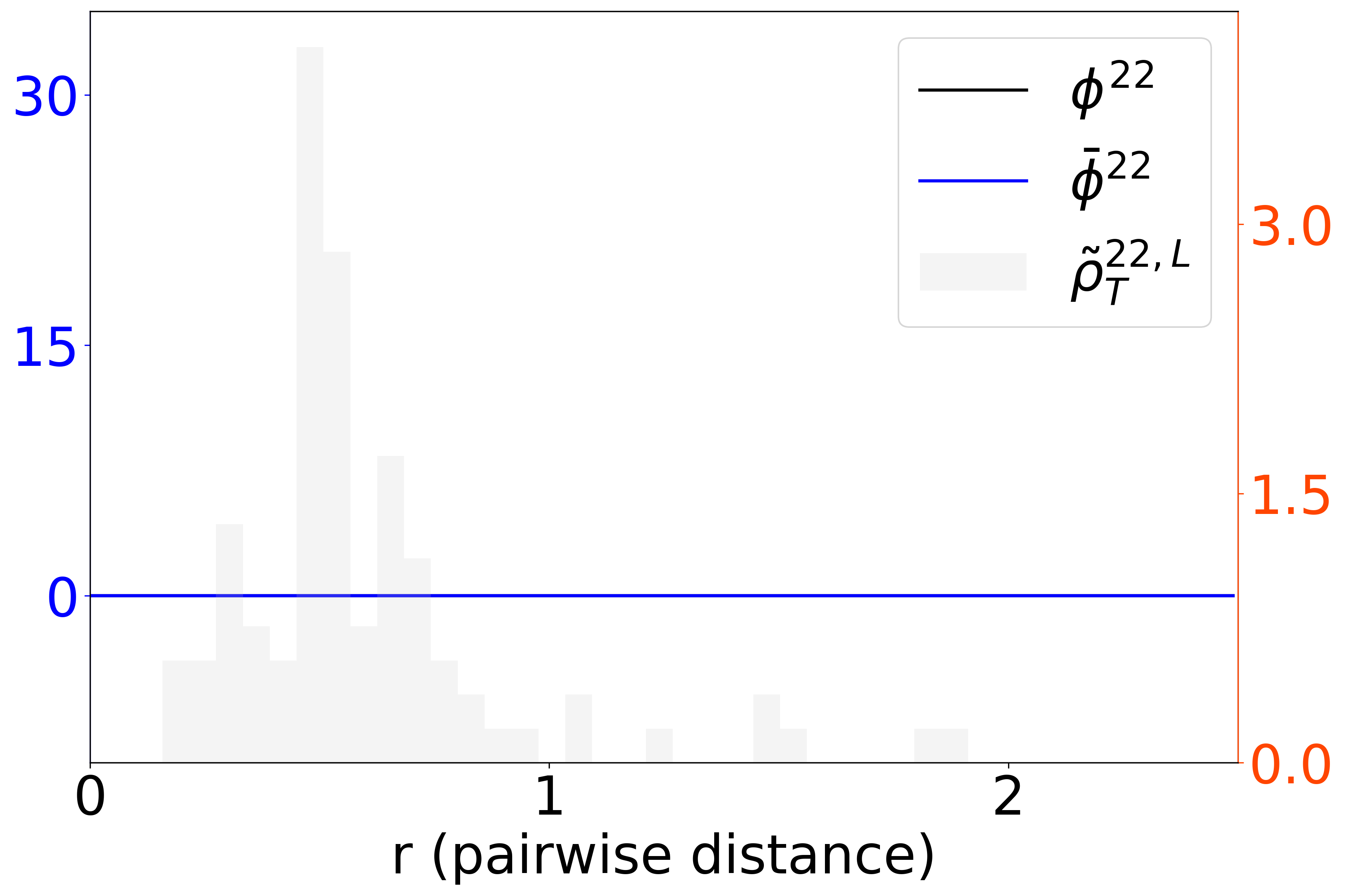}
}
\hfill
\subfigure[Testing Data]{
\includegraphics[width=0.4\linewidth]{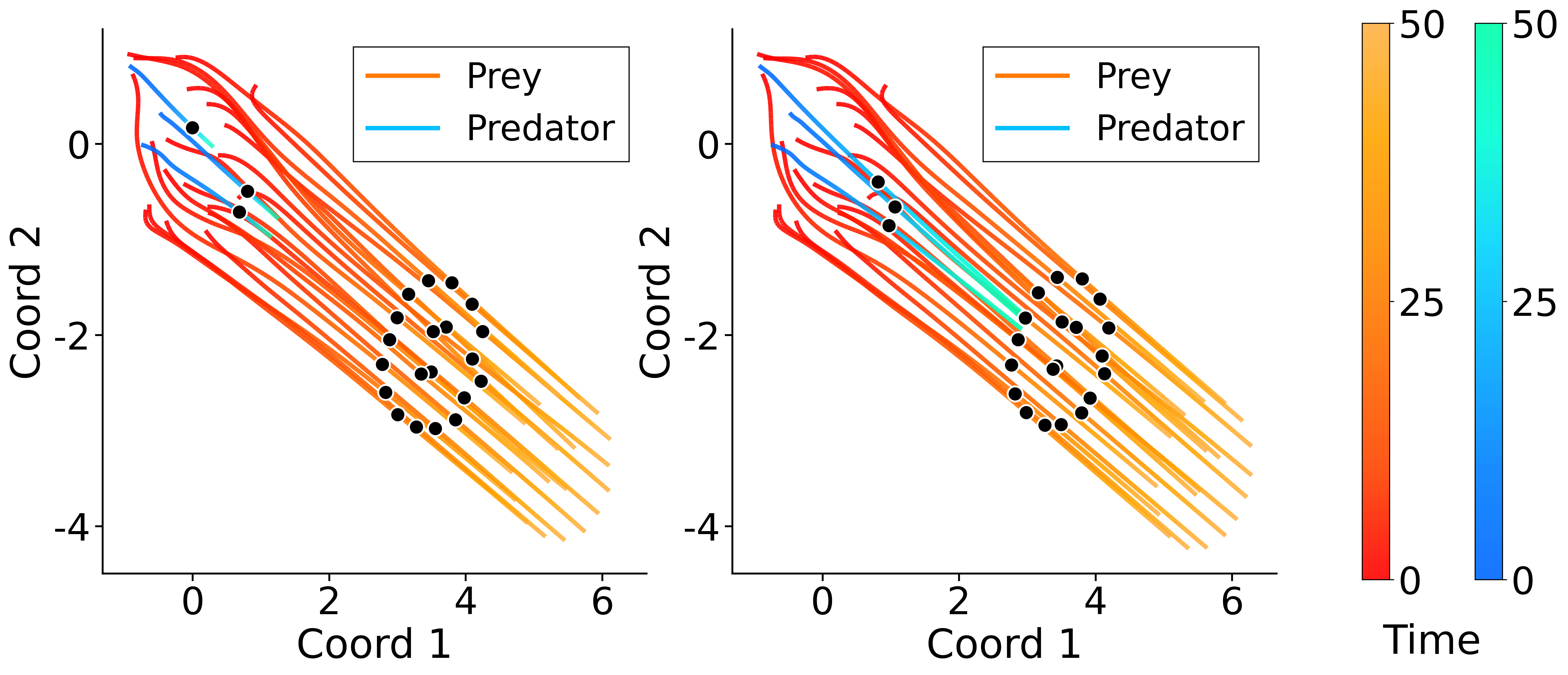}
}

\caption{Results of kernel learning for the migratory predator-prey dynamics with $N_1 = 20, N_2 = 3, L = 10, M = 3$ and noise $\sigma = 0.01$. The four interaction kernels are shown, with true function in black while predicted mean is in blue, with the blue shaded region indicating the standard deviation band. Gray bars show the empirical distribution on the learning dataset. Note that the predicted $\phi^{22}$ is correctly estimated to be very close to zero.}
\label{fig:prey_figs}
\end{figure}

\begin{figure}[htp]
\centering
\subfigure[$\phi^{11}$]{
\includegraphics[width=0.24\linewidth]{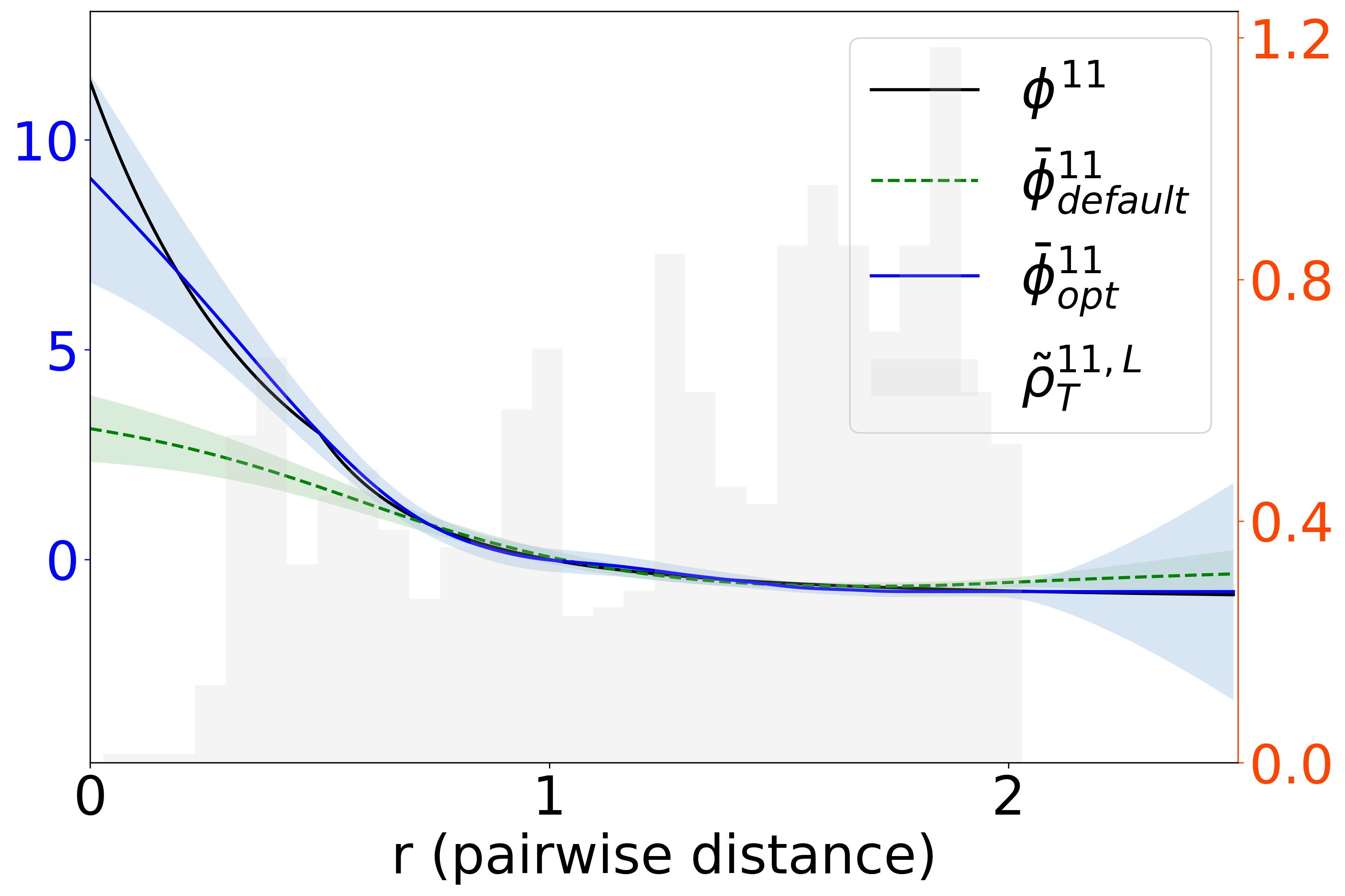}
}
\hfill
\subfigure[$\phi^{12}$]{
\includegraphics[width=0.24\linewidth]{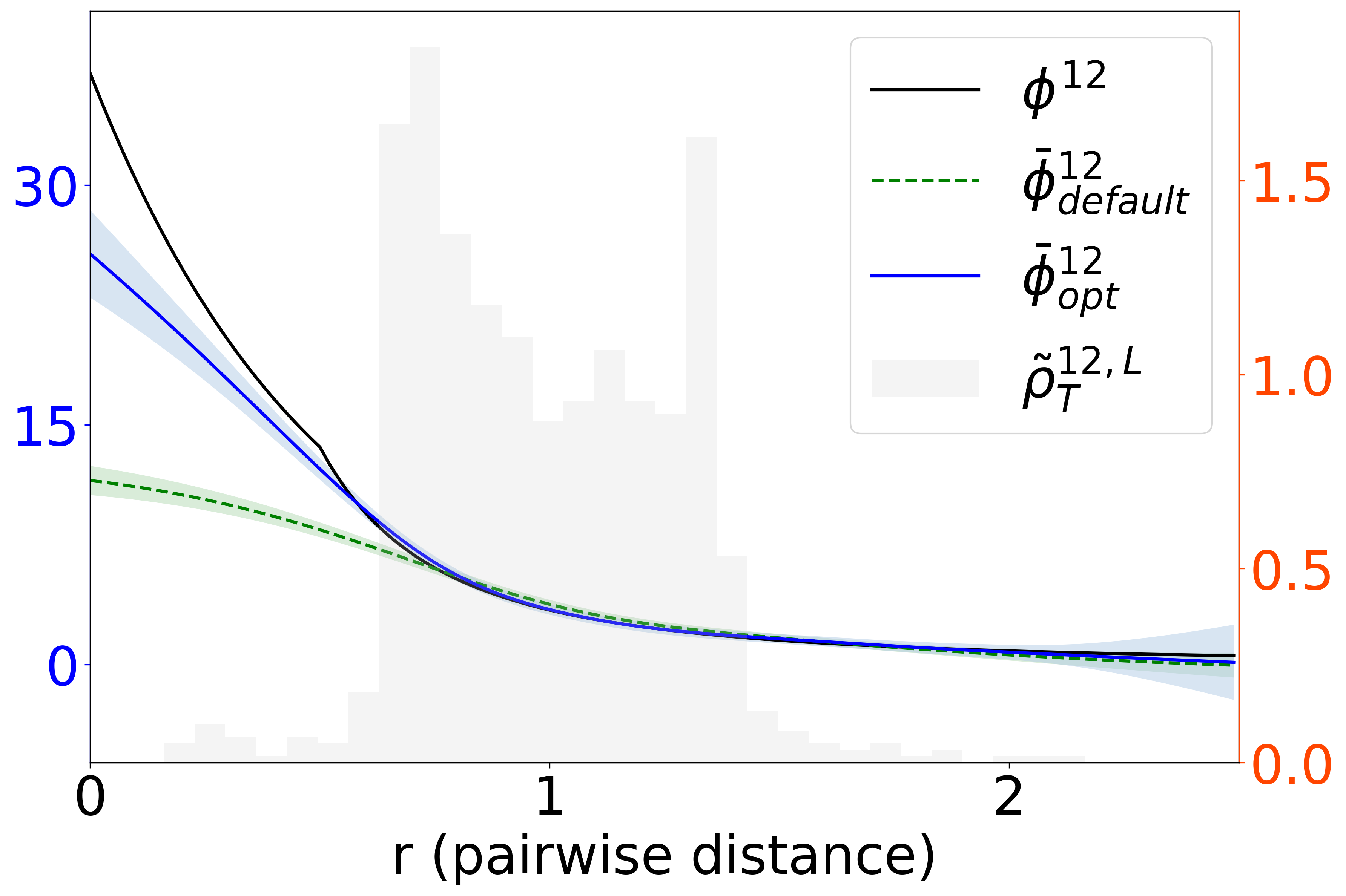}
}
\hfill
\subfigure[Training Data]{
\includegraphics[width=0.4\linewidth]{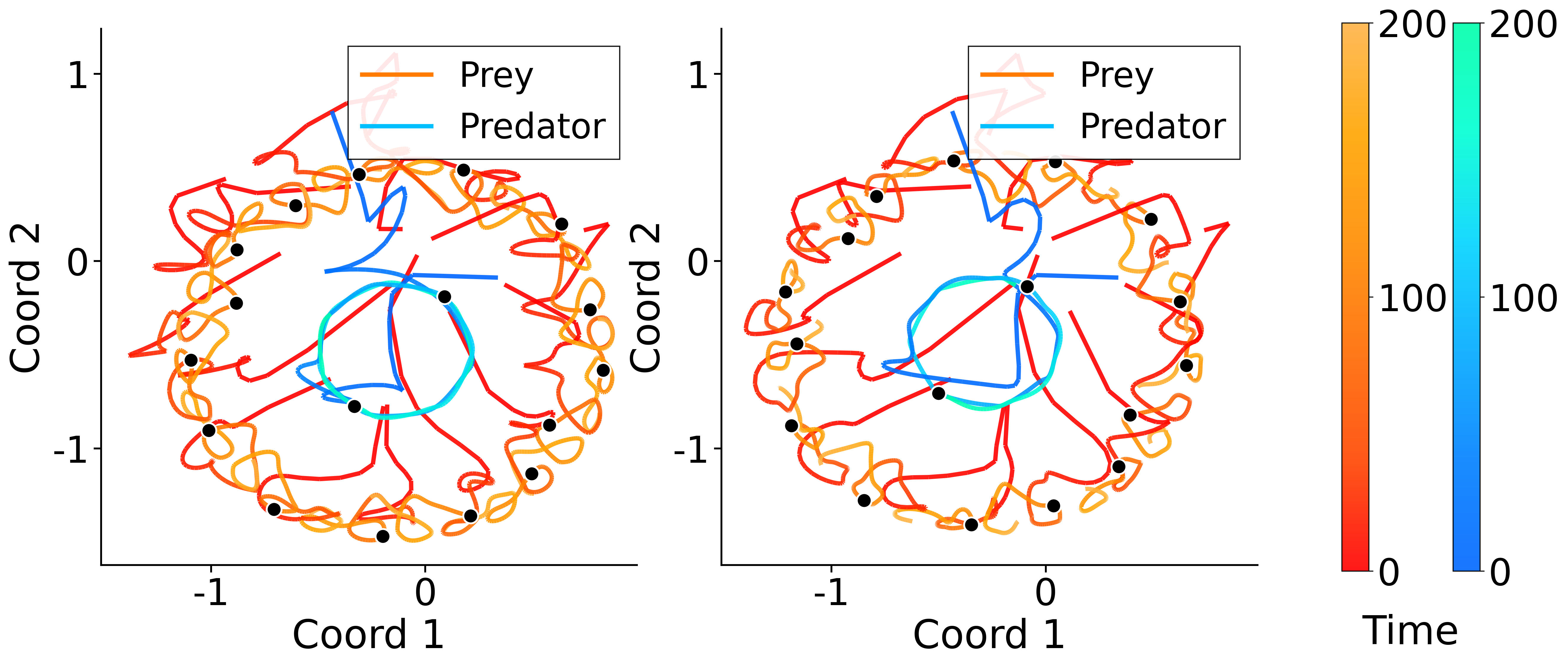}
}

\vspace{0.5em}

\subfigure[$\phi^{21}$]{
\includegraphics[width=0.24\linewidth]{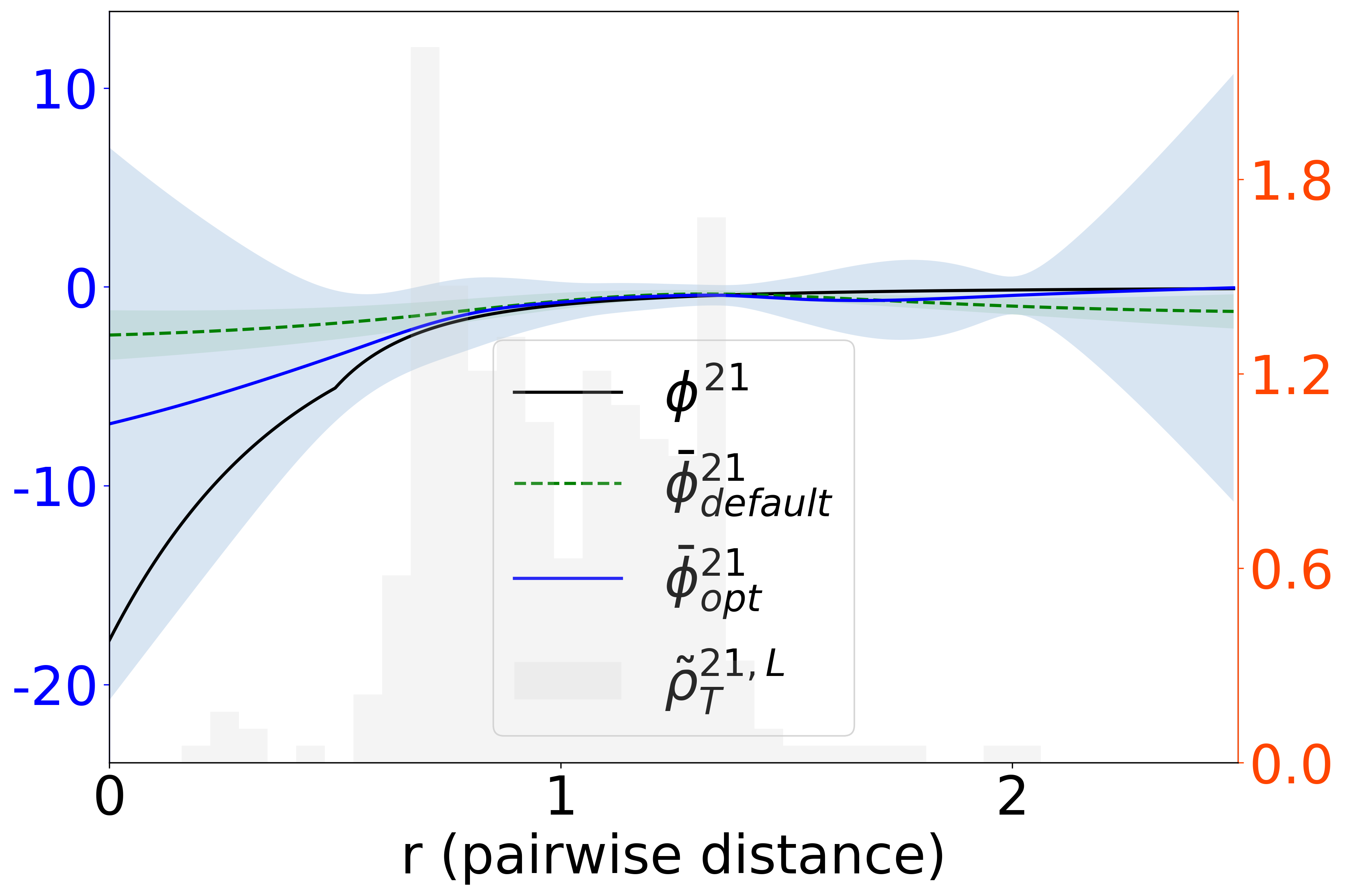}
}
\hfill
\subfigure[$\phi^{22}$]{
\includegraphics[width=0.24\linewidth]{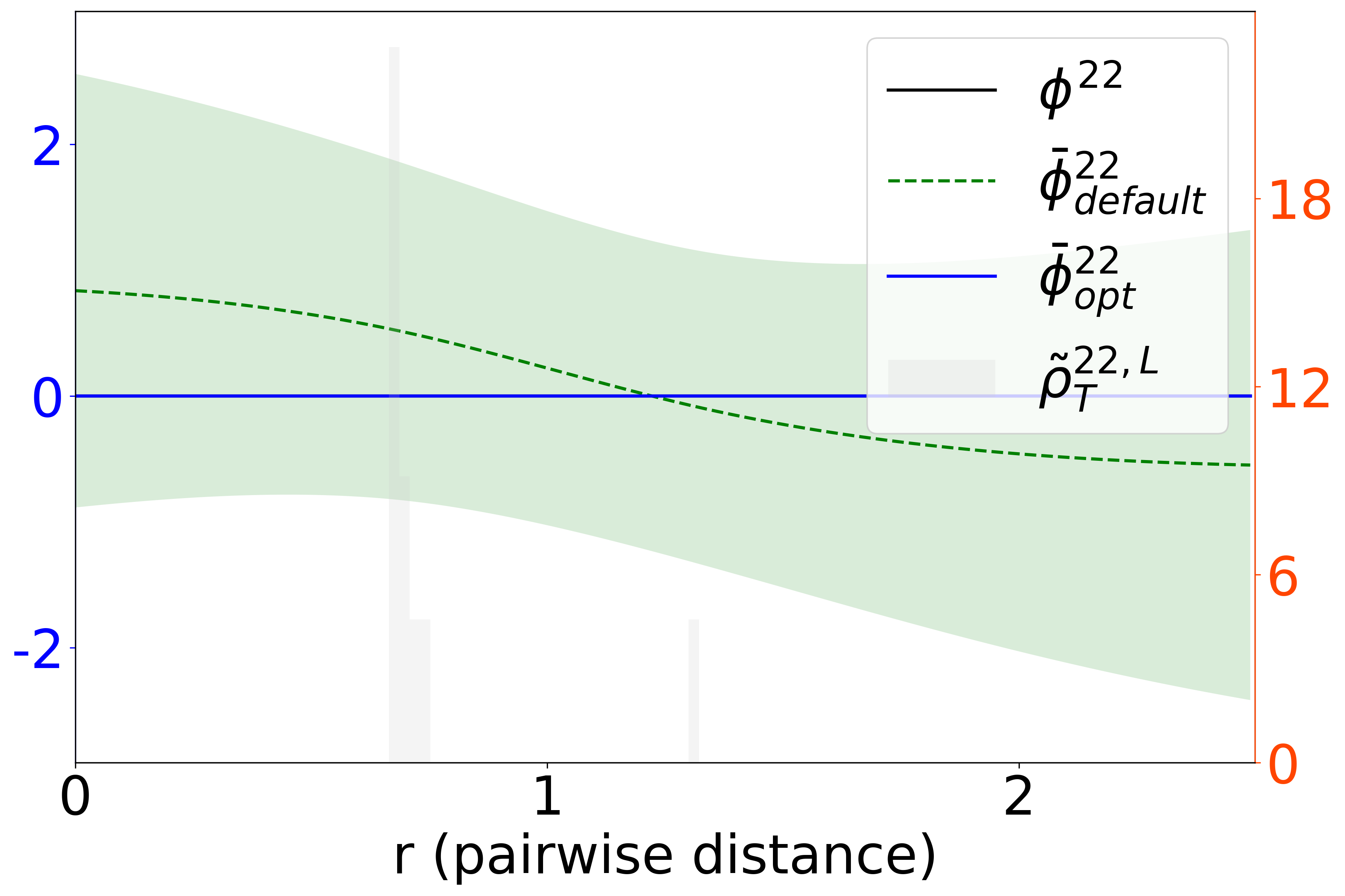}
}
\hfill
\subfigure[Testing Data]{
\includegraphics[width=0.4\linewidth]{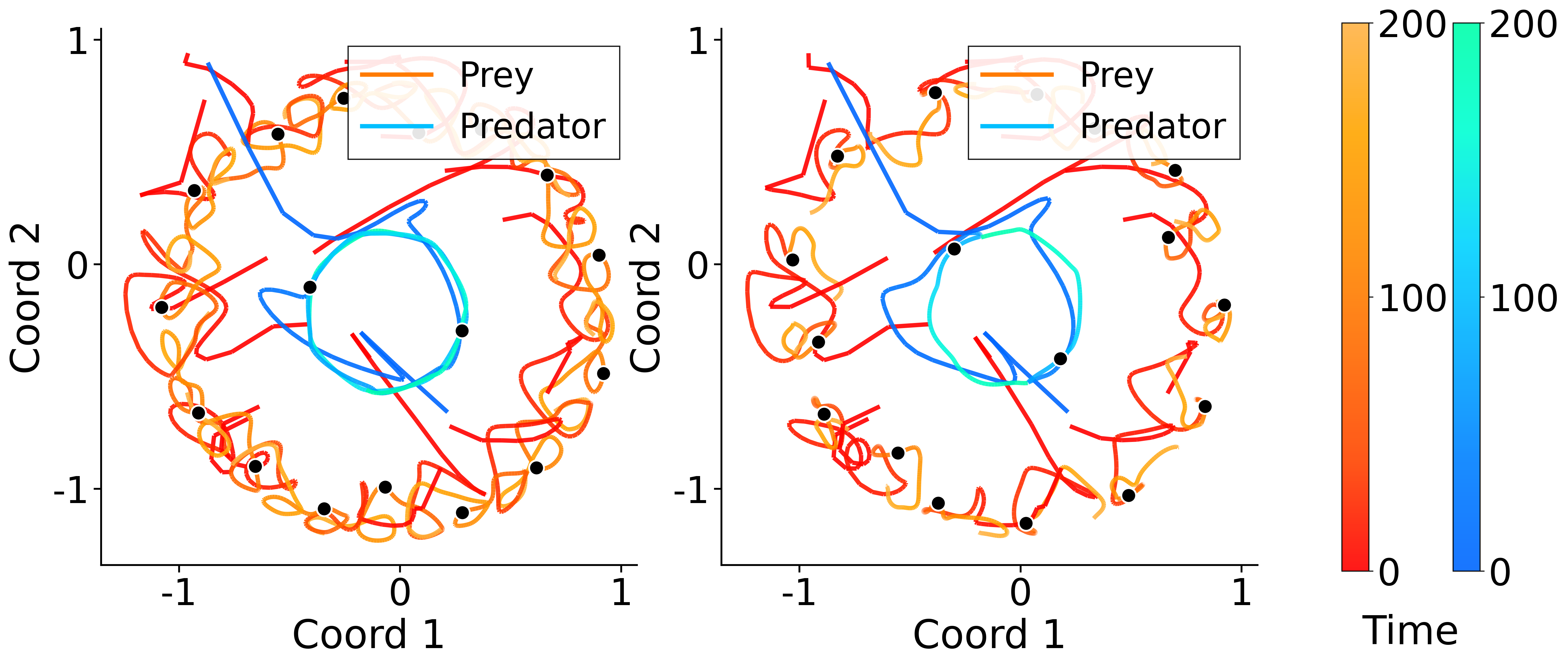}
}

\caption{Results of kernel learning for the ring formation predator-prey dynamics with $N_1 = 15, N_2 = 2, L = 10, M = 1$ and noise $\sigma = 0.01$. For each of the four interaction potentials, the default parameter predictions are plotted in dotted green, and the optimized interaction potentials are plotted in blue. Plots are zoomed in to clearly show the differences between the default and learned kernels. Gray bars show the empirical distribution on the learning dataset. Note that the predicted $\phi^{22}$ is correctly estimated to be close to zero after optimization. On the right, the dynamics predictions with the optimized interaction potentials are shown.}
\label{fig:prey_figs_comp}
\end{figure}

The optimization provides sufficiently accurate interaction potential predictions to allow for meaningful simulation of the long-term system behavior. Of note is the extremely accurate $\phi^{22}$ prediction, as the optimization correctly sends the corresponding hyperparameters very close to zero as there is no interaction force present. To further show the impact of the optimization process, we plot the predicted interaction functions with our optimized hyperparameters against the default hyperparameters in Figure \ref{fig:prey_figs_comp}, for the ring formation dynamics.

The accurate long-term prediction of the ring-formation patterns with $T = 100$ shows the effectiveness of the Gaussian process approach. While optimization of kernel parameters is fairly expensive per iteration with cubic cost, it is only necessary in cases of relatively small data, such as a single training trajectory as in Figure \ref{fig:prey_figs_comp}. For larger data regimes such as Experiments \ref{sec: repulsive} and \ref{sec: repulsiveK} where optimization would require a significant time investment, default parameters suffice to provide accurate predictions, allowing for the Gaussian process approach to flex in response to the problem requirements.

\section{Conclusion and Future Work}  
\label{sec:conclusion}

In this paper, we have developed a Gaussian process framework for learning interaction kernels in multi-species interacting particle systems. Building on our earlier work on single-species and second-order models, we established a complete statistical learning theory for both intra- and inter-species kernels. Our analysis provides recoverability, quantitative error bounds, and statistical optimality of posterior estimators, thereby unifying and extending the theory for data-driven inference of interacting particle systems. The numerical experiments corroborated the theoretical predictions and highlighted the advantages of the proposed approach over existing methods.  

Several promising directions for future research remain. First, it would be natural to extend the present framework to systems with stochastic perturbations, where uncertainty quantification plays an even more central role. Second, while our analysis focused on pairwise interactions, many real systems involve more complex multi-body or state-dependent forces; incorporating such effects into the GP framework is an important open problem.  Third, applications to empirical data, ranging from ecological predator–prey dynamics to multi-class pedestrian flows, would further demonstrate the practical utility of the methodology.

\bibliographystyle{siam}
\bibliography{ref}

\newpage
\appendix
\section*{Appendix}

\section{Operator-theoretical framework for the statistical inverse problem}\label{appendix:operator}
In our analysis, we make the following assumption. 
\begin{assumption}\label{assumptionmeasure}
For all $i,i' \in \{1,\dots,N\}$, the displacement vector $\br_{ii'}$ belongs to $L^2(\R^{dN};\rho_{\bX};\R^d)$.
\end{assumption}

Assumption~\ref{assumptionmeasure} is mild: it is satisfied whenever the distribution of initial conditions is compactly supported or decays sufficiently fast.  
We prove Theorem \ref{maingp} by deriving a Representer theorem (see Theorem \ref{representerthm}) for the empirical risk functional \eqref{regularizedrisk1}, which is also applicable to the risk functional \eqref{kr}.

To begin, we analyze relevant operators that are useful for representing the minimizers to the risk functionals. 

\begin{lemma}\label{lem1} 
For any $\intkernelvar^{pq} \in L^2(\tilde \rho_T^{pq,L})$ we have
\begin{equation}
  \|\rhsfo_{\intkernelvar^{pq}}\|_{L^2(\rho_{\bX})}^2 \;\leq\; \frac{N-1}{N}\, \|\intkernelvar^{pq}\|^2_{L^2(\tilde \rho_T^{pq,L})}.
\end{equation}
\end{lemma}

\begin{proof}
This is a direct adaptation of Proposition~16 in \cite{lu2021learning}, specialized to $K=1$.
\end{proof}

\begin{proposition}\label{propertyA}
Let $A$ be a linear operator defined by $$A\bintkernelvar=\rhsfo_{\bintkernelvar},\quad \bintkernelvar = (\varphi^{11},\varphi^{12},\varphi^{21},\varphi^{22})$$ that maps $\prod_{p,q} \mHpq$ to $L^2(\mathbb{R}^{dN};\rho_{\bX};\mathbb{R}^{dN})$. Then $A$ is bounded and its adjoint operator 
$A^{*}$ satisfies
\begin{align}\label{adjoint}
A^{*}g &=\Bigg(\int_{\bX}\frac{1}{N^2}\sum_{i=1}^{N_1}\sum_{i'=1}^{N_1} K_{r_{ii'}}^{11}\langle \br_{ii'},g_{i}(\bX) \rangle\, d\rho_{\bX},
\int_{\bX}\frac{1}{N^2}\sum_{i=1}^{N_1}\sum_{i'=N_1+1}^{N} K_{r_{ii'}}^{12}\langle \br_{ii'},g_{i}(\bX) \rangle\, d\rho_{\bX},\\
&\quad \int_{\bX}\frac{1}{N^2}\sum_{i=N_1+1}^{N}\sum_{i'=1}^{N_1} K_{r_{ii'}}^{21}\langle \br_{ii'},g_{i}(\bX) \rangle\, d\rho_{\bX},
\int_{\bX}\frac{1}{N^2}\sum_{i=N_1+1}^{N}\sum_{i'=N_1+1}^{N} K_{r_{ii'}}^{22}\langle \br_{ii'},g_{i}(\bX) \rangle\, d\rho_{\bX}
\Bigg),
\end{align} where
$g=[g_1^T,\cdots,g_{N}^T]^T$ with $g_i:\mathbb{R}^{dN} \rightarrow \mathbb{R}^d$. As a consequence, 
{\small
\begin{align}\label{positive}
&B\bintkernelvar :=A^{*}A\bintkernelvar
=\notag\\
&\Bigg(\int_{\bX}\frac{1}{N^3}\sum_{i=1}^{N_1}\sum_{i'=1}^{N_1} K_{r_{ii'}}^{11}( \sum_{i''=1}^{N_1} \langle \intkernelvar^{11}, K_{r_{ii''}}^{11}\rangle_{\mH^{11}}\langle \br_{ii'}, \br_{ii''}\rangle + \sum_{i''=N_1+1}^{N} \langle \intkernelvar^{12}, K_{r_{ii''}}^{12}\rangle_{\mH^{12}}\langle \br_{ii'}, \br_{ii''}\rangle)\, d\rho_{\bX},\notag\\
&\int_{\bX}\frac{1}{N^3}\sum_{i=1}^{N_1}\sum_{i'=N_1+1}^{N} K_{r_{ii'}}^{12}( \sum_{i''=1}^{N_1} \langle \intkernelvar^{11}, K_{r_{ii''}}^{11}\rangle_{\mH^{11}}\langle \br_{ii'}, \br_{ii''}\rangle + \sum_{i''=N_1+1}^{N} \langle \intkernelvar^{12}, K_{r_{ii''}}^{12}\rangle_{\mH^{12}}\langle \br_{ii'}, \br_{ii''}\rangle)\, d\rho_{\bX},\notag\\
&\int_{\bX}\frac{1}{N^3}\sum_{i=N_1+1}^{N}\sum_{i'=1}^{N_1} K_{r_{ii'}}^{21}( \sum_{i''=1}^{N_1} \langle \intkernelvar^{21}, K_{r_{ii''}}^{21}\rangle_{\mH^{21}}\langle \br_{ii'}, \br_{ii''}\rangle + \sum_{i''=N_1+1}^{N} \langle \intkernelvar^{22}, K_{r_{ii''}}^{22}\rangle_{\mH^{22}}\langle \br_{ii'}, \br_{ii''}\rangle)\, d\rho_{\bX},\notag\\
&\int_{\bX}\frac{1}{N^3}\sum_{i=N_1+1}^{N}\sum_{i'=N_1+1}^{N} K_{r_{ii'}}^{22}( \sum_{i''=1}^{N_1} \langle \intkernelvar^{21}, K_{r_{ii''}}^{21}\rangle_{\mH^{21}}\langle \br_{ii'}, \br_{ii''}\rangle + \sum_{i''=N_1+1}^{N} \langle \intkernelvar^{22}, K_{r_{ii''}}^{22}\rangle_{\mH^{22}}\langle \br_{ii'}, \br_{ii''}\rangle)\, d\rho_{\bX}\Bigg),
\end{align}
}
is a trace class operator mapping
$\prod_{p,q} \mHpq$ to $\prod_{p,q} \mHpq$. In addition, $B$ can be also viewed as a bounded linear operator from $L^2( \tilde\rho_T^{L})$ to $L^2( \tilde\rho_T^{L})$. 
\end{proposition}

\begin{proof} Since $\prod_{p,q} \mHpq$ can be naturally embedded as a subspace of $L^2( \tilde\rho_T^{L})$, using Lemma \ref{lem1} and Lemma \ref{infbound}, we have that 
\begin{align}\label{opinequality} \Rhoxnorm{A\bintkernelvar}^2&=\Rhoxnorm{\rhsfo_{\bintkernelvar}}^2 \leq 2(\sum_{p,q} \Rhoxnorm{\rhsfo_{\intkernelvar^{pq}}}^2) \notag\\
&\leq \frac{2(N-1)}{N}   (\sum_{p,q}\|\intkernelvar^{pq}\|^2_{L^2(\tilde \rho_T^{pq,L})} ) \notag\\
&< \sum_{p,q}2R^2\|\intkernelvar^{pq}\|_{\infty}^2 \notag\\
&\leq \sum_{p,q} 2\kappa_{pq}^2R^2\|\intkernelvar^{pq}\|_{\mathcal{H}_{{K}^{pq}}}^2.
\end{align}

 This shows that $A$ is a bounded linear operator mapping  $\prod_{p,q} \mHpq$ to $L^2(\mathbb{R}^{dN};\rho_{\bX};\mathbb{R}^{dN})$. 

Next, we prove \eqref{adjoint}. We first show that the map for each corresponding $(i,i')$ and $(p,q)$
$$\bX \rightarrow  K_{r_{ii'}}^{pq} \in {\mathcal{H}_{{K}^{pq}}},
$$ is continuous since   $\|K_{r_{ii'}}^{pq}-K_{r'_{ii'}}^{pq}\|_{\mHpq}^2=K^{pq}(r_{ii'}, r_{ii'})+K^{pq}(r'_{ii'}, r'_{ii'})-2K^{pq}(r_{ii'},r'_{ii'})$ for all $r_{ii'} = \|\bx_i-\bx_{i'}\|$, $r'_{ii'} = \|\bx'_i-\bx'_{i'}\|$, and $\bX,\bX' \in \mathbb{R}^{dN}$, and both $\mK^{pq}$ and $\|\cdot\|$ are  continuous for all $p,q$. 
Hence given a function $g \in L^2(\mathbb{R}^{dN};\rho_{\bX};\mathbb{R}^{dN})$, the map
\begin{align}
\bX &\rightarrow (\frac{1}{N^2}\sum_{i=1}^{N_1}\sum_{i'=1}^{N_1} K_{r_{ii'}}^{11}\langle \br_{ii'},g_{i}(\bX) \rangle,
\frac{1}{N^2}\sum_{i=1}^{N_1}\sum_{i'=N_1+1}^{N} K_{r_{ii'}}^{12}\langle \br_{ii'},g_{i}(\bX) \rangle,\\
&\quad \frac{1}{N^2}\sum_{i=N_1+1}^{N}\sum_{i'=1}^{N_1} K_{r_{ii'}}^{21}\langle \br_{ii'},g_{i}(\bX) \rangle,
\frac{1}{N^2}\sum_{i=N_1+1}^{N}\sum_{i'=N_1+1}^{N} K_{r_{ii'}}^{22}\langle \br_{ii'},g_{i}(\bX) \rangle
),
\end{align}
is measurable from $\mathbb{R}^{dN}$ to $\prod_{p,q} \mHpq$.  
Moreover,
$$\|\frac{1}{N^2}\sum_{(i,i') \in \mathcal{N}_{pq}}\mK_{r_{ii'}}^{pq}\langle \br_{ii'},g_{i}(\bY)\rangle\|_{\mHpq} \leq \frac{\kappa_{pq}}{N^2}\sum_{i=1,i'\neq i}^{N} |\langle  \br_{ii'},g_{i}(\bX)\rangle|, $$
for all $(p,q)$, with $\mathcal{N}_{11} = \{(i,i')| 1\leq i\leq N_1,1\leq i'\leq N_1\}$, $\mathcal{N}_{12} = \{(i,i')| 1\leq i\leq N_1,N_1+1\leq i'\leq N\}$, and similarly for $\mathcal{N}_{2q},\  q=1,2$.

Since $\rho_{\bX}$ is finite and $\langle  \br_{ii'},g_{i}(\bX)\rangle$ is in $L^1(\mathbb{R}^{dN};\rho_{\bX};\mathbb{R})$, hence $(\frac{1}{N^2}\sum_{(i,i') \in \mathcal{N}_{pq}}\mK^{pq}_{r_{ii'}}\langle \br_{ii'},g_{i}(\bX)\rangle)$ is integrable, as a vector-valued map.

Finally, for all
$\psi = (\psi^{pq}) \in \prod_{p,q} \mHpq$,
\begin{align*}
&\qquad \Rhoxinnerp{ A\psi}{g}\\
&=\frac{1}{N} \sum_{i=1}^{N}\int_{\bX} \langle [\rhsfo_{\psi}(\bX)]_i, g_i(\bX) \rangle \,d\rho_{\bX}\\
&= \frac{1}{N^2} \Bigg[\sum_{i=1}^{N_1} \sum_{i'=1}^{N_1}\int_{\bX} \psi^{11}(r_{ii'})\langle \br_{ii'}, g_i(\bX)\rangle + \sum_{i=1}^{N_1} \sum_{i'=N_1+1}^{N} \psi^{12}(r_{ii'})\langle \br_{ii'}, g_i(\bX)\rangle \, d\rho_{\bX}\\
&\qquad + \sum_{i=N_1+1}^{N} \sum_{i'=1}^{N_1}\int_{\bX} \psi^{21}(r_{ii'})\langle \br_{ii'}, g_i(\bX)\rangle + \sum_{i=N_1+1}^{N} \sum_{i'=N_1+1}^{N} \psi^{22}(r_{ii'})\langle \br_{ii'}, g_i(\bX)\rangle \, d\rho_{\bX} \Bigg]\\
&=\frac{1}{N^2} \Bigg[\sum_{i=1}^{N_1} \sum_{i'=1}^{N_1}\int_{\bX} \langle \psi^{11}, \mK^{11}_{r_{ii'}}\rangle_{\mH^{11}}\langle \br_{ii'}, g_i(\bX)\rangle + \sum_{i=1}^{N_1} \sum_{i'=N_1+1}^{N}\langle \psi^{12}, \mK^{12}_{r_{ii'}}\rangle_{\mH^{12}}\langle \br_{ii'}, g_i(\bX)\rangle \, d\rho_{\bX}\\
&\qquad + \sum_{i=N_1+1}^{N} \sum_{i'=1}^{N_1}\int_{\bX} \langle \psi^{21}, \mK^{21}_{r_{ii'}}\rangle_{\mH^{21}}\langle \br_{ii'}, g_i(\bX)\rangle + \sum_{i=N_1+1}^{N} \sum_{i'=N_1+1}^{N} \langle \psi^{22}, \mK^{22}_{r_{ii'}}\rangle_{\mH^{22}}\langle \br_{ii'}, g_i(\bX)\rangle \Bigg]\\
&= \langle \psi^{11}, \frac{1}{N^2} \sum_{i=1}^{N_1} \sum_{i'=1}^{N_1} \int_{\bX} \mK^{11}_{r_{ii'}}\langle \br_{ii'}, g_i(\bX)\rangle  \,d\rho_{\bX}\rangle_{\mathcal{H}_{{K}^{11}}} + \langle \psi^{12}, \frac{1}{N^2} \sum_{i=1}^{N_1} \sum_{i'=N_1+1}^{N} \int_{\bX} \mK^{12}_{r_{ii'}}\langle \br_{ii'}, g_i(\bX)\rangle  \,d\rho_{\bX}\rangle_{\mathcal{H}_{{K}^{12}}}\\
&+\langle \psi^{21}, \frac{1}{N^2} \sum_{i=N_1+1}^{N_1} \sum_{i'=1}^{N_1} \int_{\bX} \mK^{21}_{r_{ii'}}\langle \br_{ii'}, g_i(\bX)\rangle  \,d\rho_{\bX}\rangle_{\mathcal{H}_{{K}^{21}}} + \langle \psi^{22}, \frac{1}{N^2} \sum_{i=N_1+1}^{N} \sum_{i'=N_1+1}^{N} \int_{\bX} \mK^{22}_{r_{ii'}}\langle \br_{ii'}, g_i(\bX)\rangle  \,d\rho_{\bX}\rangle_{\mathcal{H}_{{K}^{22}}}\\
&=\langle \psi, A^*g\rangle_{\prod_{p,q} \mHpq},
\end{align*} 
where $\langle \psi_1, \psi_2\rangle_{\prod_{p,q} \mHpq} := \sum_{p,q} \langle \psi_1^{pq}, \psi_2^{pq} \rangle_{\mHpq}$ for all $ \psi_1,\psi_2 \in \prod_{p,q} \mHpq$.
So by uniqueness of the integral, \eqref{adjoint} holds. Equation \eqref{positive} is a consequence of \eqref{adjoint} and the fact that the integral commutes with the scalar product. 

We now prove that $B$ is a trace class operator, i.e. to show that $ \mathrm{Tr}(|B|)<\infty,$ where $|B| =\sqrt{B^*B}$. Let $(e_n)_{n\in \mathbb{N}}$ be a Hilbert basis of $\prod_{p,q} \mHpq$. Since $B$ is positive, we have $|B| =B$. Therefore it is equivalent to show $   \mathrm{Tr}(B) <\infty$. 
\begin{align*}
  \mathrm{Tr}(B)=\mathrm{Tr}(A^*A)&=\sum_n \langle A^*Ae_n,e_n\rangle_{\prod_{p,q} \mHpq}= \sum_n \langle Ae_n,Ae_n\rangle_{L^2(\rho_{\bX})}  \\ &=\sum_n \|\rhsfo_{e_n}(\bX)\|^2_{L^2(\rho_{\bX})}\leq \sum_n \|e_n\|^2_{L^2(\tilde{\rho}_T^L)}\\
&\leq R^2 \sum_n\|e_n\|^2_{L^2(\rhoL)}= R^2 \int \langle \mK_{r},\mK_{r}\rangle_{\prod_{p,q} \mHpq} \,d\rhoL(r) \leq 2\kappa_{max}^2R^2,
\end{align*} 
where $\mK_{r} = (\mK_{r}^{pq})_{p,q}$, $\kappa_{max} = \max_{p,q}(\kappa^{pq})$, $R$ is the upper bound for all $\{r_{ii'}\}$, and we used Lemma \ref{lem1} to show the inequality in the second line and
\begin{align*}
\langle \mK_{r},\mK_{r}\rangle_{\prod_{p,q} \mHpq}&=\langle \sum_n \langle \mK_{r},e_n\rangle_{\mH}e_n, \mK_r \rangle_{\prod_{p,q} \mHpq}\\
&=\langle \sum_n \langle \mK_{r},e_n\rangle_{\prod_{p,q} \mHpq}e_n, \mK_r \rangle_{\prod_{p,q} \mHpq}\\
&=\sum_n e_n^2(r). 
\end{align*}

Lastly, if $\bintkernelvar \in L^2( \tilde\rho_T^{L})$, based on the identity that  $$B\bintkernelvar(r) = ((B\bintkernelvar(r))_{pq})=(\langle \rhsfo_{\bintkernelvar}(\bX),  \rhsfo_{K_{r}^{pq}}(\bX) \rangle_{L^2(\rho_{\bX})})$$
following from the equation \eqref{positive}. We apply the Cauchy-Schwartz inequality, Lemma \ref{infbound} and \ref{lem1} to obtain that 
\begin{align}
|(B\bintkernelvar)_{pq}(r)| &\leq \|\rhsfo_{\bintkernelvar}(\bX)\|_{L^2(\rho_{\bX})} \|\rhsfo_{K_r}^{pq}(\bX)\|_{L^2(\rho_{\bX})} \notag\\
&\leq \sqrt{\frac{2(N-1)}{N}} \|\bintkernelvar\|_{{L^2(\tilde{\rho}_T^L)}}\|K_r^{pq}\|_{{L^2(\tilde{\rho}_T^{pq,L})}}\notag\\
&\leq \sqrt{2} \|\bintkernelvar\|_{{L^2(\tilde{\rho}_T^L)}}R\|K_r^{pq}\|_{\infty}\notag\\
&\leq \sqrt{2} \|\bintkernelvar\|_{{L^2(\tilde{\rho}_T^L)}}\kappa_{pq} R\|K_r^{pq}\|_{\mHpq}\notag\\
&\leq \sqrt{2}  \|\bintkernelvar\|_{{L^2(\tilde{\rho}_T^L)}}\kappa_{pq}^2 R.
\end{align}  where the last inequality follows from $\|K_r\|_{\mHpq}=\sqrt{K^{pq}(r,r)} \leq \kappa_{pq}$. 

As a result, $B\bintkernelvar\in L^2( \tilde\rho_T^{L})$, and $B$ can be viewed as a bounded linear operator from $L^2( \tilde\rho_T^{L})$ to $L^2( \tilde\rho_T^{L})$ with 
\begin{align}
\label{opinequality2}
\|B\|_{L^2( \tilde\rho_T^{L})}\leq 2\kappa_{max}^2 R^2.
\end{align}

\end{proof}

\paragraph{Operator representations for minimizers} When the trajectory data is infinite ($M \rightarrow \infty$), the expected risk functional of $\mathcal{E}^{\lambda,M}(\cdot)$ is 
\begin{align}\label{exp}
\mE^{\lambda,\infty}{(\bintkernelvar)}:&=\mathbb{E}\bigg[\frac{1}{LM}\sum_{l=1,m=1}^{L,M} \| \rhsfo_{\bintkernelvar}(\bX^{(m,l)})-{\bZ}^{m,l}_{\sigma^2}\|^2\bigg]+ \sum_{p,q}\lambda^{pq} \|\intkernelvar^{pq}\|_{\mHpq}^2 
\\&=\Rhoxnorm{A\bintkernelvar- A\intkernel}^2+ \|\sqrt{\lambda} \cdot \bintkernelvar\|_{\prod_{p,q} \mHpq}^2, \qquad \lambda = (\lambda^{pq}),
\end{align} where the expectation is taken with respect to the joint distribution of $\mu_0$ and Gaussian noise $\mathcal{N}(\bm{0}, I_{dN})$ (independent of $\mu_0$).

\begin{proposition} Consider the expected risk $\mE^{\lambda,\infty}(\cdot)$ in \eqref{exp} with a possible regularization term determined by $\lambda\geq 0$. We solve the minimization problem $$\argmin{\bintkernelvar \in \prod_{p,q} \mHpq}\mE^{\lambda,\infty} (\bintkernelvar).$$
\begin{itemize}
\item Case $\lambda=0$. Then its minimizer $\bintkernel_{\prod_{p,q} \mHpq}^{0,\infty} = (\intkernel_{\mHpq}^{0,\infty})$ always exists and  satisfies 
$$B\bintkernel_{\prod_{p,q} \mHpq}^{0,\infty}=A^{*}\rhsfo_{\bintkernel}.$$ 
\item Case $\lambda>0$. Then a unique minimizer exists and it is given by 
$$\bintkernel_{\prod_{p,q} \mHpq}^{\lambda,\infty}: = (\intkernel_{\mHpq}^{pq,\lambda^{pq},\infty}):=(B+\lambda)^{-1}A^{*}\rhsfo_{\bintkernel}.$$
\end{itemize}
\end{proposition}

\begin{corollary}\label{residual}For any $\bintkernelvar \in \prod_{p,q} \mHpq$, we have that $\mE^{0,\infty}(\bintkernelvar)-\mE^{0,\infty}(\bintkernel_{\prod_{p,q} \mHpq}^{0,\infty})=\|A\bintkernelvar-A\bintkernel_{\prod\mHpq}^{0,\infty}\|^2_{L^2(\rho_{\bX})}=\|\sqrt{B}(\bintkernelvar-\bintkernel_{\prod_{p,q} \mHpq}^{0,\infty})\|^2_{\prod_{p,q} \mHpq}$. 
\end{corollary}

\begin{remark}\label{residualpreviouswork} In the context of learning theory, $\mE^{0,\infty}(\bintkernelvar)-\mE^{0,\infty}(\bintkernel_{\prod_{p,q} \mHpq}^{0,\infty})$ is called the residual error \cite{de2005learning}. Assuming the coercivity condition \eqref{coercivity}, then we have $\bintkernel_{\prod_{p,q} \mHpq}^{0,\infty} = \bintkernel = (\intkernel^{pq})$.

\end{remark}

Now we consider the empirical setting.

\begin{proposition}\label{eoperator}  Given the empirical noisy trajectory data $\{\bbX_M,\bbZ_{\sigma^2,M}\}$ as in \eqref{empiricaldata}, we define the sampling operator
$A_{M}: \prod_{p,q} \mHpq \rightarrow \mathbb{R}^{dNML}$ by
\begin{align}\label{finiterank}
A_{M}\bintkernelvar=\rhsfo_{\bintkernelvar}(\bbX_M)&:=\mathrm{Vec}(\{\rhsfo_{\bintkernelvar}(\bX^{(m,l)})\}_{m=1,l=1}^{M,L})\notag\\
&=\mathrm{Vec}\left(\Bigg\{
\begin{bmatrix}
    \rhsfo_{\intkernelvar^{11}}(\bX^{(m,l)}) + \rhsfo_{\intkernelvar^{12}}(\bX^{(m,l)})\\
    \rhsfo_{\intkernelvar^{21}}(\bX^{(m,l)}) + \rhsfo_{\intkernelvar^{22}}(\bX^{(m,l)})
\end{bmatrix}
\Bigg\}_{m=1,l=1}^{M,L}\right),
\end{align} where $\mathbb{R}^{dNML}$ is equipped with the inner product defined in \eqref{winnerp}. 

\begin{itemize}
 \item [1.]The adjoint operator  $A_{M}^*$ is a finite rank operator. For any $\mathbb{W}$ in  $\mathbb{R}^{dNML}$, let $\mathbb{W}_{m,l,i} \in \mathbb{R}^d$ denote the $i$-th component of $(m,l)$th block of $\mathbb{W}$ as in \eqref{empiricaldata}, then we have 
\begin{align}
A^*_{M}\mathbb{W}=& ( \frac{1}{LM}\sum_{l=1,m=1}^{L,M}\sum_{i,i'=1,
}^{N_1}\frac{1}{N^2}K_{r_{ii'}^{(m,l)}}^{11} \langle \br_{ii'}^{(m,l)}, \mathbb{W}_{m,l,i}\rangle, \frac{1}{LM}\sum_{l=1,m=1}^{L,M}\sum_{i=1,
}^{N_1}\sum_{i'=N_1+1,
}^{N}\frac{1}{N^2}K_{r_{ii'}^{(m,l)}}^{12} \langle \br_{ii'}^{(m,l)}, \mathbb{W}_{m,l,i}\rangle,\notag\\
& \frac{1}{LM}\sum_{l=1,m=1}^{L,M}\sum_{i=N_1+1,
}^{N}\sum_{i'=1,
}^{N_1}\frac{1}{N^2}K_{r_{ii'}^{(m,l)}}^{21} \langle \br_{ii'}^{(m,l)}, \mathbb{W}_{m,l,i}\rangle, \frac{1}{LM}\sum_{l=1,m=1}^{L,M}\sum_{i,i'=N_1+1,
}^{N}\frac{1}{N^2}K_{r_{ii'}^{(m,l)}}^{22} \langle \br_{ii'}^{(m,l)}, \mathbb{W}_{m,l,i}\rangle),
\end{align}
For any function $\bintkernelvar \in \prod_{p,q} \mHpq$, we have that 
\begin{align}
& B_{M}\bintkernelvar:=A^*_{M}A_{M}\bintkernelvar=\{(B_{M}\bintkernelvar)_{pq}\}_{p,q=1}^{2}\notag\\
&with\  (B_{M}\bintkernelvar)_{11} = \frac{1}{LM}\sum_{l=1,m=1}^{L,M}\sum_{i,i'=1}^{N_1}\frac{1}{N^3}K_{r_{ii'}^{(m,l)}}^{11} (\sum_{i''=1}^{N_1}\langle \intkernelvar^{11} ,K_{r_{ii''}^{(m,l)}}^{11} \rangle_{\mathcal{H}_{{K}^{11}}} \langle \br_{ii'}^{(m,l)},\br_{ii''}^{(m,l)}\rangle \\
& \qquad \qquad \qquad \qquad \qquad + \sum_{i''=N_1+1}^{N}\langle \intkernelvar^{12} ,K_{r_{ii''}^{(m,l)}}^{12} \rangle_{\mathcal{H}_{{K}^{12}}} \langle \br_{ii'}^{(m,l)},\br_{ii''}^{(m,l)}\rangle),
\end{align}
and similarly for other $(B_{M}\bintkernelvar)_{pq}$ as we defined in \eqref{positive}.

\item[2.]  If $\lambda>0$, a unique minimizer  $\phi_{\prod_{p,q} \mHpq}^{\lambda,M}$ that solves 
$$
\argmin{\bintkernelvar \in \prod_{p,q} \mHpq}\mE^{\lambda,M}(\bintkernelvar)
$$
exists and is given by 
\begin{align}\label{em}
\phi_{\prod\mHpq}^{\lambda,M}=(B_M+\lambda)^{-1}A_{M}^{*}\bbZ_{\sigma^2,M}.
\end{align}
\end{itemize}
\end{proposition}

\begin{proof} Part 1 of Proposition \ref{eoperator} can be derived by using the identity $\langle A_M \bintkernelvar, \mbf{w}\rangle=\langle  \bintkernelvar, A_M^*\mbf{w} \rangle_{\prod_{p,q} \mHpq}$.  Part 2 of Proposition \ref{eoperator} is straightforward by reformulating the empirical functional \eqref{regularizedrisk1} using 
$$\mathcal{E}^{\lambda,M}(\bintkernelvar)=\|A_M\bintkernelvar -\bbZ_{\sigma^2,M}\|^2+ \|\sqrt{\lambda}\cdot\bintkernelvar\|_{\prod_{p,q} \mHpq}^2$$ and solving its normal equation. 
\end{proof}

\begin{theorem}[Representer theorem]\label{representerthm}
If $\lambda>0$, then the minimizer of the regularized empirical risk functional $\mE^{\lambda,M}(\cdot)$ defined in \eqref{regularizedrisk1} has the form
\begin{equation}
  \phi_{\prod_{p,q} \mHpq}^{\lambda,M}= (\sum_{r \in r_{\bbX_M}} \hat c_{r^{pq}} K_{r}^{pq}
  )_{p,q},
\end{equation}
where we consider $r_{\bbX_M} \in \mathbb{R}^{MLN^2}$ as the vector that contains all the pair distances in $\bbX_{M}$, i.e. 
\begin{equation}
r_{\bbX_M} = \begin{bmatrix}r_{11}^{(1,1)},\dots,r_{1N}^{(1,1)},\dots,r_{N1}^{(1,1)},\dots,r_{NN}^{(1,1)}, \dots, r_{11}^{(M,L)},\dots,r_{1N}^{(M,L)},\dots, r_{N1}^{(M,L)},\dots,r_{NN}^{(M,L)}\end{bmatrix}^T,
\end{equation}
and $r_{\bbX_M}^{pq}$ is the vector where we keep the corresponding pairs of distances of $\bbX_{M}$ in $\mathcal{N}_{pq}$ and set all the others to zero.\\
Moreover, we have 
\begin{eqnarray}\label{solution1}
    \hat c_{r^{pq}}= \frac{1}{N}(\br_{\bbX_M}^{pq})^T \cdot  (K_{\rhsfo_\bintkernel}(\bbX_M,\bbX_M) + \lambda^{pq} N ML I)^{-1}\bbZ_{\sigma^2,M},
\end{eqnarray}
where we consider the block-diagonal matrix $\br_{\bbX_M} = \mathrm{diag}(\br_{\bX^{(m,l)}}) \in \mathbb{R}^{MLdN \times MLN^2}$ with $\br_{\bX^{(m,l)}} \in \mathbb{R}^{dN\times N^2}$ defined by
\begin{equation}
\br_{\bX^{(m,l)}} = 
    \begin{bmatrix}
     \br_{11}^{(m,l)}, \dots, \br_{1N}^{(m,l)} & \mbf{0} & \cdots & \mbf{0}\\
     \mbf{0} & \br_{21}^{(m,l)}, \dots, \br_{2N}^{(m,l)} & \cdots & \mbf{0}\\
     \vdots & \vdots & \ddots & \vdots\\
    \mbf{ 0} & \mbf{0} & \cdots & \br_{N1}^{(m,l)}, \dots, \br_{NN}^{(m,l)}
    \end{bmatrix}
    \ ,
\end{equation}
and $\br_{\bbX_M}^{pq} \in \mathbb{R}^{MLdN \times MLN^2} $ is the matrix where we only keep the corresponding pairs of distances in $\mathcal{N}_{pq}$ and set others to zero.
\end{theorem}

\begin{remark}
Note that $c_{r^{pq}}$ is only relevant with $r_{\bbX_M}^{pq}$, however, in order to ease the notation, we consider the consistent basis $r_{\bbX_M}$ for all $\phi^{pq}$, the coefficients in $c_{r^{pq}}$ which correspond to the pairs of distances not in $r_{\bbX_M}^{pq}$ would be zeros.
\end{remark}

\begin{proof} Let  $\mathcal{H}_{K^{pq},M}$ be the subspace of $\mHpq$ spanned by the set of functions $\{K_{r}^{pq}: r\in r_{\bbX_M}^{pq}\}$. By Proposition \ref{eoperator}, we know that $B_M(\prod \mathcal{H}_{K,M}^{pq}) \subset \prod \mathcal{H}_{K,M}^{pq}$. Since $B_M$ is self-adjoint and compact, by spectral theory of self-adjoint compact operators (see \cite{blank2008hilbert}),  $\prod \mathcal{H}_{K,M}^{pq}$ is also an invariant subspace for the operator $(B_M+\lambda I)^{-1}$.  Then by \eqref{em}, there exists vectors
$\hat c_{r^{pq}}$ such that 
\begin{equation}\label{brep}
    \phi_{\mathcal{H}_{K^{pq}}}^{\lambda,M} = \sum_{r \in r_{\bbX_M}} \hat c_{r^{pq}} K_{r}^{pq}.
\end{equation}

Then, multiplying $(B_M+\lambda I)$ on both sides of \eqref{em} and plugging in \eqref{brep}, we can obtain 
\begin{equation}\label{mtxid}
\big((\br_{\bbX_M}^{pq})^T\br_{\bbX_M}^{pq}K^{pq}(r_{\bbX_M}^{pq}, {r_{\bbX_M}^{pq}})+\lambda^{pq} N^3ML I \big)\hat c_{r^{pq}} + (\br_{\bbX_M}^{pq})^T\br_{\bbX_M}^{pq'}K^{pq'}(r_{\bbX_M}^{pq'}, {r_{\bbX_M}^{pq'}})\hat c_{r^{pq'}}=N(\br_{\bbX_M}^{pq})^T \bbZ_{\sigma^2,M}
\end{equation}
using the matrix representation of $(B_M+\lambda I)$ with respect to the spanning sets $\{K_{r}^{pq}: r\in r_{\bbX_M}^{pq}\}$.

Recall that we have $K^{pq}(r_{\bbX_M}^{pq}, {r_{\bbX_M}^{pq}})=(K^{pq}(r_{ij}, r_{i'j'}))_{r_{ij},r_{i'j'} \in r_{\bbX_M}^{pq}}$, 
and $K_{\rhsfo_\bintkernel}(\bbX_M,\bbX_M)=\mathrm{Cov}(\rhsfo_{\bintkernel}(\bbX_M), 
\rhsfo_{\bintkernel}(\bbX_M))$, so using the identity 
\begin{align}\label{id}
\sum_{p,q}
\br_{\bbX_M}^{pq}K^{pq}(r_{\bbX_M}^{pq}, {r_{\bbX_M}^{pq}}) (\br_{\bbX_M}^{pq})^T
=N^2K_{\rhsfo_\bintkernel}(\bbX_M,\bbX_M)
\end{align}
and the fact that the matrices $\big((\br_{\bbX_M}^{pq})^T\br_{\bbX_M}^{pq}K^{pq}(r_{\bbX_M}^{pq}, {r_{\bbX_M}^{pq}})+\lambda^{pq} N^3ML I \big)$
are invertible, one can verify that 
\begin{eqnarray}
    \hat c_r^{pq}&= \frac{1}{N}(\br_{\bbX_M}^{pq})^T \cdot  (K_{\rhsfo_\bintkernel}(\bbX_M,\bbX_M) + \lambda^{pq} N ML I)^{-1}\bbZ_{\sigma^2,M},
\end{eqnarray}
is the solution. 

\end{proof}

Now we are ready to present the proof for Theorem 4 in the main text.

\begin{proof} [Proof of Theorem \ref{maingp} ]
Let $\tilde K^{pq}=\frac{\sigma^2 K^{pq}}{MNL\lambda^{pq}}$.

\begin{itemize}
\item  Since $\intkernel^{pq} \sim \mathcal{GP}(0,\tilde K^{pq})$, the posterior mean in $\eqref{secondorder:pos}$ will then become
\begin{align*}
\bar{\intkernel}_M^{pq}(r^{\ast})&= \tilde K_{\intkernel^{pq},\rhsfo_{\bintkernel}}(r^\ast,\bbX_M)(\tilde K_{\rhsfo_\bintkernel}(\bbX_M,\bbX_M) + \sigma^2I)^{-1}\bbZ_{\sigma^2,M}\\
&=\frac{1}{N}\tilde K_{r_{\bbX_M}^T}^{pq}(r^{\ast})\br_{\bbX_M}^T(\tilde K_{\rhsfo_\bintkernel}(\bbX_M,\bbX_M) + \sigma^2I)^{-1}\bbZ_{\sigma^2,M}\\
&=\frac{1}{N} K_{r_{\bbX_M}^T}^{pq}(r^{\ast})\br_{\bbX_M}^T( K_{\rhsfo_\bintkernel}(\bbX_M,\bbX_M) + NML\lambda^{pq} I)^{-1}\bbZ_{\sigma^2,M}\\
&= K_{\intkernel^{pq},\rhsfo_{\bintkernel}}(r^\ast,\bbX_M)( K_{\rhsfo_\bintkernel}(\bbX_M,\bbX_M) + NML\lambda^{pq} I)^{-1}\bbZ_{\sigma^2,M}\\
&= \sum_{r \in r_{\bbX_M}^{pq}} \hat c_{r^{pq}} K_{r}^{pq}, 
\end{align*} where $\hat c_{r^{pq}}$ is defined in \eqref{solution1} and  we used  the identity 
$ K_{\intkernel^{pq},\rhsfo_{\bintkernel}}(r^\ast,\bbX_M)=\frac{1}{N} K_{r_{\bbX_M}^{pq}}^{pq}(r^{\ast})(\br_{\bbX_M}^{pq})^T$ (also for $\tilde K$) in the proof. 

\item If we replace the true kernels $\intkernel^{pq}$ with $K_{r_*}^{pq}$, and then apply the representer theorem \eqref{representerthm} for the empirical risk functional \eqref{kr}, we have that 
$$K_{r_*}^{pq,\lambda^{pq},M}(\cdot)=K_{\intkernel^{pq},\rhsfo_{\bintkernel}}(\cdot,\bbX_M)(K_{\rhsfo_\bintkernel}(\bbX_M,\bbX_M) + ML\lambda^{pq} NI)^{-1}K_{\rhsfo_{\intkernel^{pq}},\bintkernel}(\bbX_M,r^\ast),$$

Since $\intkernel^{pq} \sim \mathcal{GP}(0,\tilde K^{pq})$, the marginal posterior  variance  in $\eqref{secondorder:var}$ will then become
\small{
\begin{align*}
&\tilde K_{r^\ast}^{pq} (r^\ast) - \tilde K_{\intkernel^{pq},\rhsfo_{\bintkernel}}(r^\ast,\bbX_M)(\tilde K_{\rhsfo_\bintkernel}(\bbX_M,\bbX_M) + \sigma^2I)^{-1}\tilde K_{\rhsfo_{\intkernel^{pq}},\bintkernel}(\bbX_M,r^\ast)\\
&= \frac{\sigma^2}{ML\lambda^{pq} N} \bigg(K_{r^\ast}^{pq}(r^\ast)- K_{\intkernele,\rhsfo_{\bintkernel}}(r^\ast,\bbX_M)(\frac{\sigma^2}{ML\lambda^{pq} N}  K_{\rhsfo_\bintkernel}(\bbX_M,\bbX_M)+\sigma^2 I)^{-1}\frac{\sigma^2}{ML\lambda^{pq} N} K_{\rhsfo_{\bintkernel},\intkernel^{pq}}(\bbX_M,r^\ast)\bigg)\\
&=\frac{\sigma^2}{ML\lambda^{pq} N}[K^{pq}(r^\ast,r^\ast)- K_{r_*}^{pq,\lambda^{pq},M}(r^\ast)]
\end{align*}
}

\end{itemize}
\end{proof}

\section{Finite sample analysis of reconstruction error}\label{appendix:finitesample}

In this subsection, we shall assume that $\intkernel^{pq}\sim \mathcal{GP}(0,\tilde K^{pq})$ with $ \tilde K^{pq}=\frac{\sigma^2 K^{pq}} {MNL\lambda^{pq}}$  $(\lambda^{pq}>0)$ and the coercivity condition \eqref{coercivity} holds.

\paragraph{Analysis of sample errors} We employ the operator representation: 
\begin{align*}
\bintkernel^{\lambda,M}_{\prod_{p,q} \mHpq}&=(B_M+\lambda)^{-1}A_{M}^{*}\bbZ_{\sigma^2,M}\\
&=\underbrace{(B_M+\lambda)^{-1}B_M\bintkernel}_{\tilde\bintkernel_{\prod_{p,q} \mHpq}^{\lambda,M}}+\underbrace{(B_M+\lambda)^{-1}A_{M}^{*}\mathbb{W}_M}_{\text{Noise term}},\\
\bintkernel^{\lambda,\infty}_{\prod_{p,q} \mHpq}&=(B+\lambda)^{-1}B\bintkernel, 
\end{align*} where $\tilde\bintkernel_{\prod_{p,q} \mHpq}^{\lambda,M}$ is the empirical minimizer of $\mE^{\lambda,M}(\cdot)$ for noise-free observations and $\mathbb{W}$ denotes the noise vector. 

We first provide non-asymptotic analysis of the sample error $\|(B_M+\lambda)^{-1}B_M\bintkernelvar-(B+\lambda)^{-1}B\bintkernelvar\|_{\prod_{p,q} \mHpq}$ for any $\bintkernelvar\in \prod_{p,q} \mHpq$ and then apply it to $\bintkernel$. This allows us to obtain a bound on $\|\tilde\bintkernel_{\prod_{p,q} \mHpq}^{\lambda,M}-\bintkernel_{\prod_{p,q} \mHpq}^{\lambda,M}\|_{\prod_{p,q} \mHpq}$.

\begin{lemma}\label{varinequality} For a bounded function $\bintkernelvar = (\intkernelvar^{pq}) \in L^2( \tilde\rho_T^{L})$ and any positive integer $M$, we have that
\begin{align}
\|B_{M}\bintkernelvar\|_{\prod_{p,q} \mHpq} 
&\leq 8 \kappa_{max}  \|\bintkernelvar\|_{\infty} R^2, a.s.,\\
\E\|B_{M}\bintkernelvar\|_{\prod_{p,q} \mHpq}^2&\leq 2\sqrt{2}\|\bintkernelvar\|_{L^2(\tilde\rho_{T}^L)}^2\kappa_{max}^2R^2.
\end{align}
where $\|\bintkernelvar\|_{\infty} = \max(\|\intkernelvar^{pq}\|_{\infty})$, $\kappa_{max} = \max(\kappa^{pq})$.
\end{lemma}
\begin{proof} 
Note that $\big\|K_{r}^{pq}\big\|_{\mHpq} \leq \kappa_{pq}$ for any $r\in [0,R]$, then for $B\bintkernelvar = ((B\bintkernelvar)_{11},(B\bintkernelvar)_{12},(B\bintkernelvar)_{21},(B\bintkernelvar)_{22})$ 
\begin{align*}
&\qquad \|B_{M}\bintkernelvar\|_{\prod_{p,q} \mHpq} \leq  \sum_{p,q} \|(B_{M}\bintkernelvar)_{pq}\|_{\mHpq} \\
&\leq \sum_{p,q} \frac{1}{LM}\sum_{l=1,m=1}^{L,M}\sum_{i=1,i', i'' \neq i}^{N}\frac{1}{N^3}\big\|K_{r_{ii'}^{(m,l)}}^{pq}\big\|_{\mHpq}  (\|\intkernelvar^{pq}\|_{\infty} R^2 + \|\intkernelvar^{pq'}\|_{\infty} R^2) \\
&\leq 8 \kappa_{max}  \|\bintkernelvar\|_{\infty} R^2, a.s.
\end{align*}
For the second inequality, we have that 

\begin{align*}
\E\|B_{M}\bintkernelvar\|_{\prod_{p,q} \mHpq}^2&=\E \langle A^*_{M}A_{M}\bintkernelvar,  A^*_{M}A_{M}\bintkernelvar\rangle_{\prod_{p,q} \mHpq}= \langle A^*A\bintkernelvar, B\bintkernelvar\rangle_{\prod_{p,q} \mHpq}\\
&=\langle A\bintkernelvar,  A(B\bintkernelvar)\rangle_{L^2(\mbf{\rho}_X)} \leq \| A\bintkernelvar\|_{L^2(\mbf{\rho}_X)}\| A(B\bintkernelvar)\|_{L^2(\mbf{\rho}_X)} \\
&\leq  \| \bintkernelvar\|_{L^2(\tilde \rho_T^L)} \| (B\bintkernelvar)\|_{L^2(\tilde \rho_T^L)} 
\leq \|B\|_{L^2(\tilde \rho_T^L)}\| \bintkernelvar\|_{L^2(\tilde \rho_T^L)} ^2 \\
&\leq 2\sqrt{2}\kappa_{max}^2 R^2\| \bintkernelvar\|_{L^2(\tilde \rho_T^L)} ^2,
\end{align*}
where we used the Lemma \ref{lem1} and \eqref{opinequality2}.
\end{proof}

\begin{theorem} \label{decomp}For a  bounded function $\bintkernelvar \in \prod L^2(\tilde \rho_T^{pq,L})$ and $0< \delta <1$, with probability at least $1-\delta$, there holds 
\begin{align}
\|B_M\bintkernelvar-B\bintkernelvar\|_{\prod_{p,q} \mHpq} \leq &  \frac{32\kappa_{max} R^2 \|\bintkernelvar\|_{\infty} \log(2/\delta)}{M} \notag\\
&\quad + 2\sqrt{2}\kappa_{max} R\rhotnorm{\bintkernelvar}
\sqrt{ \frac{\log(2/\delta)}{M}}
\end{align}
\end{theorem}
\begin{proof} Define the $\prod_{p,q} \mHpq$-valued random variable $\xi^{(m)} = (\xi^{(m)}_{11},\xi^{(m)}_{12},\xi^{(m)}_{21},\xi^{(m)}_{22})$ with
\begin{align*}
    \xi^{(m)}_{11}=&(\frac{1}{L}\sum_{l=1}^{L}\sum_{i,i'=1}^{N_1}\frac{1}{N^3} (K_{r_{ii'}^{(m,l)}}^{11} (\sum_{i''=1}^{N_1}\langle \intkernelvar^{11} ,K_{r_{ii''}^{(m,l)}}^{11} \rangle_{\mH^{11}} \langle \br_{ii'}^{(m,l)},\br_{ii''}^{(m,l)}\rangle + \sum_{i''=N_1+1}^{N}
    \langle \intkernelvar^{12} ,K_{r_{ii''}^{(m,l)}}^{12} \rangle_{\mH^{12}} \langle \br_{ii'}^{(m,l)},\br_{ii''}^{(m,l)}\rangle)
\end{align*}
and similarly for other $\xi^{(m)}_{pq}$ as we defined in \eqref{positive}.
 Then the random variables $\{\xi^{(m)}\}_{m=1}^{M}$ are i.i.d. According to Lemma \ref{varinequality}, we have that 
\begin{align*}
\|\xi^{(m)}\|_{\prod_{p,q} \mHpq}&\leq 8\kappa_{max} R^2 \|\bintkernelvar\|_{\infty},\\
\E\|\xi^{(m)}\|_{\prod_{p,q} \mHpq}^2&\leq 2\sqrt{2}\kappa_{max}^2 R^2\rhotnorm{\bintkernelvar}.
\end{align*} Note that $B_M\bintkernelvar-B\bintkernelvar=\frac{1}{M}\sum_{m=1}^M (\xi^{(m)}-\E(\xi^{(m)})).$  The conclusion follows by applying Lemma \ref{coninequality2} to $\{\xi^{(m)}\}_{m=1}^{M}$. 
\end{proof}

\begin{theorem}[Sampling Error]\label{sampleerror}For any bounded function $\bintkernelvar \in \prod L^2(\tilde \rho_T^{pq,L})$, let $0< \delta <1$, with probability at least $1-\delta$, there holds 
\begin{align}
    &\|(B_M+\lambda)^{-1}B_M\bintkernelvar-(B+\lambda)^{-1}B\bintkernelvar\|_{\prod_{p,q} \mHpq} \notag\\
    &\leq  \frac{8\kappa_{max} R^2\|\bintkernelvar\|_{\infty}) \sqrt{2\log(4/\delta)}}{\sqrt{M}\lambda_{min}} (C_{{\prod_{p,q} \mHpq}}+\frac{C_{\kappa,R,\lambda}\sqrt{2\log(4/\delta)}}{\sqrt{M\lambda_{min}}} ),
\end{align}
where $C_{{\prod_{p,q} \mHpq}}=2\sqrt{\frac{2}{c_{min}}}+1$,  $C_{\kappa,R,\lambda}=8\kappa_{max}R + 4\sqrt{\lambda_{min}}$, and $c_{min} = \min(c_{\mHpq})$, $\lambda_{min} = \min(\lambda^{pq})$.
\end{theorem}
\begin{proof} We introduce an intermediate quantity $(B_M+\lambda)^{-1}B\bintkernelvar$ and decompose 
\begin{align*}
&(B_M+\lambda)^{-1}B_M\bintkernelvar-(B+\lambda)^{-1}B\bintkernelvar \\
&=(B_M+\lambda)^{-1}B_M\bintkernelvar-(B_M+\lambda)^{-1}B\bintkernelvar+(B_M+\lambda)^{-1}B\bintkernelvar- (B+\lambda)^{-1}B\bintkernelvar.
\end{align*}
First of all,  since $\|(B_M+\lambda)^{-1}\|_{\prod_{p,q} \mHpq}\leq \frac{2}{\min(\lambda^{pq})}$, we have that 
\begin{align*}
\|(B_M+\lambda)^{-1}B_M\bintkernelvar-(B_M+\lambda)^{-1}B\bintkernelvar\|_{\prod_{p,q} \mHpq}\leq \frac{2}{\lambda_{min}} \|B_M\bintkernelvar-B\bintkernelvar\|_{\prod_{p,q} \mHpq}.
\end{align*}

Applying Theorem \ref{decomp} to $B_M\bintkernelvar-B\bintkernelvar$, we obtain with probability at least $1-\delta/2$
\begin{align*}
\frac{2}{\lambda_{min}}\|B_M\bintkernelvar-B\bintkernelvar\|_{\prod_{p,q} \mHpq} &\leq \frac{64\kappa_{max} R^2 \|\bintkernelvar\|_{\infty} \log(4/\delta)}{\lambda_{min} M}  + 4\sqrt{2}\kappa_{max} R\rhotnorm{\bintkernelvar} \sqrt{ \frac{\log(4/\delta)}{\lambda^2_{min} M}}\\ 
& \leq \frac{64\kappa_{max} R^2 \|\bintkernelvar\|_{\infty}\log(4/\delta)}{\lambda_{min} M} + 4\sqrt{2}\kappa_{max} R^2 \|\bintkernelvar\|_{\infty}\sqrt{ \frac{2\log(4/\delta)}{\lambda^2_{min} M}}
\end{align*}

On the other hand, we have
\begin{align*}
\|(B_M+\lambda)^{-1}B\bintkernelvar- (B+\lambda)^{-1}B\bintkernelvar\|_{\prod_{p,q} \mHpq} &= \|(B_M+\lambda)^{-1}(B-B_M)(B+\lambda)^{-1}B\bintkernelvar\|_{\prod_{p,q} \mHpq}\\ 
&\leq \frac{2}{\lambda_{min}}\|(B-B_M)(B+\lambda)^{-1}B\bintkernelvar\|_{\prod_{p,q} \mHpq}
\end{align*}

Since $\bintkernelvar^{\lambda,\infty}_{\prod_{p,q} \mHpq}=(B+\lambda)^{-1}B\bintkernelvar$ is the unique minimizer of the expected risk functional $ \mE(\psi)=\|A{\psi}-A{\bintkernelvar}\|^2_{L^2(\rho_{\bY})}+\|\sqrt{\lambda}\cdot\psi\|_{\mHe\times\mHa}^2,$ plugging in $\psi=0$, we obtain that 
$$\|A{\bintkernelvar^{\lambda,\infty}_{\prod_{p,q}\mHpq}}-A{\bintkernelvar}\|^2_{L^2(\rho_{\bX})}+\|\sqrt{\lambda}\cdot\bintkernelvar^{\lambda,\infty}_{\prod_{p,q} \mHpq}\|_{\prod_{p,q} \mHpq}^2 <\| A{\bintkernelvar}\|^2_{L^2(\rho_{\bX})},$$ which implies that 
\begin{align}
\label{eq1}\|\bintkernelvar^{\lambda,\infty}_{\prod_{p,q} \mHpq}\|_{\prod_{p,q} \mHpq}&\leq \frac{1}{\sqrt{\lambda_{min}}}\| A{\bintkernelvar}\|_{L^2(\rho_{\bX})},\\
\|A{\bintkernelvar^{\lambda,\infty}_{\prod_{p,q} \mHpq}}\|^2_{L^2(\rho_{\bX})} &\leq 2\|A{\bintkernelvar}\|^2_{L^2(\rho_{\bX})}.
\end{align}

Then by  Lemma \ref{infbound} and \eqref{eq1},
\begin{eqnarray}
\label{eq2}
    \infnorm{\bintkernelvar^{pq,\lambda^{pq},\infty}_{\mHpq}}&\leq \kappa_{pq} \|\bintkernelvar^{\lambda,\infty}_{\prod_{p,q} \mHpq}\|_{\prod_{p,q} \mHpq} \leq \frac{\kappa_{pq}}{\sqrt{\lambda_{min}}}\|A\bintkernelvar\|_{L^2(\rho_{\bX})},\notag\\
\end{eqnarray}

Suppose the coercivity condition \eqref{coercivity} holds true, we have \begin{eqnarray}
\label{eq3}
    \rhotnorm{\bintkernelvar^{pq,\lambda^{pq},\infty}_{\mH^{pq}}}^2&\leq \frac{1}{c_{\mHpq}}\|A{\bintkernelvar^{\lambda,\infty}_{\prod_{p,q} \mHpq}}\|_{L^2(\rho_{\bX})}^2 \leq \frac{2}{c_{\mHpq}}\|A{\bintkernelvar}\|_{L^2(\rho_{\bX})}^2,\notag\\
\end{eqnarray}
and note that $ \Rhoxnorm{A\bintkernelvar}^2< \sum_{pq} 2R^2\|\intkernelvar^{pq}\|_{\infty}^2 <8R^2 \|\bintkernelvar\|_{\infty}^2$(see \eqref{opinequality}), therefore,
applying theorem \ref{decomp} to $\bintkernelvar^{\lambda,\infty}_{\prod_{p,q} \mHpq}=(B+\lambda)^{-1}B\bintkernelvar$, and using \eqref{eq2}, \eqref{eq3} , we obtain that, with probability at least $1-\delta/2$,
\begin{align*}
&\qquad \frac{2}{\lambda_{min}}\|(B-B_M)(B+\lambda)^{-1}B\bintkernelvar\|_{\prod_{p,q}\mHpq} \\
& \leq \frac{64\kappa_{max} R^2\|\bintkernelvar^{\lambda,\infty}_{\prod_{p,q}\mHpq}\|_{\infty} \log(4/\delta)}{\lambda_{min} M}  + 4\sqrt{2}\kappa_{max} R \rhotnorm{\bintkernelvar^{\lambda,\infty}_{\mHe\times\mHa}}
\sqrt{ \frac{\log(4/\delta)}{\lambda_{min}^2 M}}\\
&\leq  \frac{64\kappa_{max}^2 R^3 \|\bintkernelvar\|_{\infty} \log(4/\delta)}{\lambda_{min}^{\frac{3}{2}} M}+ \frac{16\sqrt{2}}{\sqrt{c_{min}}} \kappa_{max} R^2 \|\bintkernelvar\|_{\infty} \sqrt{ \frac{\log(4/\delta)}{\lambda_{min}^2 M}}.
\end{align*}

Finally, by combining two bounds together, we obtain that, with probability at least $1-\delta$
\begin{align*}
&\|(B_M+\lambda)^{-1}B_M\bintkernelvar-(B_M+\lambda)^{-1}B\bintkernelvar\|_{\prod_{p,q} \mHpq} \\ 
&\leq  \frac{8\kappa_{max} R^2\|\bintkernelvar\|_{\infty}\sqrt{2\log(4/\delta)}}{\sqrt{M}\lambda_{min}}\bigg[ (2\sqrt{\frac{2}{c_{min}}}+1) + \frac{(8\kappa_{max} R+4\sqrt{\lambda_{min}})\sqrt{2\log(4/\delta)}}{\sqrt{M\lambda_{min}}} \bigg]\\
&\leq \frac{8\kappa_{max} R^2\|\bintkernelvar\|_{\infty}\sqrt{2\log(4/\delta)}}{\sqrt{M}\lambda_{min}} (C_{{\prod_{p,q} \mHpq}}+\frac{C_{\kappa,R,\lambda}\sqrt{2\log(4/\delta)}}{\sqrt{M\lambda_{min}}} ).
\end{align*} where $C_{{\prod_{p,q} \mHpq}}=2\sqrt{\frac{2}{c_{min}}}+1$ and $C_{\kappa,R,\lambda}=8\kappa_{max} R + 4\sqrt{\lambda_{min}}$.
\end{proof}

\begin{theorem}[$\mH$-bound]\label{hbound} For any $\delta \in (0,1)$, it holds with probability at least $1-\delta$ that 
\begin{eqnarray}
& &\|\bintkernel_{\prod_{p,q} \mHpq}^{\lambda,M}-\bintkernel_{\prod_{p,q} \mHpq}^{\lambda,\infty}\|_{\prod_{p,q} \mHpq} \notag\\
&\lesssim&   \frac{8\kappa_{max} R^2\|\intkernel\|_{\infty}\sqrt{2\log(8/\delta)}}{\sqrt{M}\lambda_{min}}(C_{{\prod_{p,q} \mHpq}}+\frac{C_{\kappa,R,\lambda}\sqrt{2\log(8/\delta)}}{\sqrt{M\lambda_{min}}}) + \frac{8\kappa_{max} R\sigma \log(8/\delta)}{\sqrt{c}\lambda_{min} d \sqrt{MLN}}\notag\\
\end{eqnarray}
where $c$ is an absolute constant appearing in the Hanson-Wright inequality (Theorem \ref{HAnson}),  $\|\bintkernelvar\|_{\infty} = \max(\|\intkernelvar^{pq}\|_{\infty})$, $C_{{\prod_{p,q} \mHpq}}=2\sqrt{\frac{2}{c_{min}}}+1$,  $C_{\kappa,R,\lambda}=8\kappa_{max}R + 4\sqrt{\lambda_{min}}$, and $c_{min} = \min(c_{\mHpq})$, $\lambda_{min} = \min(\lambda^{pq})$, $\kappa_{max} = \max(\kappa^{pq})$.
\end{theorem}

\begin{proof} We decompose $\bintkernel_{\prod_{p,q} \mHpq}^{\lambda,M}-\bintkernel_{\prod_{p,q} \mHpq}^{\lambda,\infty}=\bintkernel_{\prod_{p,q} \mHpq}^{\lambda,M}-\tilde\bintkernel_{\prod_{p,q} \mHpq}^{\lambda,M}+\tilde\bintkernel_{\prod_{p,q} \mHpq}^{\lambda,M}-\bintkernel_{\prod_{p,q} \mHpq}^{\lambda,\infty}$ where $\tilde\bintkernel_{\prod_{p,q} \mHpq}^{\lambda,M}$ is the empirical minimizer for noise-free observations. Then applying Theorem \ref{sampleerror} to the term $\tilde\bintkernel_{\prod_{p,q} \mHpq}^{\lambda,M}-\bintkernel_{\prod_{p,q} \mHpq}^{\lambda,\infty}$, we obtain that with probability at least $1-\delta$,
\begin{align}\label{trian1}
&\qquad \|\tilde\bintkernel_{\prod_{p,q} \mHpq}^{\lambda,M}-\bintkernel_{\prod_{p,q} \mHpq}^{\lambda,\infty}\|_{\prod_{p,q} \mHpq} \notag\\
&\leq  \frac{8\kappa_{max} R^2\|\bintkernelvar\|_{\infty}) \sqrt{2\log(4/\delta)}}{\sqrt{M}\lambda_{min}} (C_{{\prod_{p,q}\mHpq}}+\frac{C_{\kappa,R,\lambda}\sqrt{2\log(4/\delta)}}{\sqrt{M\lambda_{min}}} ),
\end{align}

 We now just need to estimate the ``noise part" $\bintkernel_{\prod_{p,q} \mHpq}^{\lambda,M}-\tilde\bintkernel_{\prod_{p,q} \mHpq}^{\lambda,M}$. According to \eqref{em}, 
\begin{align}\label{em1}
\tilde\bintkernel_{\prod_{p,q} \mHpq}^{\lambda,M}-\bintkernel_{\prod_{p,q} \mHpq}^{\lambda,M}=(B_M+\lambda)^{-1}A_{M}^{*}\mathbb{W}_M
\end{align} 
where the noise vector $\mathbb{W}_M$ follows a multivariate Gaussian distribution with zero mean and variance $\sigma^2I_{dNML}$.  Note that
\begin{align*}
\|\tilde\bintkernel_{\prod_{p,q} \mHpq}^{\lambda,M}-\bintkernel_{\prod_{p,q} \mHpq}^{\lambda,M}\|_{\prod_{p,q} \mHpq}^2 &= \langle \mathbb{W}_M, A_M(B_M+\lambda)^{-2}A_M^*\mathbb{W}_M\rangle\\
&= \sum_{p,q} \mathbb{W}_M^T\Sigma_M^{pq} \mathbb{W}_M,
\end{align*} where the matrix $$\Sigma_M^{pq} = (\mK_{\rhsfo_\bintkernel}(\bbX_M,\bbX_M) + \lambda^{pq} NdML I)^{-1}\mK_{\rhsfo_\bintkernel}(\bbX_M,\bbX_M)  (\mK_{\rhsfo_\intkernel}(\bbX_M,\bbX_M) +\lambda^{pq} dNML I)^{-1},$$

Note that $\sum_{p,q} \Sigma_M^{pq}$ is the matrix form of the operator  $A_M(B_M+\lambda)^{-2}A_M^*$, whose formula is derived from \eqref{em}, \eqref{solution1} and \eqref{id}, and we have 
\begin{align*}
\mathrm{Tr}(\sum_{p,q}\Sigma_M^{pq})&\leq \sum_{p,q}\frac{1}{(\lambda^{pq})^2(MLNd)^2}\mathrm{Tr}(\mK_{\rhsfo_\bintkernel}(\bbX_M,\bbX_M) )\\
&\leq \sum_{p,q}\frac{1}{(\lambda_{min})^2(MLNd)^2}  ( \sum_{m=1,l=1,i=1}^{M,L,N} \frac{1}{N^2}\sum_{k\neq i,k'\neq i}\mK^{pq}(r_{ik}^{(m,l)},r_{ik'}^{(m,l)})(\br_{ik'}^{(m,l)})^T \br_{ik}^{(m,l)}\notag\\
&\leq  \frac{4}{(\lambda_{min})^2d^2MLN}\kappa_{max}^2R^2, a.s.
\end{align*}

\begin{align*}
\mathrm{Tr}((\sum_{p,q}\Sigma_M^{pq})^2)&\leq \frac{16}{\lambda_{min}^4(MLNd)^4}\mathrm{Tr}(\mK_{\rhsfo_\bintkernel}(\bbX_M,\bbX_M)^2 )\\
&= \frac{16}{\lambda_{min}^4(MLNd)^4}  (\sum_{p,q}\sum_{m,m'=1,l,l'=1,i,i'=1}^{M,L,N} \bigg\| \frac{1}{N^2}\sum_{k\neq i,k'\neq i'}\mK^{pq}(r_{ik}^{(m,l)},r_{i'k'}^{(m',l')}) \br_{ik}^{(m,l)}(\br_{i'k'}^{(m',l')})^T\bigg\|_F^2 \\
&\leq  \frac{64\kappa_{max}^4R^4}{\lambda_{min}^4d^4(MLN)^2}, a.s.
\end{align*}

Then applying the Hanson-Wright inequality for the Gaussian random vector $\mathbb{W}_M$ with $S_0=\sigma^2$, since for any $\epsilon>0$, 
 \begin{align*}
 \min \bigg\{ \frac{\epsilon^2}{\sigma^4\|\sum_{p,q}\Sigma_M^{pq}\|_{\mathrm{HS}}^2}, \frac{\epsilon}{\sigma^2\|\sum_{p,q}\Sigma_M^{pq}\|}\bigg\} &\geq \min \bigg\{ \frac{\epsilon^2}{\sigma^4\mathrm{Tr}((\sum_{p,q}\Sigma_M^{pq})^2)}, \frac{\epsilon}{\sigma^2 \mathrm{Tr}(\sum_{p,q}\Sigma_M^{pq})}\bigg\},
\end{align*} 
we obtain that,  with probability at least $1-e^{-t^2}$,  
 \begin{align*}
\mathbb{W}_M^T (\sum_{pq}\Sigma_M^{pq})  \mathbb{W}_M &\leq \frac{1}{c}\sigma^2\max\{\mathrm{Tr}(\sum_{p,q}\Sigma_M^{pq}),\sqrt{\mathrm{Tr}((\sum_{p,q}\Sigma_M^{pq})^2)}\}(1+2t+t^2)\\
 &\leq  \frac{8\kappa_{max}^2R^2\sigma^2}{c\lambda_{min}^2 d^2 {MLN} }(1+2t+t^2)
 \end{align*} for any $t>0$, where $c$ is an absolute positive constant appearing in Hanson-Wright inequality.  Therefore, with probability at least $1-\delta$, there holds 
 \begin{align}\label{trian2}
 \|\tilde\bintkernel_{\prod_{p,q} \mHpq}^{\lambda,M}-\bintkernel_{\prod_{p,q} \mHpq}^{\lambda,M}\|_{\prod_{p,q} \mHpq} \leq  \frac{4\kappa_{max} R\sigma(\log(1/\delta)+1)}{\sqrt{c}\lambda_{min} d \sqrt{MLN}}<\frac{8\kappa_{max} R\sigma \log(4/\delta)}{\sqrt{c}\lambda_{min} d \sqrt{MLN}}
 \end{align} 
 
 Now combining \eqref{trian1} and \eqref{trian2}, we obtain that with probability at least $1-\delta$,
 \begin{align}
&\qquad \|\bintkernel_{\prod_{p,q} \mHpq}^{\lambda,M}-\bintkernel_{\prod_{p,q} \mHpq}^{\lambda,\infty}\|_{\prod_{p,q} \mHpq} \notag\\
&\lesssim  \frac{8\kappa_{max} R^2\|\intkernel\|_{\infty}\sqrt{2\log(8/\delta)}}{\sqrt{M}\lambda_{min}}(C_{{\prod_{p,q} \mHpq}}+\frac{C_{\kappa,R,\lambda}\sqrt{2\log(8/\delta)}}{\sqrt{M\lambda_{min}}}) + \frac{8\kappa_{max} R\sigma \log(8/\delta)}{\sqrt{c}\lambda_{min} d \sqrt{MLN}}
 \end{align}
 
\end{proof}

\paragraph{ Analysis of approximation error $\|\bintkernel_{\prod_{p,q} \mHpq}^{\lambda,\infty}-\bintkernel\|_{\prod_{p,q} \mHpq}$} To get a convergence rate for the reconstruction error $\|\bintkernel_{\prod_{p,q} \mHpq}^{\lambda,M}-\bintkernel\|_{\prod_{p,q} \mHpq}$, we need to get an estimation of the approximation error $\| \bintkernel_{\prod_{p,q} \mHpq}^{\lambda,\infty}-\bintkernel\|_{\prod_{p,q} \mHpq}$. Assume the coercivity condition, then $B \in \mathcal{B}(\prod_{p,q} \mHpq)$ is a strictly positive operator.   Let $B=\sum_{n=1}^{N}\lambda_n\langle \cdot, e_n\rangle e_n$ (possibly $N=\infty$) be the spectral decomposition of $B$ with $0<\lambda_{n+1}<\lambda_{n}$ and $\{e_n\}_{n=1}^{N}$ be an orthonormal basis of $\prod_{p,q} \mHpq$. 
 Then
\begin{eqnarray}
\label{eq: approx error}
\| \bintkernel_{\prod_{p,q} \mHpq}^{\lambda,\infty}-\bintkernel\|_{\prod_{p,q} \mHpq}^2&=\|(B+\lambda)^{-1}B\bintkernel-\bintkernel\|_{\prod_{p,q} \mHpq}^2 =\|\lambda (B+\lambda)^{-1}\bintkernel\|_{\prod_{p,q} \mHpq}^2\notag\\
&=\sum_{n=1}^{N}(\frac{\lambda}{\lambda_n+\lambda})^2|\langle \bintkernel, e_n\rangle_{\prod_{p,q} \mHpq}|^2. 
\end{eqnarray}
Assume now that $\bintkernel \in \mathrm{Im}\, B^{\gamma}$ with $0< \gamma \leq \frac{1}{2}$. Since the function $x^\gamma$ is concave on $[0,\infty]
$, therefore $\frac{\lambda}{\lambda_n+\lambda}\leq \frac{\lambda^\gamma}{\lambda_n^\gamma}$. 
Then we have $\| \bintkernel_{\prod_{p,q} \mHpq}^{\lambda,\infty}-\bintkernel\|_{\prod_{p,q} \mHpq} \leq \lambda^{\gamma}\|B^{-\gamma}\bintkernel\|_{\prod_{p,q} \mHpq}$ where $B^{-\gamma}\bintkernel$ represents the pre-image of $\bintkernel$. 

\begin{proof} Without loss of generality, let $\lambda=M^{-\frac{1}{2\gamma+1}}$.  By Theorem \ref{hbound} and approximation error \eqref{eq: approx error}, with probability at least $1-\delta$,
\begin{align*}
 &\|\bintkernel_{\prod_{p,q}\mHpq}^{\lambda,M}-\bintkernel\|_{\prod_{p,q}\mHpq} \leq  \|\bintkernel_{\prod_{p,q}\mHpq}^{\lambda,M}- \bintkernel_{\prod_{p,q}\mHpq}^{\lambda,\infty}\|_{\prod_{p,q}\mHpq}+ \|\bintkernel_{\prod_{p,q}\mHpq}^{\lambda,\infty}-\bintkernel\|_{\prod_{p,q}\mHpq}\\
 &\leq \frac{8\kappa_{max} R^2\|\intkernel\|_{\infty}\sqrt{2\log(4/\delta)}}{\sqrt{M}\lambda_{min}}(C_{{\prod_{p,q}\mHpq}}+\frac{C_{\kappa,R,\lambda}\sqrt{2\log(4/\delta)}}{\sqrt{M\lambda_{min}}}) \\
 &\qquad + \frac{8\kappa_{max} R\sigma \log(4/\delta)}{\sqrt{c}\lambda_{min} d \sqrt{MLN}} + \lambda^{\gamma}\|B^{-\gamma}\bintkernel\|_{\prod_{p,q}\mHpq}\\
 &\lesssim C\log(\frac{4}{\delta}) M^{-\frac{\gamma}{2\gamma+1}},
 \end{align*}
 where $C=\max\{\frac{\kappa_{max} R^2 \|\bintkernel\|_{\infty}}{\sqrt{c_{min}}}, \frac{\kappa_{max} R\sigma}{\sqrt{cLN}d},\|B^{-\gamma}\bintkernel \|_{\prod_{p,q}\mHpq}\}$, and   the symbol $\lesssim$ means that the inequality holds up to a multiplicative constant that is an independent absolute constant from the listed parameters.
 
\end{proof}

\noindent
As previously mentioned, we can also apply the same framework to the reconstruction errors $\|K_{r_*}^{pq,\lambda^{pq},M}- K_{r_*}^{pq}\|_{\mHpq}$, and provide an upper bound on worst case $L^\infty$ error for the marginal posterior variances, providing direct insight into uncertainty quantification.

\begin{theorem}\label{marginalpos}[Worst-case $L^\infty$ error analysis for marginal  posterior variances \eqref{secondorder:var}] For any $\delta \in (0,1)$, it holds with probability at least $1-\delta$ that 
\begin{align*}
&|\mathrm{Var}(\bar\intkernel^{pq}(r_*)|\mathbb{X}_{M})|\\
&\leq \frac{\kappa_{max} \sigma^2}{ML\lambda N} \bigg(\sqrt{2}\kappa_{max}+ \frac{8\kappa_{max} R^2\|K_{r^\ast} \|_{\infty}) \sqrt{2\log(4/\delta)}}{\sqrt{M}\lambda_{min}} (C_{{\prod_{p,q}\mHpq}}+\frac{C_{\kappa,R,\lambda}\sqrt{2\log(4/\delta)}}{\sqrt{M\lambda_{min}}} ) \bigg),
\end{align*} 
where $\|K_{r^\ast}\|_{\infty} = \max(\|K_{r^\ast}^{pq}\|_{\infty})$, $C_{{\prod_{p,q} \mHpq}}=2\sqrt{\frac{2}{c_{min}}}+1$,  $C_{\kappa,R,\lambda}=8\kappa_{max}R + 4\sqrt{\lambda_{min}}$, and $c_{min} = \min(c_{\mHpq})$, $\lambda_{min} = \min(\lambda^{pq})$, $\kappa_{max} = \max(\kappa^{pq})$.
\end{theorem}
\begin{proof} Note that for $K_{r^\ast} = (K_{r^\ast}^{pq})$, $K_{r^\ast}^{\lambda,M} = (K_{r^\ast}^{pq,\lambda^{pq},M})$, we have  $K_{r^\ast}^{\lambda,M}=(B_M+\lambda)^{-1}B_MK_{r^\ast} $. Then 
\begin{align*}
K_{r^\ast}^{\lambda,M}-K_{r^\ast}&=(B_M+\lambda)^{-1}B_MK_{r^\ast} -(B+\lambda)^{-1}BK_{r^\ast}+(B+\lambda)^{-1}BK_{r^\ast}-K_{r^\ast} \\
&=(B_M+\lambda)^{-1}B_MK_{r^\ast} -(B+\lambda)^{-1}BK_{r^\ast}+\lambda (B+\lambda)^{-1}K_{r^\ast}.
\end{align*}
Applying Theorem \ref{sampleerror} to $K_{r^\ast}$, we know that,  for any $0< \delta <1$, with probability at least $1-\delta$, there holds 
\begin{align*}
    &\|(B_M+\lambda)^{-1}B_MK_{r^\ast} -(B+\lambda)^{-1}BK_{r^\ast} \|_{\prod_{p,q}\mHpq} \\
    &\leq  \frac{8\kappa_{max} R^2\|K_{r^\ast} \|_{\infty}) \sqrt{2\log(4/\delta)}}{\sqrt{M}\lambda_{min}} (C_{{\prod_{p,q}\mHpq}}+\frac{C_{\kappa,R,\lambda}\sqrt{2\log(4/\delta)}}{\sqrt{M\lambda_{min}}} )
\end{align*}

On the other hand,
$$\|\lambda (B+\lambda)^{-1}K_{r^\ast}\|_{\prod_{p,q} \mHpq} \leq \|K_{r^\ast}\|_{\prod_{p,q} \mHpq}.$$

Therefore,  for any $0< \delta <1$, with probability at least $1-\delta$, 
\begin{align*}
&|\mathrm{Var}(\bar\intkernel^{pq}(r_{*})|\mathbb{X}_{M})|\\
& \leq \frac{\sigma^2}{ML\lambda N}\|K_{r^\ast}^{pq,\lambda^{pq},M}-K_{r^\ast}^{pq}\|_{\infty} \\
&\leq \frac{\kappa_{max} \sigma^2}{ML\lambda N}  \|K_{r^\ast}^{\lambda,M}-K_{r^\ast}\|_{\prod_{p,q} \mHpq} \\
&\leq   \frac{\kappa_{max} \sigma^2}{ML\lambda N} \bigg(\sqrt{2}\kappa_{max}+ \frac{8\kappa_{max} R^2\|K_{r^\ast} \|_{\infty}) \sqrt{2\log(4/\delta)}}{\sqrt{M}\lambda_{min}} (C_{{\prod_{p,q} \mHpq}}+\frac{C_{\kappa,R,\lambda}\sqrt{2\log(4/\delta)}}{\sqrt{M\lambda_{min}}} ) \bigg).
\end{align*} The conclusion follows.
\end{proof}

Suppose we choose $\lambda =O(M^{-\gamma})$ where $\gamma <\frac{1}{4}$, then Theorem \ref{marginalpos} suggests that we can obtain a parametric decay rate of $\|\mathrm{Var}(\bar\intkernel^{pq}(\cdot)|\mathbb{X}_{M})\|_{\infty}$, which is unlikely to be further improved.

\section{Auxiliary lemmas and theorems}

\begin{lemma}\label{lemma: conditioning Gaussian}
Let $\bx$ and $\by$ be jointly Gaussian random vectors
 \begin{equation}
 \begin{bmatrix}
 \bx\\ \by
 \end{bmatrix}
 \sim \mathcal{N} (
 \begin{bmatrix}
 \mu_{\bx}\\ \mu_{\by}
 \end{bmatrix}
 , 
 \begin{bmatrix}
 A & C\\
 C^T & B
 \end{bmatrix}
 ),
\end{equation}
then the marginal distribution of $\bx$ and the conditional distribution of $\bx$ given $\by$ are
\begin{equation}
    \bx \sim \mathcal{N}(\mu_{\bx},A), \quad \textrm{and } \bx|\by \sim \mathcal{N}(\mu_{\bx} + CB^{-1}(\by - \mu_{\by}), A - CB^{-1}C^T).
\end{equation}
\end{lemma}
\begin{proof}
See, e.g. \cite{williams2006gaussian}, Appendix A.
\end{proof}

\begin{lemma}[Lemma 8 in \cite{de2005learning}]\label{coninequality2}
Let $\mathcal{H}$ be a Hilbert space and $\xi$ be a random variable  on $(Z,\rho)$ with values in $\mathcal{H}$. Suppose that, $\|\xi\|_{\mathcal{H}}\leq S < \infty$ almost surely.  Let $z_m$ be i.i.d drawn from $\rho$. For any $0<\delta<1$, with confidence $1-\delta$,
$$\bigg\| \frac{1}{M}\sum_{m=1}^{M}(\xi(z_m)-\E(\xi))\bigg\|  \leq \frac{4S\log(2/\delta)}{M}+\sqrt{\frac{2\E(\|\xi\|_{H}^2)\log(2/\delta)}{M}}.$$
\end{lemma}

The original version of Lemma \ref{coninequality2} is presented in \cite{yurinsky1995sums}.

\begin{theorem}[Hanson-Wright inequality \cite{rudelson2013hanson}] Let $X=(X_1,\cdots,X_n) \in \mathbb{R}^n$ be a random vector with independent components $X_i$ which satisfy 
$\E X_i=0$ and $\|X_i\|_{\psi_2} \leq S_0$, where $\|\cdot\|_{\psi_2}$ is the subGaussian norm. Let $A$ be an $n \times n$ matrix and $\|A\|_{HS}$ denotes the Hilbert-Schmidt norm. Then, for every $\epsilon \geq 0$

$$\mathbb{P}\bigg\{\bigg\| X^TAX-\E X^TAX \bigg\| \geq \epsilon \bigg\} \leq 2\exp \bigg\{ -c \min \bigg\{ \frac{\epsilon^2}{S_0^4\|A\|_{HS}^2}, \frac{\epsilon}{S_0^2\|A\|}\bigg\} \bigg\},$$ where $c$ is an absolute positive constant. 
\label{HAnson}
\end{theorem}

\section{Additional experimental results}

We present full tables of results for experiments in Section \ref{sec:numerics} in Tables \ref{tab:kernel_errors_repulsive} and \ref{tab:trajectory_errors_repulsive}. We first show the full results of Experiment \ref{sec: repulsive} where noise level $\sigma$ is varied to examine the convergence behavior as the noise level decreases. The results for $\sigma = 0$ are also shown to calibrate accuracy in the no-noise scenario.

\begin{table}[htbp]
\centering
\caption{Kernel learning errors for the repulsive interaction potentials with $N_1 = N_2 = 10$, $L=10$, $M=10$, and varied noise. Each group of four rows corresponds to a different noise level $\sigma$. The trend can clearly be seen; as noise decreases, so too do all kernel prediction errors, leading to more accurate performance.}
\label{tab:kernel_errors_repulsive}
\begin{tabular}{c|c|c|c}
\toprule
\textbf{Parameters} & \textbf{Kernel} & \textbf{$L^\infty([0,R])$ Error} & \textbf{$L^2(\tilde\rho_T^{pq,L})$ Error} \\
\midrule
\multirow{4}{*}{$\sigma=0$} & $\phi^{11}$ & $2.38 \cdot 10^{-3} \pm 2.61 \cdot 10^{-3}$ & $2.74 \cdot 10^{-3} \pm 1.61 \cdot 10^{-3}$ \\
& $\phi^{12}$ & $2.35 \cdot 10^{-3} \pm 1.96 \cdot 10^{-3}$ & $2.05 \cdot 10^{-3} \pm 9.02 \cdot 10^{-4}$ \\
& $\phi^{21}$ & $2.34 \cdot 10^{-3} \pm 1.96 \cdot 10^{-3}$ & $2.03 \cdot 10^{-3} \pm 9.15 \cdot 10^{-4}$ \\
& $\phi^{22}$ & $1.97 \cdot 10^{-3} \pm 1.10 \cdot 10^{-3}$ & $5.02 \cdot 10^{-3} \pm 2.42 \cdot 10^{-3}$ \\
\midrule
\multirow{4}{*}{$\sigma=0.0001$} & $\phi^{11}$ & $2.35 \cdot 10^{-2} \pm 3.78 \cdot 10^{-3}$ & $2.94 \cdot 10^{-3} \pm 1.67 \cdot 10^{-3}$ \\
& $\phi^{12}$ & $2.48 \cdot 10^{-2} \pm 3.01 \cdot 10^{-3}$ & $2.31 \cdot 10^{-3} \pm 9.92 \cdot 10^{-4}$ \\
& $\phi^{21}$ & $2.41 \cdot 10^{-2} \pm 2.23 \cdot 10^{-3}$ & $2.36 \cdot 10^{-3} \pm 1.06 \cdot 10^{-3}$ \\
& $\phi^{22}$ & $2.79 \cdot 10^{-2} \pm 7.97 \cdot 10^{-3}$ & $5.40 \cdot 10^{-3} \pm 2.58 \cdot 10^{-3}$ \\
\midrule
\multirow{4}{*}{$\sigma=0.0005$} & $\phi^{11}$ & $3.98 \cdot 10^{-2} \pm 4.77 \cdot 10^{-3}$ & $3.25 \cdot 10^{-3} \pm 1.73 \cdot 10^{-3}$ \\
& $\phi^{12}$ & $3.85 \cdot 10^{-2} \pm 5.81 \cdot 10^{-3}$ & $2.68 \cdot 10^{-3} \pm 1.03 \cdot 10^{-3}$ \\
& $\phi^{21}$ & $3.84 \cdot 10^{-2} \pm 3.72 \cdot 10^{-3}$ & $2.83 \cdot 10^{-3} \pm 1.17 \cdot 10^{-3}$ \\
& $\phi^{22}$ & $4.54 \cdot 10^{-2} \pm 1.04 \cdot 10^{-2}$ & $5.82 \cdot 10^{-3} \pm 2.64 \cdot 10^{-3}$ \\
\midrule
\multirow{4}{*}{$\sigma=0.001$} & $\phi^{11}$ & $5.06 \cdot 10^{-2} \pm 5.79 \cdot 10^{-3}$ & $3.54 \cdot 10^{-3} \pm 1.73 \cdot 10^{-3}$ \\
& $\phi^{12}$ & $4.67 \cdot 10^{-2} \pm 7.55 \cdot 10^{-3}$ & $3.09 \cdot 10^{-3} \pm 1.05 \cdot 10^{-3}$ \\
& $\phi^{21}$ & $4.76 \cdot 10^{-2} \pm 5.67 \cdot 10^{-3}$ & $3.35 \cdot 10^{-3} \pm 1.23 \cdot 10^{-3}$ \\
& $\phi^{22}$ & $5.68 \cdot 10^{-2} \pm 1.19 \cdot 10^{-2}$ & $6.15 \cdot 10^{-3} \pm 2.61 \cdot 10^{-3}$ \\
\midrule
\multirow{4}{*}{$\sigma=0.005$} & $\phi^{11}$ & $8.76 \cdot 10^{-2} \pm 1.19 \cdot 10^{-2}$ & $5.65 \cdot 10^{-3} \pm 1.41 \cdot 10^{-3}$ \\
& $\phi^{12}$ & $7.50 \cdot 10^{-2} \pm 1.40 \cdot 10^{-2}$ & $6.72 \cdot 10^{-3} \pm 1.42 \cdot 10^{-3}$ \\
& $\phi^{21}$ & $7.92 \cdot 10^{-2} \pm 2.16 \cdot 10^{-2}$ & $7.51 \cdot 10^{-3} \pm 1.94 \cdot 10^{-3}$ \\
& $\phi^{22}$ & $9.77 \cdot 10^{-2} \pm 2.06 \cdot 10^{-2}$ & $8.20 \cdot 10^{-3} \pm 2.28 \cdot 10^{-3}$ \\
\midrule
\multirow{4}{*}{$\sigma=0.010$} & $\phi^{11}$ & $1.09 \cdot 10^{-1} \pm 1.82 \cdot 10^{-2}$ & $8.26 \cdot 10^{-3} \pm 1.12 \cdot 10^{-3}$ \\
& $\phi^{12}$ & $9.25 \cdot 10^{-2} \pm 2.03 \cdot 10^{-2}$ & $1.14 \cdot 10^{-2} \pm 2.14 \cdot 10^{-3}$ \\
& $\phi^{21}$ & $9.61 \cdot 10^{-2} \pm 3.69 \cdot 10^{-2}$ & $1.25 \cdot 10^{-2} \pm 2.76 \cdot 10^{-3}$ \\
& $\phi^{22}$ & $1.23 \cdot 10^{-1} \pm 3.07 \cdot 10^{-2}$ & $1.06 \cdot 10^{-2} \pm 2.15 \cdot 10^{-3}$ \\
\midrule
\multirow{4}{*}{$\sigma=0.050$} & $\phi^{11}$ & $1.66 \cdot 10^{-1} \pm 5.97 \cdot 10^{-2}$ & $2.80 \cdot 10^{-2} \pm 4.44 \cdot 10^{-3}$ \\
& $\phi^{12}$ & $1.52 \cdot 10^{-1} \pm 6.18 \cdot 10^{-2}$ & $4.40 \cdot 10^{-2} \pm 9.29 \cdot 10^{-3}$ \\
& $\phi^{21}$ & $1.57 \cdot 10^{-1} \pm 8.80 \cdot 10^{-2}$ & $4.83 \cdot 10^{-2} \pm 7.48 \cdot 10^{-3}$ \\
& $\phi^{22}$ & $2.08 \cdot 10^{-1} \pm 9.00 \cdot 10^{-2}$ & $2.91 \cdot 10^{-2} \pm 6.72 \cdot 10^{-3}$ \\
\midrule
\multirow{4}{*}{$\sigma=0.100$} & $\phi^{11}$ & $1.96 \cdot 10^{-1} \pm 9.98 \cdot 10^{-2}$ & $5.07 \cdot 10^{-2} \pm 1.03 \cdot 10^{-2}$ \\
& $\phi^{12}$ & $2.02 \cdot 10^{-1} \pm 8.22 \cdot 10^{-2}$ & $7.93 \cdot 10^{-2} \pm 1.95 \cdot 10^{-2}$ \\
& $\phi^{21}$ & $2.13 \cdot 10^{-1} \pm 1.16 \cdot 10^{-1}$ & $8.89 \cdot 10^{-2} \pm 1.32 \cdot 10^{-2}$ \\
& $\phi^{22}$ & $2.59 \cdot 10^{-1} \pm 1.44 \cdot 10^{-1}$ & $5.05 \cdot 10^{-2} \pm 1.47 \cdot 10^{-2}$ \\
\bottomrule
\end{tabular}
\end{table}

\begin{table}[htbp]
\centering
\caption{Trajectory prediction errors for repulsive interaction potentials with $N_1 = N_2 = 10$, $L=10$, $M=10$, and varied noise. For each $\sigma$ value, the top row reports error in the time interval $[0,5]$, with training data error on the left and testing data error on the right. The bottom row reports error in the time interval $[5, 10]$, which measures temporal generalization error for both settings. Note that due to the steady-state achieved by the system, the temporal generalization error is in some cases slightly smaller than the error in the transient state portion which occurs mostly within $[0,5]$.}
\label{tab:trajectory_errors_repulsive}
\begin{tabular}{c|c|c}
\toprule
\multirow{2}{*}{\textbf{Parameters}} & \multicolumn{2}{c}{\textbf{Relative Trajectory Error}} \\
\cmidrule(lr){2-3}
& \textbf{Training Data} & \textbf{Test Data} \\
\midrule
$\sigma=0$ & $9.15 \cdot 10^{-4} \pm 1.89 \cdot 10^{-4}$ & $9.14 \cdot 10^{-4} \pm 2.97 \cdot 10^{-4}$ \\
 & $6.25 \cdot 10^{-4} \pm 1.78 \cdot 10^{-4}$ & $4.60 \cdot 10^{-4} \pm 2.06 \cdot 10^{-6}$ \\
\midrule
$\sigma=0.0001$ & $9.22 \cdot 10^{-4} \pm 1.86 \cdot 10^{-4}$ & $9.78 \cdot 10^{-4} \pm 3.58 \cdot 10^{-4}$ \\
 & $6.62 \cdot 10^{-4} \pm 1.68 \cdot 10^{-4}$ & $4.77 \cdot 10^{-4} \pm 1.76 \cdot 10^{-5}$ \\
\midrule
$\sigma=0.0005$ & $1.02 \cdot 10^{-3} \pm 1.59 \cdot 10^{-4}$ & $1.15 \cdot 10^{-3} \pm 3.03 \cdot 10^{-4}$ \\
 & $1.29 \cdot 10^{-3} \pm 4.86 \cdot 10^{-4}$ & $7.71 \cdot 10^{-4} \pm 1.17 \cdot 10^{-4}$ \\
\midrule
$\sigma=0.001$ & $1.19 \cdot 10^{-3} \pm 2.26 \cdot 10^{-4}$ & $1.43 \cdot 10^{-3} \pm 1.79 \cdot 10^{-4}$ \\
 & $2.00 \cdot 10^{-3} \pm 7.43 \cdot 10^{-4}$ & $1.36 \cdot 10^{-3} \pm 2.68 \cdot 10^{-4}$ \\
\midrule
$\sigma=0.005$ & $2.98 \cdot 10^{-3} \pm 9.10 \cdot 10^{-4}$ & $3.51 \cdot 10^{-3} \pm 7.74 \cdot 10^{-4}$ \\
 & $5.91 \cdot 10^{-3} \pm 1.54 \cdot 10^{-3}$ & $5.10 \cdot 10^{-3} \pm 1.37 \cdot 10^{-3}$ \\
\midrule
$\sigma=0.010$ & $5.04 \cdot 10^{-3} \pm 1.74 \cdot 10^{-3}$ & $5.57 \cdot 10^{-3} \pm 1.28 \cdot 10^{-3}$ \\
 & $9.51 \cdot 10^{-3} \pm 2.04 \cdot 10^{-3}$ & $8.88 \cdot 10^{-3} \pm 2.69 \cdot 10^{-3}$ \\
\midrule
$\sigma=0.050$ & $1.98 \cdot 10^{-2} \pm 6.48 \cdot 10^{-3}$ & $1.86 \cdot 10^{-2} \pm 4.70 \cdot 10^{-3}$ \\
 & $3.46 \cdot 10^{-2} \pm 9.72 \cdot 10^{-3}$ & $3.07 \cdot 10^{-2} \pm 1.03 \cdot 10^{-2}$ \\
\midrule
$\sigma=0.100$ & $3.61 \cdot 10^{-2} \pm 9.08 \cdot 10^{-3}$ & $3.35 \cdot 10^{-2} \pm 9.48 \cdot 10^{-3}$ \\
 & $6.31 \cdot 10^{-2} \pm 1.77 \cdot 10^{-2}$ & $5.45 \cdot 10^{-2} \pm 2.21 \cdot 10^{-2}$ \\
\bottomrule
\end{tabular}
\end{table}

We also present full tables of results for Experiment \ref{sec: repulsiveK} in Tables \ref{tab:kernel_errors_repulsiveK} and \ref{tab:trajectory_errors_repulsiveK} where we examine the convergence behavior as $M$ increases and thus more data is used for training.

\begin{table}[h!]
\centering
\caption{Kernel learning errors for linear-repulsive interaction potentials with $N_1 = N_2 = 5, L = 2$, and $\sigma = 0.05$. For each $M$ value, both relative errors are reported. Note that all kernel prediction grows more accurate as the amount of data increases.}
\label{tab:kernel_errors_repulsiveK}
\begin{tabular}{c|c|c|c}
\toprule
\textbf{Parameters} & \textbf{Kernel} & \textbf{$L^\infty([0,R])$ Error} & \textbf{$L^2(\tilde\rho_T^{pq,L})$ Error} \\
\midrule
\multirow{4}{*}{$M=1$} & $\phi^{11}$ & $2.30 \cdot 10^{-1} \pm 8.68 \cdot 10^{-2}$ & $1.37 \cdot 10^{-1} \pm 5.21 \cdot 10^{-2}$ \\
& $\phi^{12}$ & $6.02 \cdot 10^{-2} \pm 2.47 \cdot 10^{-2}$ & $8.37 \cdot 10^{-2} \pm 4.50 \cdot 10^{-2}$ \\
& $\phi^{21}$ & $8.41 \cdot 10^{-2} \pm 3.93 \cdot 10^{-2}$ & $9.64 \cdot 10^{-2} \pm 5.00 \cdot 10^{-2}$ \\
& $\phi^{22}$ & $2.39 \cdot 10^{-1} \pm 1.37 \cdot 10^{-1}$ & $4.03 \cdot 10^{-1} \pm 9.44 \cdot 10^{-2}$ \\
\midrule
\multirow{4}{*}{$M=10$} & $\phi^{11}$ & $1.94 \cdot 10^{-1} \pm 5.48 \cdot 10^{-2}$ & $4.47 \cdot 10^{-2} \pm 8.68 \cdot 10^{-3}$ \\
& $\phi^{12}$ & $3.58 \cdot 10^{-2} \pm 1.36 \cdot 10^{-2}$ & $2.00 \cdot 10^{-2} \pm 6.91 \cdot 10^{-3}$ \\
& $\phi^{21}$ & $3.54 \cdot 10^{-2} \pm 3.02 \cdot 10^{-2}$ & $1.85 \cdot 10^{-2} \pm 8.31 \cdot 10^{-3}$ \\
& $\phi^{22}$ & $2.27 \cdot 10^{-1} \pm 9.70 \cdot 10^{-2}$ & $1.72 \cdot 10^{-1} \pm 6.92 \cdot 10^{-2}$ \\
\midrule
\multirow{4}{*}{$M=50$} & $\phi^{11}$ & $1.28 \cdot 10^{-1} \pm 4.07 \cdot 10^{-2}$ & $1.89 \cdot 10^{-2} \pm 5.06 \cdot 10^{-3}$ \\
& $\phi^{12}$ & $3.05 \cdot 10^{-2} \pm 1.57 \cdot 10^{-2}$ & $7.99 \cdot 10^{-3} \pm 1.55 \cdot 10^{-3}$ \\
& $\phi^{21}$ & $2.44 \cdot 10^{-2} \pm 8.48 \cdot 10^{-3}$ & $8.91 \cdot 10^{-3} \pm 1.44 \cdot 10^{-3}$ \\
& $\phi^{22}$ & $1.44 \cdot 10^{-1} \pm 3.32 \cdot 10^{-2}$ & $6.73 \cdot 10^{-2} \pm 3.07 \cdot 10^{-2}$ \\
\midrule
\multirow{4}{*}{$M=100$} & $\phi^{11}$ & $1.03 \cdot 10^{-1} \pm 3.36 \cdot 10^{-2}$ & $1.46 \cdot 10^{-2} \pm 2.73 \cdot 10^{-3}$ \\
& $\phi^{12}$ & $2.80 \cdot 10^{-2} \pm 1.03 \cdot 10^{-2}$ & $5.89 \cdot 10^{-3} \pm 1.18 \cdot 10^{-3}$ \\
& $\phi^{21}$ & $2.74 \cdot 10^{-2} \pm 1.46 \cdot 10^{-2}$ & $6.58 \cdot 10^{-3} \pm 8.65 \cdot 10^{-4}$ \\
& $\phi^{22}$ & $1.37 \cdot 10^{-1} \pm 5.33 \cdot 10^{-2}$ & $4.36 \cdot 10^{-2} \pm 1.94 \cdot 10^{-2}$ \\
\midrule
\multirow{4}{*}{$M=250$} & $\phi^{11}$ & $1.12 \cdot 10^{-1} \pm 3.36 \cdot 10^{-2}$ & $1.17 \cdot 10^{-2} \pm 1.05 \cdot 10^{-3}$ \\
& $\phi^{12}$ & $2.35 \cdot 10^{-2} \pm 1.65 \cdot 10^{-2}$ & $4.30 \cdot 10^{-3} \pm 8.34 \cdot 10^{-4}$ \\
& $\phi^{21}$ & $2.36 \cdot 10^{-2} \pm 1.13 \cdot 10^{-2}$ & $4.50 \cdot 10^{-3} \pm 6.80 \cdot 10^{-4}$ \\
& $\phi^{22}$ & $9.45 \cdot 10^{-2} \pm 3.85 \cdot 10^{-2}$ & $2.56 \cdot 10^{-2} \pm 1.28 \cdot 10^{-2}$ \\
\midrule
\multirow{4}{*}{$M=500$} & $\phi^{11}$ & $1.04 \cdot 10^{-1} \pm 3.58 \cdot 10^{-2}$ & $8.56 \cdot 10^{-3} \pm 2.02 \cdot 10^{-3}$ \\
& $\phi^{12}$ & $2.34 \cdot 10^{-2} \pm 8.17 \cdot 10^{-3}$ & $3.06 \cdot 10^{-3} \pm 2.42 \cdot 10^{-4}$ \\
& $\phi^{21}$ & $1.88 \cdot 10^{-2} \pm 1.13 \cdot 10^{-2}$ & $3.18 \cdot 10^{-3} \pm 7.09 \cdot 10^{-4}$ \\
& $\phi^{22}$ & $8.32 \cdot 10^{-2} \pm 3.07 \cdot 10^{-2}$ & $1.33 \cdot 10^{-2} \pm 4.76 \cdot 10^{-3}$ \\
\midrule
\multirow{4}{*}{$M=750$} & $\phi^{11}$ & $7.83 \cdot 10^{-2} \pm 2.49 \cdot 10^{-2}$ & $8.02 \cdot 10^{-3} \pm 9.98 \cdot 10^{-4}$ \\
& $\phi^{12}$ & $2.14 \cdot 10^{-2} \pm 1.13 \cdot 10^{-2}$ & $3.00 \cdot 10^{-3} \pm 4.82 \cdot 10^{-4}$ \\
& $\phi^{21}$ & $1.26 \cdot 10^{-2} \pm 6.67 \cdot 10^{-3}$ & $2.76 \cdot 10^{-3} \pm 3.39 \cdot 10^{-4}$ \\
& $\phi^{22}$ & $6.57 \cdot 10^{-2} \pm 1.41 \cdot 10^{-2}$ & $1.02 \cdot 10^{-2} \pm 2.42 \cdot 10^{-3}$ \\
\midrule
\multirow{4}{*}{$M=1000$} & $\phi^{11}$ & $8.42 \cdot 10^{-2} \pm 1.69 \cdot 10^{-2}$ & $6.20 \cdot 10^{-3} \pm 9.49 \cdot 10^{-4}$ \\
& $\phi^{12}$ & $1.94 \cdot 10^{-2} \pm 7.59 \cdot 10^{-3}$ & $2.49 \cdot 10^{-3} \pm 2.77 \cdot 10^{-4}$ \\
& $\phi^{21}$ & $1.33 \cdot 10^{-2} \pm 8.05 \cdot 10^{-3}$ & $2.24 \cdot 10^{-3} \pm 3.71 \cdot 10^{-4}$ \\
& $\phi^{22}$ & $6.69 \cdot 10^{-2} \pm 2.22 \cdot 10^{-2}$ & $9.11 \cdot 10^{-3} \pm 1.48 \cdot 10^{-3}$ \\
\bottomrule
\end{tabular}
\end{table}

\begin{table}[htbp]
\centering
\caption{Trajectory prediction errors for linear-repulsive interaction potentials with $N_1 = N_2 = 5, L = 2$, and $\sigma = 0.05$. For each $M$ value, the top row reports error in the time interval $[0,5]$, with training data error on the left and testing data error on the right. The bottom row reports error in the time interval $[5, 10]$, which measures temporal generalization error for both settings. Both errors steadily decrease as more training data is utilized.}
\label{tab:trajectory_errors_repulsiveK}
\begin{tabular}{c|c|c}
\toprule
\multirow{2}{*}{\textbf{Parameters}} & \multicolumn{2}{c}{\textbf{Relative Trajectory Error}} \\
\cmidrule(lr){2-3}
& \textbf{Training Data} & \textbf{Test Data} \\
\midrule
$M=1$ & $1.80 \cdot 10^{-1} \pm 6.37 \cdot 10^{-2}$ & $2.34 \cdot 10^{-1} \pm 1.19 \cdot 10^{-1}$ \\
 & $2.45 \cdot 10^{-1} \pm 1.46 \cdot 10^{-1}$ & $2.91 \cdot 10^{-1} \pm 1.82 \cdot 10^{-1}$ \\
\midrule
$M=10$ & $6.01 \cdot 10^{-2} \pm 2.63 \cdot 10^{-2}$ & $8.98 \cdot 10^{-2} \pm 5.09 \cdot 10^{-2}$ \\
 & $6.99 \cdot 10^{-2} \pm 3.38 \cdot 10^{-2}$ & $9.11 \cdot 10^{-2} \pm 5.90 \cdot 10^{-2}$ \\
\midrule
$M=50$ & $3.49 \cdot 10^{-2} \pm 2.46 \cdot 10^{-2}$ & $4.00 \cdot 10^{-2} \pm 1.21 \cdot 10^{-2}$ \\
 & $4.88 \cdot 10^{-2} \pm 2.60 \cdot 10^{-2}$ & $3.67 \cdot 10^{-2} \pm 1.06 \cdot 10^{-2}$ \\
\midrule
$M=100$ & $2.52 \cdot 10^{-2} \pm 1.57 \cdot 10^{-2}$ & $3.27 \cdot 10^{-2} \pm 1.59 \cdot 10^{-2}$ \\
 & $3.11 \cdot 10^{-2} \pm 1.77 \cdot 10^{-2}$ & $3.02 \cdot 10^{-2} \pm 1.35 \cdot 10^{-2}$ \\
\midrule
$M=250$ & $1.94 \cdot 10^{-2} \pm 7.19 \cdot 10^{-3}$ & $1.72 \cdot 10^{-2} \pm 5.03 \cdot 10^{-3}$ \\
 & $2.20 \cdot 10^{-2} \pm 8.09 \cdot 10^{-3}$ & $1.66 \cdot 10^{-2} \pm 4.63 \cdot 10^{-3}$ \\
\midrule
$M=500$ & $1.31 \cdot 10^{-2} \pm 8.05 \cdot 10^{-3}$ & $1.22 \cdot 10^{-2} \pm 4.77 \cdot 10^{-3}$ \\
 & $1.79 \cdot 10^{-2} \pm 6.80 \cdot 10^{-3}$ & $1.14 \cdot 10^{-2} \pm 3.19 \cdot 10^{-3}$ \\
\midrule
$M=750$ & $1.06 \cdot 10^{-2} \pm 5.30 \cdot 10^{-3}$ & $9.32 \cdot 10^{-3} \pm 2.41 \cdot 10^{-3}$ \\
 & $1.21 \cdot 10^{-2} \pm 4.32 \cdot 10^{-3}$ & $8.79 \cdot 10^{-3} \pm 3.33 \cdot 10^{-3}$ \\
 \midrule
$M=1000$ & $9.31 \cdot 10^{-3} \pm 4.85 \cdot 10^{-3}$ & $8.13 \cdot 10^{-3} \pm 2.35 \cdot 10^{-3}$ \\
 & $1.32 \cdot 10^{-2} \pm 5.95 \cdot 10^{-3}$ & $7.46 \cdot 10^{-3} \pm 1.87 \cdot 10^{-3}$ \\
\bottomrule
\end{tabular}
\end{table}

\end{document}